\documentclass{article}

\usepackage{microtype}
\usepackage{graphicx}
\usepackage{subcaption}
\usepackage{booktabs} 

\usepackage{hyperref}

 \usepackage[preprint]{icml2026}

\usepackage{amsmath}
\usepackage{amssymb}
\usepackage{mathtools}
\usepackage{amsthm}

\usepackage[capitalize,noabbrev]{cleveref}

\theoremstyle{plain}
\newtheorem{theorem}{Theorem}[section]
\newtheorem{proposition}[theorem]{Proposition}
\newtheorem{lemma}[theorem]{Lemma}

\theoremstyle{definition}

\theoremstyle{remark}

\usepackage[textsize=tiny]{todonotes}

\icmltitlerunning{High-Probability Convergence Guarantees of Decentralized SGD}

\usepackage{nicefrac}
\usepackage{amsmath}
\usepackage{amsthm}
\usepackage{amssymb}
\usepackage{amsfonts}
\usepackage{comment}
\usepackage{mathtools}
\usepackage{algorithm}
\usepackage{graphicx}
\usepackage{xspace}

\DeclareMathOperator*{\argmin}{arg\,min}

\usepackage{xcolor}
\usepackage{enumitem}

\usepackage{bbding}

\newcommand{\lp}{\left(}
\newcommand{\rp}{\right)}

\newcommand{\lcb}{\left\{}
\newcommand{\rcb}{\right\}}
\newcommand{\lbr}{\left[}
\newcommand{\rbr}{\right]}

\newcommand{\bigO}{\mathcal{O}}

\newcommand{\sfo}{$\mathcal{SFO}$\xspace}

\newcommand{\sgd}{\textbf{\texttt{SGD}}\xspace}
\newcommand{\dsgd}{\textbf{\texttt{DSGD}}\xspace}

\newcommand{\bg}{{\mathbf g}}
\newcommand{\bgt}{{\mathbf g^{t}}}
\newcommand{\bgk}{{\mathbf g^{k}}}

\newcommand{\bx}{{\mathbf x}}
\newcommand{\bxt}{{\mathbf x^{t}}}

\newcommand{\bxs}{{\mathbf{x}^\star}}
\newcommand{\bxtp}{{\mathbf x^{t+1}}}

\newcommand{\bzt}{{\mathbf z^{t}}}

\newcommand{\Prob}{\mathbb{P}}

\usepackage{xcolor}

\usepackage{tikz}
\usetikzlibrary{calc,trees,positioning,arrows,chains,shapes.geometric,%
	decorations.pathreplacing,decorations.pathmorphing,shapes,%
	matrix,shapes.symbols}

\tikzstyle{startstop} = [rectangle, draw, rounded corners, align=center, minimum width=3cm, minimum height=1cm,text centered]
\tikzstyle{decision} = [diamond, draw, fill=blue!20, 
text width=4.5em, text badly centered, node distance=3cm, inner sep=0pt]
\tikzstyle{block} = [rectangle, draw, fill=blue!10, align=center, rounded corners, minimum width=3cm, minimum height=1cm]
\tikzstyle{blockcast} = [rectangle, draw, fill=red!10, align=center, rounded corners, minimum width=3cm, minimum height=0.45cm]
\tikzstyle{line} = [draw, -latex']
\tikzstyle{cloud} = [draw, ellipse,fill=red!20, node distance=3cm,
minimum height=2em]

\newcommand{\R}{\mathbb{R}}
\newcommand{\E}{\mathbb{E}}
\newcommand{\N}{\mathbb{N}}
\newcommand{\D}{\mathcal{D}}

\newcommand{\calN}{\mathcal{N}}

\newcommand{\F}{\mathcal{F}}
\newcommand{\Ft}{\mathcal{F}_t}
\newcommand{\FT}{\mathcal{F}_T}

\newcommand{\bW}{\mathbf{W}}
\newcommand{\tbW}{\widetilde{\mathbf{W}}}
\newcommand{\bJ}{\mathbf{J}}
\newcommand{\bI}{\mathbf{I}}
\newcommand{\ox}{\overline{x}}
\newcommand{\obx}{\overline{\bx}}

\newcommand{\og}{\overline{g}}
\newcommand{\onab}{\overline \nabla}
\newcommand{\nbfs}{\nabla \mathbf{f}^\star}
\newcommand{\nbf}{\nabla \mathbf{f}}
\newcommand{\nbft}{\nabla \mathbf{f}^t}
\newcommand{\oz}{\overline{z}}

\newcommand{\bHt}{\mathbf{H}^t}

\newcommand{\tbx}{\widetilde{\bx}}

\newcommand{\tbxtp}{\widetilde{\bx}^{t+1}}
\newcommand{\osigma}{\overline{\sigma}}
\newcommand{\tlambda}{\widetilde{\lambda}}

\newcommand{\xit}{x^t_i}
\newcommand{\xik}{x^k_i}
\newcommand{\xitp}{x^{t+1}_i}
\newcommand{\git}{g^t_i}
\newcommand{\gik}{g^k_i}
\newcommand{\ogt}{\overline{g}^t}
\newcommand{\obgt}{\overline{\bg}^t}
\newcommand{\zit}{z^t_i}
\newcommand{\zik}{z^k_i}
\newcommand{\oxt}{\overline{x}^t}
\newcommand{\oxtp}{\overline{x}^{t+1}}
\newcommand{\obxt}{\overline{\bx}^t}
\newcommand{\obxtp}{\overline{\bx}^{t+1}}
\newcommand{\ozt}{\overline{z}^t}

\begin{document}

\twocolumn[
  \icmltitle{High-Probability Convergence Guarantees of Decentralized SGD}



  \icmlsetsymbol{equal}{*}

  \begin{icmlauthorlist}
    \icmlauthor{Aleksandar Armacki}{yyy}
    \icmlauthor{Ali H. Sayed}{yyy}
  \end{icmlauthorlist}

  \icmlaffiliation{yyy}{STI, EPFL, Lausanne, Switzerland}

  \icmlcorrespondingauthor{Aleksandar Armacki}{aleksandar.armacki@epfl.ch}

  \icmlkeywords{Machine Learning, ICML}

  \vskip 0.3in
]



\printAffiliationsAndNotice{}  

\begin{abstract}
  Convergence in high-probability (HP) has attracted increasing interest, due to implying exponentially decaying tail bounds and strong guarantees for individual runs of an algorithm. While many works study HP guarantees in centralized settings, much less is understood in the decentralized setup, where existing works require strong assumptions, like uniformly bounded gradients, or asymptotically vanishing noise. This results in a significant gap between the assumptions used to establish convergence in the HP and the mean-squared error (MSE) sense, and is also contrary to centralized settings, where it is known that \sgd converges in HP under the same conditions on the cost function as needed for MSE convergence. Motivated by these observations, we study the HP convergence of Decentralized \sgd (\dsgd) in the presence of light-tailed noise, providing several strong results. First, we show that \dsgd converges in HP under the same conditions on the cost as in the MSE sense, removing the restrictive assumptions used in prior works. Second, our sharp analysis yields order-optimal rates for both non-convex and strongly convex costs. Third, we establish a linear speed-up in the number of users, leading to matching or strictly better transient times than those obtained from MSE results, further underlining the tightness of our analysis. To the best of our knowledge, this is the first work that shows \dsgd achieves a linear speed-up in the HP sense. Our relaxed assumptions and sharp rates stem from several technical results of independent interest, including a result on the variance-reduction effect of decentralized methods in the HP sense, as well as a novel bound on the moment-generating function of strongly convex costs, of interest even in centralized settings. Numerical experiments validate our theory.
\end{abstract}

\section{Introduction}

Modern large-scale machine learning applications and the abundance of data necessitate alternatives to the centralized computation framework, giving rise to distributed learning, a paradigm where multiple users collaborate to jointly train a model, e.g., \cite{sayed-networks,pmlr-v54-mcmahan17a,vlaski_et_al}. The features of distributed learning, such as storing the data locally and only exchanging smaller updates, like (quantized) model parameters or gradients, and the lack of a single point of failure, further make it an attractive paradigm from a privacy and security perspective \cite{bonawitz2017practical,yu-secure}. Many applications, such as federated training of models on mobile devices \cite{konecny2016federated}, controlling and coordinating robot swarms \cite{bullo2009distributed}, or distributed control and power grids \cite{dist-learn-noncvx}, all rely on distributed computation. From a communication perspective,  distributed frameworks can be client-server (i.e., federated) or decentralized (i.e., networked), with the main difference being that in the client-server setup users communicate with a central server, while in decentralized settings users communicate directly with each other.\footnote{While the client-server setup is often characterized by performing local updates and periodic communication \cite{stichlocal}, as well as partial user participation \cite{pmlr-v151-jee-cho22a}, we note that these features can also be incorporated in the decentralized setup, by modifying the methods to perform multiple local updates before communicating \cite{pmlr-v119-koloskova20a} and considering a dynamic network with users that are occasionally idle \cite{Cannelli2020}. However, this is not the focus of our current work.} Noting that, from the model update perspective, the client-server setup is equivalent to the decentralized setup with a fully connected communication network, we focus on the more general, decentralized setup.

The study of convergence guarantees of decentralized optimization algorithms has a long history, e.g., \cite{nedic-subgrad,jakovetic-fast,shi-extra,scutari-next,kun-dgd,jakovetic-unification,swenson-dsgd}, with most works focusing on MSE convergence, e.g., \cite{conv-rates-dsgd-jakovetic,Pu2021,vlaski-dsgd-nonconv-1,vlaski-dsgd-nonconv-2,pmlr-v119-koloskova20a,ran-vr,ran-improved,wang-cooperative}, see also \cite{adaptive-learning-ali} for an extensive treatment of the topic. Another type of convergence guarantees, namely convergence in HP, has garnered increasing attention recently. In particular, for a non-negative stochastic process $\{X^t\}_{t\in\N}$, the goal of HP convergence is to establish that, for all $t \in \N$ and any $\epsilon > 0$ 
\begin{equation}\label{eq:hp-definition}
    \Prob\big(X^t > \epsilon\big) \leq \exp\big(-Ct^{\gamma_1}\epsilon^{\gamma_2}\big),
\end{equation} where $\gamma_1,\gamma_2,C > 0$ are constants. If $\{X^t\}_{t \in \N}$ is a measure of performance, e.g., $X^t = \frac{1}{t}\sum_{k \in [t]}\|\nabla f(x^k)\|^2$, where $\{x^t\}_{t \in \N}$ is a sequence generated by some algorithm and $f$ is a non-convex cost, the relation \eqref{eq:hp-definition} provides strong guarantees with respect to a single run of that algorithm. This is particularly important in modern applications like LLMs, where it is often intractable to perform multiple runs. Numerous works study HP guarantees of \sgd-type methods in centralized settings, under both light-tailed \cite{nemirovski2009robust,ghadimi2013stochastic,li2020high,harvey2019tight,pmlr-v206-bajovic23a,liu2023high} and heavy-tailed noise \cite{nguyen2023improved,liu2023breaking,hubler2025normalization,kornilov2025sign,armacki2025high}. Comparatively, there have been very few studies of HP convergence of decentralized methods. 

\subsection{Literature Review}

We now review the literature, focusing on centralized HP results and decentralized MSE and HP convergence results. 

\textbf{Centralized HP Convergence.} \citet{nemirovski2009robust} show optimal convergence rates of \sgd for convex costs under light-tailed stochastic gradients (see assumption \textbf{(A4)} ahead for a formal definition of light-tailed noise). \citet{ghadimi2013stochastic} establish, among other, the optimal HP convergence rate of \sgd for non-convex costs under light-tailed noise, while \citet{li2020high} show the same for momentum \sgd. \citet{harvey2019tight} and \citet{pmlr-v206-bajovic23a} respectively provide optimal HP convergence rates for the last iterate of \sgd for non-smooth and smooth strongly convex costs, while \citet{liu2024revisiting} establish unified HP convergence guarantees for smooth and non-smooth convex and strongly convex costs. \citet{liu2023high} generalize the previous works on non-convex and convex costs, providing unified guarantees for several algorithms, including \sgd and AdaGrad for smooth and non-smooth costs. \citet{madden2020high} study HP convergence of \sgd under sub-Weibull noise. Another line of work studies HP convergence under heavy-tailed noise,\footnote{Since heavy-tailed noise is not the focus of our work, we point the reader to \cite{heavy-tail-book,nguyen2023improved,armacki2025high} for various conditions used to study heavy-tailed noise.} e.g., \cite{gorbunov2020stochastic,sadiev2023highprobability,nguyen2023improved,liu2023breaking,gorbunov2023breaking,hubler2025normalization,kornilov2025sign,armacki2025high,armacki2026sharp,armacki2024_ldp+mse}, where it is necessary to introduce algorithmic modifications, e.g., clipping, normalization, or sign, to ensure concentration of the form in \eqref{eq:hp-definition}. Crucially, convergence in the HP sense is achieved under the same conditions on the cost function as in the MSE sense, for both light-tailed and heavy-tailed noise. 

\textbf{Decentralized MSE Convergence.} MSE guarantees are typically studied under a variety of bounds on the second noise moment.\footnote{See, e.g., \cite{khaled2022better} for a good overview of the various conditions used in the literature.} Following one of these settings, \citet{conv-rates-dsgd-jakovetic} show that \dsgd converges at an optimal rate for strongly convex costs, while \citet{Pu2021} show that \dsgd with a fixed step-size and the gradient tracking (GT) mechanism, e.g., \cite{scutari-next,nedic-tracking,qu-harnessing}, converges to a neighbourhood of the optimal solution. \citet{vlaski-dsgd-nonconv-1,vlaski-dsgd-nonconv-2} study MSE guarantees of \dsgd for non-convex costs and show that it escapes saddle points with high probability. \citet{wang-cooperative} propose a general framework dubbed cooperative \sgd, showing optimal rates for non-convex costs. \citet{koloskova2023grad_clip} provide unified guarantees for \dsgd with local updates and changing network topology, with optimal rates and linear speed-up in the number of users for non-convex and (strongly) convex costs. \citet{ran-vr} show \dsgd with GT and variance reduction converges at a linear rate for strongly convex costs, while \citet{ran-improved} establish optimal rates of \dsgd with GT and linear speed-up in the number of users for non-convex costs and costs satisfying the Polyak-\L{}ojasiewicz (PL) condition, e.g., \cite{karimi2016linear}. \citet{yu2026decentralized} show that \dsgd with GT, momentum and normalization achieves optimal rates under heavy-tailed noise. It is worth mentioning a rich line of works studying MSE guarantees for decentralized problems such as estimation, detection, multi-objective and multitask optimization, see, e.g., \cite{consensus+innovation1,chen-decent-pareto,nassif-multitask} and references therein. 

\textbf{Decentralized HP Convergence.} Compared to MSE guarantees, there is a significantly smaller body of work on HP convergence guarantees of decentralized algorithms. In particular, \citet{dsgd-high-prob} study HP convergence of \dsgd for general non-convex and P\L{} costs under light-tailed noise, requiring uniformly bounded gradients, as well as asymptotically vanishing noise for P\L{} costs (see the discussion after Theorem \ref{thm:main-dsgd-str-cvx} for details), while showing optimal rates in both cases. \citet{dsgd-online-noncooperative-games,dsgd-online-stochastic} study HP convergence of a decentralized mirror descent algorithm under light-tailed noise, for online noncooperative games and dynamic regrets, respectively, requiring bounded gradients and compact domains. \citet{clipped-dual-avg,clipped-dsgd2} study HP convergence of decentralized algorithms with clipping, for convex costs under heavy-tailed noise, while \cite{clipped-dsgd1} also consider general non-convex costs. Finally, \citet{pmlr-v235-gorbunov24a} study clipped SGD under heavy-tailed noise, in the distributed client-server (i.e., \emph{fully connected network}) setting. It is worth mentioning works like \cite{bajovic-detection-ld,bajovic-detection-ld2,matta-diffusion-1,matta-difussion-2,bajovic-inference-ld}, where the authors study asymptotic large deviation guarantees for decentralized problems such as detection and inference. A common thread for all these works is the need for uniformly bounded gradients, or algorithmic modifications like gradient clipping, which ensure that the gradients stay bounded, as well as the fact that none of the existing HP studies achieve linear speed-up in the number of users, which is a known feature of decentralized algorithms in the MSE sense in, e.g., \cite{pmlr-v119-koloskova20a,ran-improved,unified-refined}. These observations raise the following important questions: (i) Can decentralized algorithms converge in the HP sense under the same assumptions on the cost as used for MSE results, where bounded gradients are not required? (ii) Is it possible to achieve linear speed-up in the number of users in the HP sense, which is a well-known feature of decentralized algorithms? This is further contrasted by the centralized setting, where \sgd-type algorithms converge in both HP and MSE sense under the same conditions on the cost function, e.g., \cite{ghadimi2013stochastic,liu2024revisiting}. 

\subsection{Contributions}

Motivated by the observed gaps between the existing literature on HP and MSE convergence in decentralized settings, we revisit the HP convergence guarantees in decentralized optimization, by studying a variant of vanilla \dsgd in the presence of light-tailed noise, establishing several strong results in the process. In particular, we show that \dsgd converges in HP under the same conditions on the cost as required in the MSE sense, for both non-convex and strongly convex costs, removing strong conditions like uniformly bounded gradients and (asymptotically) vanishing noise. Next, we show that \dsgd achieves optimal rates and linear speed-up in the number of users, for both non-convex and strongly convex costs, further improving on existing works and closing the gap between HP and MSE guarantees in decentralized settings. Our results are established by carefully bounding the moment-generating function (MGF) of the quantity of interest (i.e., average norm-squared of the gradient for non-convex and optimality gap of the last iterate for strongly convex costs) and the MGF of the consensus gap. Compared to \cite{dsgd-high-prob}, the work closest to ours, we provide several improvements. In particular, we remove the uniformly bounded gradient requirement, as well as the vanishing noise requirement imposed for P\L{} costs (see the discussion after Theorem \ref{thm:main-dsgd-str-cvx} ahead), while achieving linear speed-up in the number of users. Compared to works studying MSE guarantees for \dsgd, e.g., \cite{pmlr-v119-koloskova20a,conv-rates-dsgd-jakovetic}, our results require directly working with the MGF and carefully balancing between the MGFs of the quantity of interest and of the consensus gap. This is particularly challenging for strongly convex costs, where we need to show an ``almost decreasing'' property of the MGF to get improved rates on the last iterate (see Lemma \ref{lm:mgf-bound-str-cvx} and the related discussion ahead). Compared to centralized \sgd, e.g., \cite{harvey2019tight,pmlr-v206-bajovic23a,liu2023high,liu2024revisiting}, the main challenge lies in the additional consensus gap and simultaneously dealing with its MGF and that of the quantity of interest. These challenges are resolved by introducing several novelties, outlined next. 

\textbf{Novelty.} Toward establishing our improved HP guarantees, we faced several challenges requiring novel results. In particular, to remove uniformly bounded gradients, we provide Lemma \ref{lm:consensus-bound} and use the ``offset trick'' for non-convex costs (see the proof sketch of Theorem \ref{thm:main-non-cvx} in Section \ref{sec:proof-and-discuss}), while for strongly convex costs we establish a novel bound on the MGF of the consensus gap, removing the need for both bounded gradients and heterogeneity, in the form of Lemma \ref{lm:consensus-str-cvx}. To achieve linear speed-up, we show that the variance reduction benefit of decentralized learning is maintained in the HP sense, in the form of Lemma \ref{lm:noise-properties}, which is of independent interest when studying HP guarantees of decentralized algorithms. Further, we establish a novel result on the MGF of the optimality gap for strongly convex costs (more broadly, ``almost decreasing'' processes, see Section \ref{subsec:str-cvx} ahead), in the form of Lemma \ref{lm:mgf-bound-str-cvx}, which is essential in ensuring linear speed-up, by providing a more fine-grained bound on the MGF compared to existing centralized results \cite{harvey2019tight,pmlr-v206-bajovic23a,liu2024revisiting} and is of independent interest, even in centralized settings. 

\textbf{Paper Organization.} The rest of the paper is organized as follows. Section \ref{sec:problem} outlines the problem and \dsgd method, Section \ref{sec:main} presents the main results, Section \ref{sec:proof-and-discuss} provides proof sketches and discussions, Section \ref{sec:num} presents numerical experiments, while Section \ref{sec:conc} concludes the paper. Appendix contains results omitted from the main body. The remainder of this section introduces the notation used in the paper.

\textbf{Notation.} We use $\N$, $\R$ and $\R^d$ to denote positive integers, real numbers and $d$-dimensional vectors. For $m \in \N$, we use $[m] = \{1,\ldots,m\}$ to denote positive integers up to and including $m$. The notation $\langle \cdot,\cdot\rangle$ stands for the Euclidean inner product, while $\|\cdot\|$ is used for both vector and matrix induced norms. We use subscripts to denote users and superscripts to denote the iteration counter, e.g., $\xit$ refers to the model of user $i$ in iteration $t$. The ``big O'' notation $\bigO(\cdot)$ hides only global constants, unless stated otherwise.

\section{Problem Setup and \dsgd}\label{sec:problem}

In this section we introduce the problem of interest and the \dsgd algorithm. Consider a network of $n \geq 2$ users who can communicate with each other and want to jointly train a model. Formally, the problem can be cast as
\begin{equation}\label{eq:decentr-opt}
    \argmin_{x \in \R^d}\Big\{ f(x) = \frac{1}{n}\sum_{i \in [n]}f_i(x)\Big\},
\end{equation} where $x \in \R^d$ represents model parameters, $f_i: \R^d \mapsto \R$ is the cost function of user $i \in [n]$, given by $f_i(x) = \E_{\xi_i \sim \D_i}[\ell(x;\xi_i)]$, with $\xi_i \in \Xi$ being a random variable governed by an unknown distribution $\D_i$, while $\ell: \R^d \times \Xi \mapsto \R$ is a loss function. Each user has access to a Stochastic First-order Oracle (\sfo), which, when queried by user $i \in [n]$ with input $x \in \R^d$, returns the gradient of $\ell$ evaluated at a random sample, i.e., $\nabla \ell(x;\xi_i)$. The \sfo model subsumes:

1. \emph{Batch (offline) learning} - users have access to a local dataset $\{\xi_{i,l}\}_{l \in [m_i]}$, so that $f_i(x) = \frac{1}{m_i}\sum_{l \in [m_i]}\ell(x;\xi_{i,l})$ and in each round users choose a random sample $\xi_{i,l}$, which is used to compute $\nabla \ell(x;\xi_{i,l})$ and update the model;

2. \emph{Streaming (online) learning} - users do not store a local dataset, but in each round observe a random sample $\xi_{i}$, which is used to compute $\nabla \ell(x;\xi_i)$ and update the model.

\begin{algorithm}[tb]
\caption{\dsgd}
\label{alg:dsgd}
\begin{algorithmic}[1]
   \REQUIRE{Model initialization $x^1_i \in \R^{d}$, $i \in [n]$, step-size schedule $\{\alpha_t\}_{t \in \N}$;}
   \FOR{$t = 1,2,\ldots$, each user $i \in [n]$ in parallel}
        \STATE Query the oracle to obtain $\nabla \ell(\xit;\xi_i^t)$;  

        \STATE Perform the model update: \\ $\xitp = \sum_{j \in \calN_i}w_{ij}\big(x_j^t - \alpha_t\nabla \ell(x_j^t;\xi_j^t) \big)$;
    \ENDFOR
\end{algorithmic}
\end{algorithm} 

The communication pattern between users is modeled as a static graph $G = (V,E)$, where $V = [n]$ is the set of vertices representing users, while $E \subset V \times V$ is the set of edges representing communication links between users. To solve \eqref{eq:decentr-opt} in decentralized fashion, we consider a version of Decentralized Stochastic Gradient Descent (\dsgd), based on the Adapt-Then-Combine (i.e., diffusion) approach, e.g., \cite{diffusion1,diffusion2,diffusion3}. The method consists of the following steps. First, users choose a shared step-size schedule $\{\alpha_t\}_{t \in \N}$ and each user $i \in [n]$ chooses an arbitrary, but deterministically selected initial model $x^1_i \in \R^d$.\footnote{While initial models can be any real vectors, possibly different across users, for analysis purposes they need to be deterministic.} In iteration $t \geq 1$, users query the \sfo with their current model $\xit$ and receive $\nabla \ell(\xit;\xi_i^t)$. Users then update their models via the rule
\begin{equation}\label{eq:dsgd-update}
    \xitp = \sum_{j \in \calN_i}w_{ij}\big(x_j^t - \alpha_t\nabla \ell(x_j^t;\xi_j^t)\big),
\end{equation} which consists of a local update, followed by a consensus step, where $\calN_i \coloneqq \{j \in V: \{i,j\} \in E\} \cup \{i\}$ is the set of users (i.e., \emph{neighbours}) with whom user $i$ can communicate (including $i$ itself), while $w_{ij} > 0$ is the weight user $i$ assigns to user $j$'s model. The method is summarized in Algorithm \ref{alg:dsgd}. The version of \dsgd considered in our work is closely related to the variant based on Combine-Then-Adapt (i.e., consensus$+$innovation) approach, e.g., \cite{nedic-subgrad,consensus+innovation1,consensus+innovation2}, whose HP convergence is studied in \citet{dsgd-high-prob}.

\section{Main Results}\label{sec:main}

In this section we present the main results. Subsection \ref{subsec:prelim} states the preliminaries, while Subsections \ref{subsec:non-conv} and \ref{subsec:str-cvx} provide results for non-convex and strongly convex costs.

\subsection{Preliminaries}\label{subsec:prelim} 

In this subsection we outline the assumptions used in our work. For any $T \geq 1$, let $\{\xi^t_i \}_{t \in [T]}$ be the random samples observed by user $i \in [n]$ up to time $T$ and denote by $\FT$ the natural filtration with respect to the sequence of user models up to time $T$, i.e., $\FT \coloneqq \sigma\lp \lcb \{x^1_i\}_{i \in [n]},\ldots,\{x_i^T\}_{i \in [n]} \rcb \rp$. For ease of notation, let $\zit \coloneqq \nabla \ell(\xit;\xi_i^t) - \nabla f_i(\xit)$ and $W \in \R^{n\times n}$, where $[W]_{i,j} \coloneqq w_{ij}$, denote the stochastic noise and the network communication matrix, respectively.

\textbf{(A1)} The network communication matrix $W \in \R^{n \times n}$ is primitive and doubly stochastic. 

Assumption \textbf{(A1)} is satisfied by connected undirected graphs, as well as a class of strongly-connected directed graphs with doubly stochastic weights, e.g., \cite{ran-vr}.\footnote{Our work can be readily extended to the setting of time-varying networks, see Appendix \ref{sup:time-var-net} for a detailed discussion.} Let $\lambda \coloneqq \|W - J\|$ denote the network connectivity parameter, where $J \coloneqq \frac{1}{n}\mathbf{1}_n\mathbf{1}_n^\top \in \R^{n \times n}$ is the ideal communication matrix. It can be shown that \textbf{(A1)} implies $\lambda \in [0,1)$, see, e.g., \cite{Horn_Johnson_2012}. 

\textbf{(A2)} The global cost $f$ is bounded from below, i.e., $f^\star \coloneqq \inf_{x \in \R^d}f(x) > -\infty$.
    
\textbf{(A3)} Each local cost $f_i$ has $L$-Lipschitz gradients, i.e., $\|\nabla f_i(x) - \nabla f_i(y)\| \leq L\|x - y\|$, for any $x, y\in \R^d$.

Assumptions \textbf{(A2)}-\textbf{(A3)} are standard in smooth non-convex optimization, e.g., \cite{ghadimi2013stochastic,ran-improved}. While some works, e.g., \cite{li2023sgd-timeseries}, assume gradients of $\ell$ are $L$-Lipschitz continuous, we note that such assumption and \textbf{(A3)} aim to exploit different structural properties. Further, it is known that under certain regularity conditions, $L$-Lipschitz gradients of $\ell$ imply $L$-Lipschitz gradients of $f_i$, making such condition slightly stronger than \textbf{(A3)}.
    
\textbf{(A4)} The stochastic quantities satisfy the following:
\vspace{-0.5em}
\begin{enumerate}\setlength\itemsep{0.25em}
    \item The random samples $\{\xi^t_i \}_{i \in [n],t \in [T]}$, are independent across users and iterations.

    \item The stochastic gradients are unbiased, i.e., for any $i \in [n]$, $t \geq 1$ and $\Ft$-measurable vector $x \in \R^d$, we have $\E[\nabla \ell(x;\xi_i^t) \: \vert \: \Ft ] = \nabla f_i(x)$.

    \item The gradient noise $\zit$ at user $i \in [n]$ and time $t \geq 1$ is $\sigma_i$-sub-Gaussian, i.e., we have
\begin{equation*}
    \E\lbr \exp\bigg(\frac{\|z_i^t \|^2}{\sigma_i^2} \bigg) \: \big\vert \: \Ft \rbr \leq \exp(1).
\end{equation*}
\end{enumerate}

The first condition in \textbf{(A4)} is standard in decentralized stochastic optimization \cite{ran-improved,yu2026decentralized}, while the second and third require noise to be unbiased and sub-Gaussian (i.e., \emph{light-tailed}). Light tails are necessary to achieve exponentially decaying tail bounds of the form in \eqref{eq:hp-definition} if the vanilla stochastic gradient estimator is employed without any modifications (e.g., using clipping operator or estimators like median-of-means) and are widely used in centralized settings, e.g., \cite{nemirovski2009robust,ghadimi2013stochastic,li2020high,harvey2019tight,liu2023high,pmlr-v206-bajovic23a}.\footnote{While it is possible to achieve concentration of the form in \eqref{eq:hp-definition} under, e.g., sub-Weibull noise, see \cite{madden2020high} and references therein, the focus of this work is on sub-Gaussian noise.} 
    
\textbf{(A5)} The gradients of users have bounded heterogeneity, i.e., for all $x \in \R^d$ and some $A,B\geq 0$,\footnote{With at least one of $A,B$ being strictly positive.} we have $\max_{i \in [n]}\|\nabla f_i(x)\|^2 \leq A^2 + B^2\|\nabla f(x)\|^2$. 

A heterogeneity bound of the type in \textbf{(A5)} is required to ensure that \dsgd converges for non-convex costs \cite{NIPS2017_f7552665,dist-learn-noncvx,pmlr-v119-koloskova20a,wang-cooperative}. Note that \textbf{(A5)} is strictly weaker than the uniformly bounded gradient assumption used in \cite{dsgd-high-prob}, as it allows users' gradients to grow with the global gradient. Compared to \cite{pmlr-v119-koloskova20a}, who analyze the MSE convergence and require a bound on the average heterogeneity, i.e., $\frac{1}{n}\sum_{i \in [n]}\|\nabla f_i(x)\|^2 \leq A^2 + B^2\|\nabla f(x)\|^2$, we impose a slightly stronger condition. While we believe that \textbf{(A5)} can be relaxed to average heterogeneity by incorporating a similar analysis technique to the one in \cite{pmlr-v119-koloskova20a}, to keep the proofs simple and instructive, we will use assumption \textbf{(A5)}. 

\textbf{(A6)} Each $f_i$ is twice continuously differentiable and $\mu$-strongly convex, i.e., for every $x,y \in \R^d$, we have $f_i(x) \geq f_i(y) + \langle \nabla f_i(y), x-y \rangle + \frac{\mu}{2}\|x-y\|^2$.

Assumption \textbf{(A6)} is used for the strongly convex case. It is well known that \dsgd requires \textbf{(A3)} and \textbf{(A6)} to hold for each $f_i$, e.g., \cite{conv-rates-dsgd-jakovetic,pmlr-v119-koloskova20a,wang-cooperative}. The heterogeneity bound in \textbf{(A5)} and strong convexity of each cost in \textbf{(A6)} can be removed by deploying the GT technique, e.g., \cite{ran-improved}, however, this is beyond the scope of the current work. Next, let $\oxt \coloneqq \frac{1}{n}\sum_{i \in [n]}\xit$, $\og^t \coloneqq \frac{1}{n}\sum_{i \in [n]}\nabla \ell(\xit;\xi_i^t)$ and $\ozt \coloneqq \frac{1}{n}\sum_{i \in [n]}z_i^t$ denote the network-average model, stochastic gradient and noise, respectively. Using \eqref{eq:dsgd-update} and the double stochasticity of $W$, it follows that
\begin{equation*}
    \oxtp = \oxt - \alpha_t\og^t.
\end{equation*} We then have the following result.

\begin{lemma}\label{lm:descent-inequality}
    Let \textup{\textbf{(A3)}} hold. If $\alpha_t \leq \frac{1}{2L}$, we have
    \begin{align*}
        f(\oxtp) &\leq f(\oxt) - \frac{\alpha_t}{2}\| \nabla f(\oxt)\|^2 - \alpha_t\langle \nabla f(\oxt), \ozt \rangle \\
        &+ \alpha_t^2L\|\ozt\|^2 + \frac{\alpha_t L^2}{2n}\sum_{i \in [n]}\|\xit - \oxt\|^2. 
    \end{align*}
\end{lemma}

Lemma \ref{lm:descent-inequality} provides an important deterministic descent-type inequality for \dsgd and is the starting point for our analysis. The right-hand side of the above inequality consists of terms that arise in centralized \sgd plus a consensus gap term $\sum_{i \in [n]}\|\xit - \oxt\|^2$, which stems from the decentralized nature of the algorithm and bounding this term is crucial to ensure convergence. To close the subsection, we provide an important result on the behaviour of the stochastic noise in the HP sense. Let $\sigma^2 \coloneqq \frac{1}{n}\sum_{i \in [n]}\sigma^2_i$ be the average noise parameter. We then have the following result.

\begin{lemma}\label{lm:noise-properties}
    If \textup{\textbf{(A4)}} holds, then the following are true for any $t \geq 1$, $i \in [n]$ and $\Ft$-measurable $v \in \R^d$.
    \begin{enumerate}
        \item $\E\lbr\exp\lp\langle v,\zit \rangle\rp \: \vert \: \Ft\rbr \leq \exp\lp\frac{3\sigma_i^2\|v\|^2}{4}\rp$.
        
        \item $\E\lbr\exp\lp\langle v, \ozt \rangle \rp \: \vert\: \Ft \rbr \leq \exp\lp\frac{3\sigma^2\|v\|^2}{4n} \rp$.

        \item $\E\lbr \exp\bigg(\frac{n\| \ozt \|^2}{15\sigma^2} \bigg) \: \big\vert \: \Ft \rbr \leq 2d\exp(1)$.
    \end{enumerate}
\end{lemma}

Lemma \ref{lm:noise-properties} provides some properties of noise, importantly showing that the average noise is $\bigO\Big(\frac{\sigma}{\sqrt{n}}\Big)$-sub-Gaussian, establishing the variance reduction benefit of decentralized optimization in the HP sense. This result is crucial toward showing that \dsgd achieves a linear speed-up in the number of users, which we do in the following subsections. We note that the third property introduces a mild logarithmic dependence on the problem dimension $d$, with further discussion on this dependence provided in Section \ref{sec:proof-and-discuss} and Appendix \ref{sup:dimension}. 

\subsection{Non-Convex Costs}\label{subsec:non-conv}

In this subsection we present results for non-convex costs. Let $\Delta_x \coloneqq \frac{1}{n}\sum_{i \in [n]}\|x^1_i - \ox^1\|^2$ and $\Delta_f \coloneqq f(\ox^1) - f^\star$ denote the initial consensus and optimality gap.

\begin{lemma}\label{lm:consensus-bound}
    If \textup{\textbf{(A1)}} and \textup{\textbf{(A5)}} hold, then for any $t \geq 1$
    \begin{align*}
        &\frac{1}{n}\sum_{i \in [n]}\|\xitp - \oxtp\|^2 \leq 2\lambda^{2t}\Delta_x + \frac{4\lambda^2A^2}{1-\lambda}\sum_{k \in [t]}\alpha_k^2\lambda^{t-k}\\ 
        &+ \frac{4\lambda^2}{n(1-\lambda)}\sum_{k \in [t]}\alpha_k^2\lambda^{t-k}\sum_{i \in [n]}\Big[\|\zik\|^2 + B^2\|\nabla f(\xik)\|^2\Big]. 
    \end{align*}
\end{lemma} 

Lemma \ref{lm:consensus-bound} provides a useful deterministic bound on the consensus gap. Define $\osigma \coloneqq \max_{i \in [n]}\sigma_i$. Building on Lemmas \ref{lm:descent-inequality} and \ref{lm:consensus-bound}, we get the following result.

\begin{theorem}\label{thm:main-non-cvx}
    Let \textup{\textbf{(A1)}-\textbf{(A5)}} hold. If for any $T \geq 1$, the step-size is chosen such that $\alpha_T \equiv \alpha = \min\lcb C, \frac{\sqrt{n}}{\sigma\sqrt{15LT}} \rcb$, where $C > 0$ is a problem related constant satisfying \linebreak $C \leq \min\lcb \frac{1}{2L}, \frac{n}{9\sigma^2}, \frac{1-\lambda}{\lambda LB\sqrt{48}}, \frac{\sqrt{n}}{3\sigma\sqrt{5L}}, \frac{\sqrt[3]{n}(1-\lambda)^{2/3}}{\osigma^{2/3}\lambda^2L^{2/3}\sqrt[3]{9}}\rcb$, then for any $\delta \in (0,1)$, with probability at least $1 - \delta$
    \begin{align*}
        \frac{1}{nT}&\sum_{t \in [T]}\sum_{i \in [n]}\|\nabla f(\xit)\|^2 = \bigO\Bigg( \frac{\sigma\sqrt{L}\big(\Delta_f + \log(\nicefrac{2d}{\delta})\big)}{\sqrt{nT}} \\
        &+ \frac{\Delta_f + \log(\nicefrac{1}{\delta})}{CT} + \frac{\Delta_xL^2}{(1-\lambda^2)T} + \frac{n\lambda^2L(A^2 + \sigma^2)}{\sigma^2(1-\lambda)^2T} \Bigg) .
    \end{align*}
\end{theorem}

Theorem \ref{thm:main-non-cvx} establishes HP convergence of \dsgd to a stationary point (or a set of stationary points), for general smooth non-convex costs and a fixed step-size. We now discuss the bound from a few different perspectives.

\emph{Centralized SGD and linear speed-up.} Compared to centralized \sgd, which achieves the rate $\mathcal{O}\Big(\frac{1}{\sqrt{T}} + \frac{1}{T}\Big)$, e.g., \cite{liu2023high}, we can see that the rate in Theorem \ref{thm:main-non-cvx} also consists of two different order terms, namely $\mathcal{O}\Big(\frac{1}{\sqrt{nT}} + \frac{n}{T} \Big)$. Crucially, the leading term achieves a linear speed-up in the number of users, which is consistent with MSE results, e.g., \cite{pmlr-v119-koloskova20a}, showing that \dsgd retains the benefits of decentralized learning in the HP sense. 

\emph{Network effect and transient time.} We can see from the bound in Theorem \ref{thm:main-non-cvx} that network connectivity does not affect the leading term $\mathcal{O}\Big(\frac{1}{\sqrt{nT}}\Big)$, but only affects terms decaying at a faster rate, which is consistent with MSE results, e.g., \cite{pmlr-v119-koloskova20a}. Moreover, the rate in Theorem \ref{thm:main-non-cvx} implies that the transient time of \dsgd in the HP sense is of order $\mathcal{O}\Big(\frac{n^3}{(1-\lambda)^4} \Big)$, matching the transient time of \dsgd implied by MSE rates, see, e.g., Table I in \cite{unified-refined}. For details on the derivation of the transient time, the reader is referred to Appendix \ref{sup:tran-time}.  

\emph{Effects of noise and heterogeneity.} The noise (through $\sigma,\osigma$) and heterogeneity (through $A,B$) mainly affect the higher-order terms,\footnote{The effect of noise on the leading term can be eliminated by choosing the step-size $\alpha = \min\Big\{C,\frac{\sqrt{n}}{\sqrt{LT}}\Big\}$.} which is again consistent with MSE results. Note that, if $A = 0$, the bound in Theorem \ref{thm:main-non-cvx} recovers the convergence rate of noiseless gradient descent, in the sense that, when $\sigma = A = 0$, the above bound becomes $\bigO\big(\frac{1}{T}\big)$. 

\emph{Comparison with \citet{dsgd-high-prob}.} The authors in \cite{dsgd-high-prob} study HP guarantees of \dsgd over undirected time-varying networks, under a time-varying step-size schedule $\alpha_t = \frac{a}{\sqrt{t + 1}}$, assuming uniformly bounded gradients and establish the rate $\bigO\big(\frac{\log(\nicefrac{T}{\delta})}{\sqrt{T}} \big)$. As highlighted next, we provide several important improvements. First, we remove the uniformly bounded gradient assumption, replacing it with the more general heterogeneity condition \textbf{(A5)}, which is facilitated by the use of the ``offset trick'', see the discussion in Section \ref{sec:proof-and-discuss} for details. Next, \citet{dsgd-high-prob} are unable to establish a linear speed-up, which is a byproduct of both their step-size choice, as well as their treatment of the noise. In particular, choosing a time-varying step-size for non-convex costs results in a loss of linear speed-up, see Appendix \ref{sup:non-conv} for a detailed discussion. However, it would be impossible for \citet{dsgd-high-prob} to obtain a linear speed-up even under a fixed step-size, as they show no variance reduction benefit of decentralized learning in the HP sense, instead only using the definition of sub-Gaussianity. On the other hand, the variance reduction benefits of decentralized learning established in Lemma \ref{lm:noise-properties} allow us to achieve linear speed-up in the number of users when a fixed step-size is used, see Appendix \ref{sup:lin-speed-up} for detailed discussion on the results needed to achieve linear speed-up.\footnote{This is consistent with MSE results, e.g., \cite{pmlr-v119-koloskova20a,ran-improved,unified-refined}, where linear speed-up in non-convex case is achieved only via a fixed step-size.} Finally, while we consider static networks, our results can be extended to time-varying networks, e.g., \cite{dsgd-high-prob}, see Appendix \ref{sup:time-var-net} for details. 

\subsection{Strongly Convex Costs}\label{subsec:str-cvx}

In this subsection we provide results for the last iterate of strongly convex costs. Recall that strong convexity implies a unique global minimizer, e.g., \cite{nesterov-lectures_on_cvxopt}, denoted by $x^\star \in \R^d$, with $f^\star \coloneqq f(x^\star)$, and let $\kappa \coloneqq \frac{L}{\mu}$ and $\|\nbfs\|^2 \coloneqq \sum_{i \in [n]}\|\nabla f_i(x^\star)\|^2$ be the condition number and heterogeneity measure, respectively. It is known from MSE analysis that \dsgd does not require a heterogeneity bound of the form in \textbf{(A5)} for strongly convex costs, e.g., \cite{conv-rates-dsgd-jakovetic,pmlr-v119-koloskova20a}. To show this is also the case in the HP sense, we next provide a refined version of Lemma \ref{lm:consensus-bound}, which carefully leverages properties of strong convexity to bound the MGF of the consensus gap.

\begin{lemma}\label{lm:consensus-str-cvx}
    Let \textup{\textbf{(A1)}}-\textup{\textbf{(A4)}} and \textup{\textbf{(A6)}} hold, let $a,t_0,K > 0$ and the step-size be given by $\alpha_t = \frac{a}{t+t_0}$, and let $x^1_i = x^1_j$, for all $i,j\in[n]$. If $a = \frac{6}{\mu}$ and $t_0 \geq \max\Big\{6, \frac{288\osigma^2K}{\mu^2},\frac{3456\osigma^2\lambda^2K}{\mu^2(1-\lambda)},\frac{12\lambda L\sqrt{10}}{\mu(1-\lambda)}\Big\}$, then for $K_{t+1} = (t+t_0+2)K$ and any $\nu \leq \min \lcb 1, \frac{\mu^2}{144\sigma^2K} \rcb$, we have 
    \begin{align*}
        &\E\Big[\exp\Big(\nu K_{t+1}\sum_{i \in [n]}\|\xitp - \oxtp \|^2\Big)\Big] \\ &\leq \exp\bigg(\nu K_{t+1}\Big(\sum_{k \in [t]}\lambda^{t-k}S_{k} + \sum_{k \in [t]}\lambda^{t-k}D_{k}\Big) \bigg),
    \end{align*} where $S_k, D_k > 0$ are problem related constants.
\end{lemma}

Lemma \ref{lm:consensus-str-cvx} bounds the MGF of the consensus gap, without requiring the bounded heterogeneity condition \textbf{(A5)}, with $S_k,D_k$ depending on the step-size and other problem parameters, see Appendix \ref{sup:str-cvx} for explicit forms of $S_k, D_k$. We note that the requirement of same model initialization across users can be removed, at the cost of an additional, geometrically decaying term in the exponent on the RHS, however, we use same initialization, for ease of exposition. Prior to stating the main theorem, we provide an important technical result which facilitates sharper bounds and ensures linear speed-up in the number of users.

\begin{lemma}\label{lm:mgf-bound-str-cvx}
    Let $\{X^t\}_{t \in \N}$ be a sequence of random variables initialized by a deterministic $X^1 > 0$, such that, for some $M \in \N$, $a,t_0 > 0$ and every $t \geq 1$ 
    \begin{equation*}
        \E[\exp(X^{t+1})] \leq \E\bigg[\hspace{-0.2em}\exp\hspace{-0.2em}\bigg(\hspace{-0.3em}\Big(1-\frac{a}{t+t_0}\Big)X^t + \hspace{-0.4em}\sum_{i \in [M]}\hspace{-0.2em}\frac{C_i}{(t+t_0)^i} \hspace{-0.2em}\bigg)\hspace{-0.2em}\bigg],
    \end{equation*} where $C_i > 0$, $i \in [M]$. If $a \in (1,2]$ and $t_0 \geq a$, we have
    \begin{align*}
        &\E[\exp(X^{t+1})] \leq \exp\bigg(\frac{(t_0+1)^aX_1}{(t+1+t_0)^a} + \frac{2^aC_1}{a}\bigg) \\
        &\times\exp\bigg(\frac{2^aC_2/(a-1)}{t+1+t_0} + \frac{2^aC_3\log(t+1+t_0)}{(t+1+t_0)^a}\bigg) \\
        &\quad\quad\quad\times\exp\bigg(\sum_{j = 4}^M\frac{2^at_0^{3-j}C_j}{(j-3)(t+1+t_0)^a} \bigg).
    \end{align*}
\end{lemma}

Lemma \ref{lm:mgf-bound-str-cvx} provides a tight bound on the MGF of an ``almost decreasing'' process, with the optimality gap of strongly convex costs being an instance of such a process. Compared to bounds used in the centralized setting, e.g., \cite{harvey2019tight,pmlr-v206-bajovic23a,liu2024revisiting}, Lemma \ref{lm:mgf-bound-str-cvx} is significantly sharper and incorporates higher-order terms, which are crucial for ensuring linear speed-up, see Section \ref{sec:proof-and-discuss} and Appendix \ref{sup:lin-speed-up} for detailed discussions. We are now ready to state the main result for strongly convex costs.

\begin{theorem}\label{thm:main-dsgd-str-cvx}
    Let \textup{\textbf{(A1)}-\textbf{(A4)}} and \textup{\textbf{(A6)}} hold, the step-size be given by $\alpha_t = \frac{a}{t+t_0}$ and let $x^1_i = x^1_j$, for all $i,j\in[n]$. If $a =\frac{6}{\mu}$, $t_0 \geq \max\Big\{ 6,\frac{3+\lambda}{1-\lambda},\frac{1960\sigma^2\kappa}{\mu},\frac{432\osigma^2\kappa^2}{\mu},$ \linebreak $\frac{12\kappa\lambda\sqrt{10}}{1-\lambda}, \frac{5184\osigma^2\lambda^2\kappa^2}{\mu(1-\lambda)}\Big\}$ and $\nu = \min\Big\{ 1,\frac{\mu}{432\sigma^2\kappa^2}, \frac{\mu}{72\kappa} \Big\}$, then for any $\delta \in (0,1)$ and $T \geq 1$, with probability at least $1 - \delta$, it holds that 
    \begin{align*}
        \frac{1}{n}\sum_{i \in [n]}&\big(f(x_i^T)-f^\star\big) = \bigO\bigg(\frac{\nu^{-1}\log(\nicefrac{2}{\delta}) + \sigma^2\kappa\log(2d)/\mu}{n(T+t_0)} \\ 
        &+ \frac{\lambda^2L(1+L)(n\sigma^2+\|\nbfs\|^2\nicefrac{(1+\kappa^2)}{(1-\lambda))}}{(1-\lambda)n(T+t_0)^2} \bigg),
    \end{align*} where $\bigO(\cdot)$ hides some higher-order terms. 
\end{theorem}

Theorem \ref{thm:main-dsgd-str-cvx} establishes HP convergence of \dsgd to the global minimum, for smooth strongly convex costs and time-varying step-size. In Appendix \ref{sup:str-cvx} we provide the full bound, containing additional higher-order terms. Next, we discuss multiple aspects of our results.

\emph{Centralized SGD and linear speed-up.} Compared to centralized \sgd, which achieves the HP rate $\bigO\Big(\frac{1}{T} + \frac{1}{T^2}\Big)$, e.g., \cite{liu2024revisiting}, the rate in Theorem \ref{thm:main-dsgd-str-cvx} also consists of two terms of different order, namely $\bigO\Big(\frac{1}{nT} + \frac{1}{T^2} \Big)$. Similarly to non-convex results, the leading term shows a linear speed-up in the number of users, again matching the corresponding MSE guarantees, e.g., \cite{pmlr-v119-koloskova20a}. 

\emph{Network effect and transient time.} We can again see that network connectivity only affects higher-order terms. Moreover, the rate in Theorem \ref{thm:main-dsgd-str-cvx} implies a transient time of \dsgd in the HP sense, of order $\bigO\Big(\max\Big\{\frac{n}{1-\lambda}, \frac{1}{(1-\lambda)^2}\Big\}\Big)$. This is \emph{strictly sharper} than the transient time of \dsgd implied by MSE rates, of order $\bigO\Big(\frac{n}{(1-\lambda)^2}\Big)$, see Table II in \cite{unified-refined}, further highlighting the tightness of our results and improved analysis. 

\emph{Effects of noise and heterogeneity.} While heterogeneity only affects higher-order terms (through $\|\nbfs\|^2$ and $\|\bx^1 - \bxs\|^2$, see Appendix \ref{sup:str-cvx} for the full bound), we can see that the noise affects both the leading and higher-order terms, stemming from the time-varying step-size schedule and the constant $\nu^{-1} = \mathcal{O}\Big(\frac{\sigma^2\kappa^2}{\mu}\Big)$. Similarly to the discussion after Theorem \ref{thm:main-non-cvx}, the effect of noise on the leading term can be mitigated by employing a carefully selected fixed step-size. For ease of exposition, we omit the results using a fixed step-size.    

\emph{Comparison with \citet{dsgd-high-prob}.} The authors in \cite{dsgd-high-prob} establish HP convergence of \dsgd for the more general case of P\L{} costs, providing a rate of $\bigO\big(\frac{1}{T}\big)$, using the same time-varying step-size schedule as in Theorem \ref{thm:main-dsgd-str-cvx}. Crucially, the authors require two very strong assumptions: $(i)$ \emph{path-wise uniformly bounded gradients}, i.e., $\|\nabla f_i(\xit)\| \leq G$, for some $G > 0$ and all $t \geq 1$, almost surely; $(ii)$ \emph{asymptotically vanishing noise}, i.e., for all $i \in [n]$ and any $t \geq 1$, they require the following 
\begin{equation*}
    \E\lbr \exp\bigg(\frac{\|z_i^t \|^2}{\alpha_t^2\sigma_i^2} \bigg) \: \big\vert \: \Ft \rbr \leq \exp(1),
\end{equation*} where $\alpha_t = \frac{a}{t+t_0}$ is the step-size. Both conditions are difficult to satisfy, as $(i)$ is a path-wise relaxation of the uniformly bounded gradient condition (which is not satisfied for strongly convex costs), while $(ii)$ implies that the noise is $\frac{\sigma_i}{t+t_0}$-sub-Gaussian at time $t$, meaning that it vanishes at rate $\bigO\big(\frac{1}{t+t_0}\big)$.\footnote{Markov's inequality implies that a $\frac{\sigma}{t+t_0}$-sub-Gaussian random variable $X$ satisfies $\Prob\Big(X^2 \leq \frac{\sigma^2(\log(\nicefrac{1}{\delta})+1)}{(t+t_0)^2}\Big) \geq 1 - \delta$, for any $\delta \in (0,1)$. Similarly, Jensen's inequality implies $\E[X^2] \leq \frac{\sigma^2}{(t+t_0)^2}$.} Crucially, \citet{dsgd-high-prob} are unable to establish linear speed-up, even though they impose much stronger assumptions on the noise and cost. On the other hand, we show optimal convergence rate with linear speed-up using standard assumptions, with sharper transient times compared to MSE results. These improvements stem not only from a tighter analysis and Lemma \ref{lm:noise-properties}, but also Lemmas \ref{lm:consensus-str-cvx} and \ref{lm:mgf-bound-str-cvx}, which allow us to remove the bounded heterogeneity condition \textbf{(A5)} and ensure linear speed-up is achieved, see Appendix \ref{sup:lin-speed-up} for a further discussion.

\emph{On the step-size.} Throughout this section we used the step-size $\alpha_t \propto t^{-1}$, known to be optimal for strongly convex costs, e.g., \cite{rakhlin2012making}. We extend the analysis in Appendix \ref{sup:gen-step} to a general step-size $\alpha_t \propto t^{-\eta}$, $\eta \in (1/2,1]$, establishing the HP rate $\bigO\Big(\frac{\log(\nicefrac{2d}{\delta})}{nt^{2\eta-1}} + \frac{1}{t^{3\eta-1}}\Big)$, providing an important insight: for strongly convex costs the rate of decay in $n$ is always optimal and independent of the step-size. The reader is referred to Appendix \ref{sup:gen-step} for details.

\section{Proof Outlines and Discussion}\label{sec:proof-and-discuss}

In this section we briefly outline proof sketches of the main results and provide further discussions.

\textit{Proof sketch of Theorem \ref{thm:main-non-cvx}.} Using Lemma \ref{lm:descent-inequality}, rearranging and summing up the first $T$ terms, we get
\begin{align}\label{eq:descent-sum}
    &\sum_{t \in [T]}\frac{\alpha_t}{2}\|\nabla f(\oxt)\|^2 \leq \Delta_f - \sum_{t \in [T]}\alpha_t\langle \nabla f(\oxt),\ozt \rangle \nonumber \\ 
    &+ L\sum_{t \in [T]}\alpha_t^2\|\ozt\|^2 + \frac{L^2}{2}\sum_{t\in[T]}\frac{\alpha_t}{n}\sum_{i \in [n]}\|\xit - \oxt\|^2.
\end{align} We use Lemma \ref{lm:consensus-bound} to control the last term on the RHS, while to deal with $\sum_{t \in [T]}\alpha_t\langle \nabla f(\oxt),\ozt \rangle$ and remove the need for bounded gradients, we use Lemma \ref{lm:noise-properties} and the ``offset trick'', e.g., \cite{li2020high,liu2023high,armacki2026sharp}, i.e., we subtract $\sum_{t \in [T]}\frac{9\sigma^2\alpha_t^2\|\nabla f(\oxt)\|^2}{4n}$ from both sides of \eqref{eq:descent-sum}, ensuring that the effects of the inner product is absorbed by the left-hand side. The rest of the proof relies on Lemma \ref{lm:noise-properties}, some technical results and a careful selection of the step-size, see Appendix \ref{sup:non-conv} for details. 

\textit{Proof sketch of Theorem \ref{thm:main-dsgd-str-cvx}.} Using Lemma \ref{lm:descent-inequality}, properties of strongly convex functions, and defining $F^{t} \coloneqq n(t+t_0)\big(f(\oxt)-f^\star\big)$ and $A_t \coloneqq \alpha_t(t+t_0+1)$, we get
\begin{align}\label{eq:proof-sketch-str-cvx}
    F^{t+1} &\leq (1-\alpha_t\mu)\frac{t+t_0+1}{t+t_0}F^t - A_t\langle \nabla f(\oxt),\ozt\rangle \nonumber \\
    &+ \alpha_tA_tnL\|\ozt\|^2 + \frac{A_tL^2}{2}\sum_{i \in [n]}\|\xit - \oxt\|^2.
\end{align} We use \eqref{eq:proof-sketch-str-cvx}, Lemma \ref{lm:consensus-str-cvx}, and a careful analysis to show that the MGF of $F^{t}$ satisfies the conditions of Lemma \ref{lm:mgf-bound-str-cvx}, which we can then apply. The proof is completed by using these results to show that both optimality and consensus gaps are small with high probability, see Appendix \ref{sup:str-cvx} for details.

\textit{Comparison with centralized works.} Compared to works studying HP convergence of centralized \sgd, e.g., \cite{harvey2019tight,pmlr-v206-bajovic23a,liu2023high,liu2024revisiting}, we make several contributions. First, we face the challenge of controlling the MGF of the consensus gap, stemming from the decentralized nature of the algorithm. To that end, we provide Lemma \ref{lm:consensus-bound} for non-convex, as well as the improved Lemma \ref{lm:consensus-str-cvx} for strongly convex costs, which removes the bounded heterogeneity condition \textbf{(A5)}. Next, we establish Lemma \ref{lm:noise-properties}, which shows the variance reduction benefits of decentralized methods in the HP sense and is of independent interest when studying HP guarantees in decentralized settings. Finally, we provide a novel result on the MGF of the optimality gap of strongly convex costs (more broadly, ``almost decreasing'' processes), in the form of Lemma \ref{lm:mgf-bound-str-cvx}. This result provides a tighter control on the MGF compared to state-of-the-art bounds in \cite{harvey2019tight,pmlr-v206-bajovic23a,liu2024revisiting} and is, in addition to Lemma \ref{lm:noise-properties}, paramount to ensuring linear speed-up is achieved for strongly convex costs. In particular, following the same approach as in \cite{harvey2019tight,pmlr-v206-bajovic23a}, one can show that the MGF of the quantity of interest is uniformly bounded, i.e., that for some $B > 0$ and all $t \geq 1$ and $\nu \in (0,B]$
\begin{equation}\label{eq:harvey-mgf}
    \E[\exp(\nu F^t)] \leq \exp\Big(\frac{\nu}{B}\Big),    
\end{equation} where $F^t = n(t+t_0)\big(f(\oxt) - f^\star\big)$. Using \eqref{eq:harvey-mgf} and Markov's inequality, setting $\nu = B$, one can show that
\begin{equation}
    \Prob\bigg(f(\oxt) - f^\star \leq \frac{\log(\nicefrac{1}{\delta}) + 1}{nB(t+t_0)}\bigg) \geq 1 - \delta,
\end{equation} losing the linear speed-up, as $B = \bigO(1/n)$, which we show in Appendix \ref{sup:lin-speed-up}. On the other hand, Lemma \ref{lm:mgf-bound-str-cvx} provides a much sharper bound on the MGF, ensuring that linear speed-up is achieved. Similar observations hold with respect to \cite{liu2024revisiting}, see Appendix \ref{sup:lin-speed-up} for details.

\textit{On the dimension dependence.} Our bounds contain a factor of $\log(d)$, introducing a mild dependence on the problem dimension not present in either MSE results in \cite{pmlr-v119-koloskova20a}, or HP bounds in \cite{dsgd-high-prob}. This dependence stems from Lemma \ref{lm:noise-properties}, where we show that variance-reduction is achieved in the HP sense, with $\log(d)$ appearing in the exponent bounding the MGF of $\ozt$. While mild, it is not clear if the dimension dependence for the average of sub-Gaussian vectors can be fully removed, noted also by \citet{jin2019short}. For a further discussion on sub-Gaussian vectors and dimension dependence, see Appendix \ref{sup:dimension}.

\section{Numerical Experiments}\label{sec:num}

\begin{figure*}[!ht]
\centering
\begin{tabular}{cc}
\includegraphics[scale=0.45]{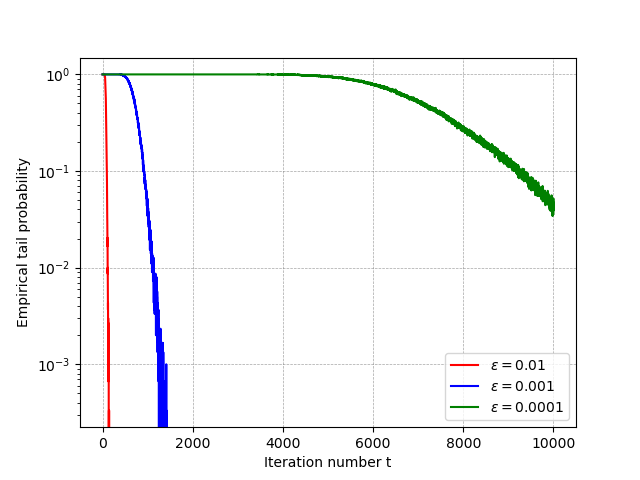}
&
\includegraphics[scale=0.45]{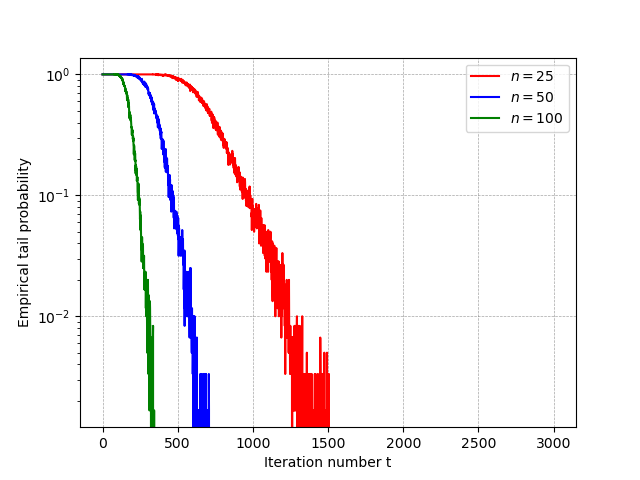}
\end{tabular}
\caption{Performance of \dsgd in the HP sense. The left figure presents performance over a fixed network of $n = 30$ users, for varying values of threshold $\varepsilon = \{10^{-2},10^{-3},10^{-4}\}$. The right figure presents performance for fixed value of threshold $\varepsilon = 10^{-3}$ and varying network size, with $n = \{25,50,100\}$ users. We can see that the empirical tail probability consistently decays exponentially fast, with faster decay over larger networks, indicating that linear speed-up is achieved.}
\label{fig:synthetic}
\end{figure*}

In this section we provide numerical experiments to evaluate our theoretical results. We consider two sets of experiments, strongly convex costs on synthetic data and non-convex costs on real data. For full details on the experimental setup and additional results, the reader is referred to Appendix \ref{sup:num}.

\textit{Synthetic data.} We consider a strongly convex quadratic problem, where the local cost of each user is given by $f_i(x) = \frac{1}{2}x^\top A_ix + b_i^\top x$, where $A_i \in \R^{d \times d}$ is positive definite, making each $f_i$ strongly convex with unbounded gradients. The users communicate over an undirected network $G$ corresponding to a random Erd\H{o}s--R\'enyi graph with connectivity parameter $p = 0.8$, while the weight matrix $W \in \R^{n \times n}$ is computed using the Metropolis-Hastings weight scheme, e.g., \cite{XIAO200465}. When queried by user $i \in [n]$, the \sfo returns $\git = A_i\xit + b_i + \zit$, where $\zit \in \R^d$ is a zero-mean Gaussian random vector, making the noise consistent with assumption \textbf{(A4)}. We use the \dsgd method outlined in Algorithm \ref{alg:dsgd}, with step-size $\alpha_t = \frac{1}{t+1}$ and shared initialization $x^1_i = 0$, for all $i \in [n]$. We aim to test two facets of our theory: (i) \emph{exponentially decaying tails} - whether the tail probability decays at an exponential scale and (ii) \emph{linear speed-up} - whether the tail probability decays faster as the number of users increases. To measure the performance, we use the empirical tail probability 
\begin{equation*}
    \Prob^t_{n,\varepsilon} = \frac{1}{R}\sum_{r \in [R]}\mathbb{I}\bigg(\frac{1}{n}\sum_{i \in [n]}\|x_i^{t,r} - x^\star\|^2 > \varepsilon\bigg),
\end{equation*} where $x_i^{t,r} \in \R^d$ is the model of user $i$ in iteration $t$ and run $r$, with $x^\star = \argmin_{x \in \R^d}f(x)$ being the solution of the global problem and $\mathbb{I}(A)$ being the indicator of event $A$. The results are presented in Figure \ref{fig:synthetic}. The left plot presents the tail decay rate for a fixed network of $n = 30$ users and varying values of accuracy threshold, given by $\varepsilon = \{10^{-2},10^{-3},10^{-4}\}$, while the right plot presents the tail decay for a fixed value of accuracy threshold $\varepsilon = 10^{-3}$ and varying number of users $n = \{25,50,100\}$. We can see from the two plots that: (i) the tail probability induced by \dsgd consistently decays at an exponential rate, for all values of $\varepsilon$ and (ii) the decay is faster for larger networks and linear speed-up is achieved, as predicted by our theory.

\textit{Real data.} We consider a non-convex logistic regression problem, where the local cost of each user is given by $f_i(x) = \frac{1}{m_i}\sum_{r \in [m_i]} \log\big(1+\exp(-y_{i,r}\langle h_{i,r}, x\rangle)\big) + \eta \sum_{k \in [d]} \frac{[x]_k^2}{1+[x]_k^2}$, where $h_{i,r} \in \mathbb{R}^{d}$ and $y_{i,r} \in \{+1, - 1\}$ are the feature vector and associated label, $\eta > 0$ is a user-specified penalty parameter, while $[x]_k$ is the $k$-th component of vector $x$. We use the ``\texttt{mushroom}'', ``\texttt{a9a}'' and ``\texttt{ijcnn1}'' datasets from the LIBSVM library \cite{chang2011libsvm}. To measure the performance, we evaluate the empirical tail probability with respect to the averaged gradient norm-squared, i.e., we compute $\Prob^t_{n,\epsilon} = \frac{1}{R}\sum_{r \in [R]}\mathbb{I}\big(G^{t,r}_n > \epsilon\big)$, where $G^{t,r}_n = \frac{1}{nt}\sum_{\tau \in [t]}\sum_{i \in [n]}\|\nabla f(x_i^{\tau,r})\|^2$. Due to space constraints, the reader is kindly referred to Appendix \ref{sup:num} for full results and further details.  

\section{Conclusion}\label{sec:conc}

We studied convergence in HP of a variant of \dsgd under light-tailed noise, showing that it is guaranteed to converge in the HP sense under the same conditions on the cost as in the MSE sense, achieving order-optimal rates and linear speed-up for both non-convex and strongly convex costs. Compared to existing works, we relax strong assumptions like uniformly bounded gradients and asymptotically vanishing noise, while simultaneously showing improved rates and sharper transient times than the ones obtained from MSE rates. Future work includes extending our results to costs satisfying the P\L{} condition, incorporating bias-correction mechanisms like GT and exact diffusion to remove bounded heterogeneity, and considering heavy-tailed noise.

\section*{Acknowledgements}

The authors would like to thank Haoyuan Cai (EPFL) for help with some of numerical experiments.

\bibliography{bibliography}

@article{jin2019short,
  title={{A Short Note on Concentration Inequalities for Random Vectors with SubGaussian Norm}},
  author={Jin, Chi and Netrapalli, Praneeth and Ge, Rong and Kakade, Sham M and Jordan, Michael I},
  journal={arXiv preprint},
  year={2019}
}

@InProceedings{sadiev2023highprobability,
  title = 	 {{High-Probability Bounds for Stochastic Optimization and Variational Inequalities: the Case of Unbounded Variance}},
  author =       {Sadiev, Abdurakhmon and Danilova, Marina and Gorbunov, Eduard and Horv\'{a}th, Samuel and Gidel, Gauthier and Dvurechensky, Pavel and Gasnikov, Alexander and Richt\'{a}rik, Peter},
  booktitle = 	 {Proceedings of the 40th International Conference on Machine Learning},
  pages = 	 {29563--29648},
  year = 	 {2023},
  volume = 	 {202},
  series = 	 {Proceedings of Machine Learning Research},
  publisher =    {PMLR},
  url = 	 {https://proceedings.mlr.press/v202/sadiev23a.html}
}

@inproceedings{armacki2025high,
  title = 	 {{High-probability Convergence Bounds for Online Nonlinear Stochastic Gradient Descent under Heavy-tailed Noise}},
  author =       {Armacki, Aleksandar and Yu, Shuhua and Sharma, Pranay and Joshi, Gauri and Bajovi\'{c}, Dragana and Jakoveti\'{c}, Du\v{s}an and Kar, Soummya},
  booktitle = 	 {Proceedings of The 28th International Conference on Artificial Intelligence and Statistics},
  pages = 	 {1774--1782},
  year = 	 {2025},
  volume = 	 {258},
  series = 	 {Proceedings of Machine Learning Research},
  publisher =    {PMLR},
  url = 	 {https://proceedings.mlr.press/v258/armacki25a.html}
}

@inproceedings{zhang2020adaptive,
 author = {Zhang, Jingzhao and Karimireddy, Sai Praneeth and Veit, Andreas and Kim, Seungyeon and Reddi, Sashank and Kumar, Sanjiv and Sra, Suvrit},
 booktitle = {Advances in Neural Information Processing Systems},
 pages = {15383--15393},
 publisher = {Curran Associates, Inc.},
 title = {{Why are Adaptive Methods Good for Attention Models?}},
 url = {https://proceedings.neurips.cc/paper_files/paper/2020/file/b05b57f6add810d3b7490866d74c0053-Paper.pdf},
 volume = {33},
 year = {2020}
}

@InProceedings{simsekli2019tail,
  title = 	 {{A Tail-Index Analysis of Stochastic Gradient Noise in Deep Neural Networks}},
  author =       {Simsekli, Umut and Sagun, Levent and Gurbuzbalaban, Mert},
  booktitle = 	 {Proceedings of the 36th International Conference on Machine Learning},
  pages = 	 {5827--5837},
  year = 	 {2019},
  volume = 	 {97},
  series = 	 {Proceedings of Machine Learning Research},
  publisher =    {PMLR},
  url = 	 {https://proceedings.mlr.press/v97/simsekli19a.html}
}

@inproceedings{rakhlin2012making,
  title={{Making Gradient Descent Optimal for Strongly Convex Stochastic Optimization}},
  author={Rakhlin, Alexander and Shamir, Ohad and Sridharan, Karthik},
  booktitle={Proceedings of the 29th International Coference on International Conference on Machine Learning},
  pages={1571--1578},
  year={2012}
}

@article{ghadimi2013stochastic,
  title = {{Stochastic First- and Zeroth-Order Methods for Nonconvex Stochastic Programming}},
  author={Ghadimi, Saeed and Lan, Guanghui},
  journal={SIAM Journal on Optimization},
  volume={23},
  number={4},
  pages={2341--2368},
  year={2013},
  publisher={SIAM},
  doi = {10.1137/120880811},
  URL = {https://doi.org/10.1137/120880811},
  eprint = {https://doi.org/10.1137/120880811}
}

@article{li2020high,
  title={{A high probability analysis of adaptive sgd with momentum}},
  author={Li, Xiaoyu and Orabona, Francesco},
  journal={Workshop on ``Beyond first-order methods in ML systems'', 37th International Conference on Machine Learning},
  year={2020}
}

@InProceedings{koloskova2023grad_clip,
  title = 	 {{Revisiting Gradient Clipping: Stochastic bias and tight convergence guarantees}},
  author =       {Koloskova, Anastasia and Hendrikx, Hadrien and Stich, Sebastian U},
  booktitle = 	 {Proceedings of the 40th International Conference on Machine Learning},
  pages = 	 {17343--17363},
  year = 	 {2023},
  volume = 	 {202},
  series = 	 {Proceedings of Machine Learning Research},
  publisher =    {PMLR},
  url = 	 {https://proceedings.mlr.press/v202/koloskova23a.html}
}

@inproceedings{nguyen2023improved,
 author = {Nguyen, Ta Duy and Nguyen, Thien H and Ene, Alina and Nguyen, Huy},
 booktitle = {Advances in Neural Information Processing Systems},
 pages = {24191--24222},
 publisher = {Curran Associates, Inc.},
 title = {{Improved Convergence in High Probability of Clipped Gradient Methods with Heavy Tailed Noise}},
 url = {https://proceedings.neurips.cc/paper_files/paper/2023/file/4c454d34f3a4c8d6b4ca85a918e5d7ba-Paper-Conference.pdf},
 volume = {36},
 year = {2023}
}

@InProceedings{liu2023breaking,
  title = 	 {{Breaking the Lower Bound with (Little) Structure: Acceleration in Non-Convex Stochastic Optimization with Heavy-Tailed Noise}},
  author =       {Liu, Zijian and Zhang, Jiawei and Zhou, Zhengyuan},
  booktitle = 	 {Proceedings of Thirty Sixth Conference on Learning Theory},
  pages = 	 {2266--2290},
  year = 	 {2023},
  volume = 	 {195},
  series = 	 {Proceedings of Machine Learning Research},
  publisher =    {PMLR},
  url = 	 {https://proceedings.mlr.press/v195/liu23c.html}
}

@inproceedings{gorbunov2020stochastic,
 author = {Gorbunov, Eduard and Danilova, Marina and Gasnikov, Alexander},
 booktitle = {Advances in Neural Information Processing Systems},
 pages = {15042--15053},
 publisher = {Curran Associates, Inc.},
 title = {{Stochastic Optimization with Heavy-Tailed Noise via Accelerated Gradient Clipping}},
 url = {https://proceedings.neurips.cc/paper_files/paper/2020/file/abd1c782880cc59759f4112fda0b8f98-Paper.pdf},
 volume = {33},
 year = {2020}
}

@InProceedings{pmlr-v206-bajovic23a,
  title = 	 {{Large deviations rates for stochastic gradient descent with strongly convex functions}},
  author =       {Bajovi\'{c}, Dragana and Jakoveti\'{c}, Du\v{s}an and Kar, Soummya},
  booktitle = 	 {Proceedings of The 26th International Conference on Artificial Intelligence and Statistics},
  pages = 	 {10095--10111},
  year = 	 {2023},
  volume = 	 {206},
  series = 	 {Proceedings of Machine Learning Research},
  publisher =    {PMLR},
  url = 	 {https://proceedings.mlr.press/v206/bajovic23a.html}
}

@InProceedings{harvey2019tight,
  title = 	 {{Tight analyses for non-smooth stochastic gradient descent}},
  author =       {Harvey, Nicholas J. A. and Liaw, Christopher and Plan, Yaniv and Randhawa, Sikander},
  booktitle = 	 {Proceedings of the Thirty-Second Conference on Learning Theory},
  pages = 	 {1579--1613},
  year = 	 {2019},
  volume = 	 {99},
  series = 	 {Proceedings of Machine Learning Research},
  publisher =    {PMLR},
  url = 	 {https://proceedings.mlr.press/v99/harvey19a.html}
}

@article{nemirovski2009robust,
  author={Nemirovski{\u\i}, Arkadi and Juditsky, Anatoli and Lan, Guanghui and Shapiro, Alexander},
  title = {{Robust Stochastic Approximation Approach to Stochastic Programming}},
  journal = {SIAM Journal on Optimization},
  volume = {19},
  number = {4},
  pages = {1574-1609},
  year = {2009},
  doi = {10.1137/070704277},
  URL = {https://doi.org/10.1137/070704277},
  eprint = {https://doi.org/10.1137/070704277}
}

@book{nesterov-lectures_on_cvxopt,
  author = {Nesterov, Yurii},
  title = {{Lectures on Convex Optimization}},
  year = {2018},
  isbn = {978-3-319-91577-7},
  doi = {https://doi.org/10.1007/978-3-319-91578-4},
  series = {Springer Optimization and Its Applications},
  publisher = {Springer Cham},
  edition = {2nd}
}

@InProceedings{gorbunov2023breaking,
  title = 	 {{Breaking the Heavy-Tailed Noise Barrier in Stochastic Optimization Problems}},
  author =       {Puchkin, Nikita and Gorbunov, Eduard and Kutuzov, Nickolay and Gasnikov, Alexander},
  booktitle = 	 {Proceedings of The 27th International Conference on Artificial Intelligence and Statistics},
  pages = 	 {856--864},
  year = 	 {2024},
  volume = 	 {238},
  series = 	 {Proceedings of Machine Learning Research},
  publisher =    {PMLR},
  url = 	 {https://proceedings.mlr.press/v238/puchkin24a.html}
}

@InProceedings{liu2023high,
  title = 	 {{High Probability Convergence of Stochastic Gradient Methods}},
  author =       {Liu, Zijian and Nguyen, Ta Duy and Nguyen, Thien Hang and Ene, Alina and Nguyen, Huy},
  booktitle = 	 {Proceedings of the 40th International Conference on Machine Learning},
  pages = 	 {21884--21914},
  year = 	 {2023},
  volume = 	 {202},
  series = 	 {Proceedings of Machine Learning Research},
  publisher =    {PMLR},
  url = 	 {https://proceedings.mlr.press/v202/liu23aa.html}
}

@article{madden2020high,
  author  = {Liam Madden and Emiliano Dall'Anese and Stephen Becker},
  title   = {{High Probability Convergence Bounds for Non-convex Stochastic Gradient Descent with Sub-Weibull Noise}},
  journal = {Journal of Machine Learning Research},
  year    = {2024},
  volume  = {25},
  number  = {241},
  pages   = {1--36},
  url     = {http://jmlr.org/papers/v25/23-0466.html}
}

@inproceedings{karimi2016linear,
  title={{Linear Convergence of Gradient and Proximal-Gradient Methods Under the Polyak-{\L}ojasiewicz Condition}},
  author={Karimi, Hamed and Nutini, Julie and Schmidt, Mark},
  booktitle={Machine Learning and Knowledge Discovery in Databases},
  pages={795--811},
  year={2016},
  publisher={Springer International},
  isbn={978-3-319-46128-1}
}

@book{heavy-tail-book, 
    place={Cambridge}, 
    series={Cambridge Series in Statistical and Probabilistic Mathematics}, 
    title={{The Fundamentals of Heavy Tails: Properties, Emergence, and Estimation}}, 
    publisher={Cambridge University Press}, 
    author={Nair, Jayakrishnan and Wierman, Adam and Zwart, Bert}, 
    year={2022}, 
    collection={Cambridge Series in Statistical and Probabilistic Mathematics}
}

@ARTICLE{bajovic-detection-ld,
  author={Bajovi\'{c}, Dragana and Jakoveti\'{c}, Dus˘an and Xavier, João and Sinopoli, Bruno and Moura, José M. F.},
  journal={IEEE Transactions on Signal Processing}, 
  title={{Distributed Detection via Gaussian Running Consensus: Large Deviations Asymptotic Analysis}}, 
  year={2011},
  volume={59},
  number={9},
  pages={4381-4396},
  doi={10.1109/TSP.2011.2157147}
}

@ARTICLE{bajovic-inference-ld,
  author={Bajovi\'{c}, Dragana},
  journal={IEEE Transactions on Information Theory}, 
  title={{Inaccuracy Rates for Distributed Inference Over Random Networks With Applications to Social Learning}}, 
  year={2024},
  volume={70},
  number={1},
  pages={415-435},
  doi={10.1109/TIT.2023.3324866}
}

@ARTICLE{bajovic-detection-ld2,
  author={Bajovi\'{c}, Dragana and Jakoveti\'{c}, Dušan and Moura, José M. F. and Xavier, João and Sinopoli, Bruno},
  journal={IEEE Transactions on Signal Processing}, 
  title={{Large Deviations Performance of Consensus+Innovations Distributed Detection With Non-Gaussian Observations}}, 
  year={2012},
  volume={60},
  number={11},
  pages={5987-6002},
  doi={10.1109/TSP.2012.2210885}
}

@article{khaled2022better,
    title={{Better Theory for {SGD} in the  Nonconvex World}},
    author={Ahmed Khaled and Peter Richt{\'a}rik},
    journal={Transactions on Machine Learning Research},
    issn={2835-8856},
    year={2023},
    url={https://openreview.net/forum?id=AU4qHN2VkS}
}

@article{bertsekas-gradient,
    author = {Bertsekas, Dimitri P. and Tsitsiklis, John N.},
    title = {{Gradient Convergence in Gradient methods with Errors}},
    journal = {SIAM Journal on Optimization},
    volume = {10},
    number = {3},
    pages = {627-642},
    year = {2000},
    doi = {10.1137/S1052623497331063},
    URL = {https://doi.org/10.1137/S1052623497331063},
    eprint = {https://doi.org/10.1137/S1052623497331063}
}

@article{armacki2024_ldp+mse,
  author = {Armacki, Aleksandar and Yu, Shuhua and Bajovi\'{c}, Dragana and Jakoveti\'{c}, Du\v{s}an and Kar, Soummya},
  title = {{Large Deviation Upper Bounds and Improved MSE Rates of Nonlinear SGD: Heavy-Tailed Noise and Power of Symmetry}},
  journal = {SIAM Journal on Optimization},
  volume = {36},
  number = {1},
  pages = {32-59},
  year = {2026},
  doi = {10.1137/24M1704154},
  URL = {https://doi.org/10.1137/24M1704154},
  eprint = {https://doi.org/10.1137/24M1704154}
}

@InProceedings{hubler2025normalization,
  title = 	 {{From Gradient Clipping to Normalization for Heavy Tailed SGD}},
  author =       {H{\"u}bler, Florian and Fatkhullin, Ilyas and He, Niao},
  booktitle = 	 {Proceedings of The 28th International Conference on Artificial Intelligence and Statistics},
  pages = 	 {2413--2421},
  year = 	 {2025},
  volume = 	 {258},
  series = 	 {Proceedings of Machine Learning Research},
  month = 	 {03--05 May},
  publisher =    {PMLR},
  url = 	 {https://proceedings.mlr.press/v258/hubler25a.html}
}

@article{kornilov2025sign,
      title={{Sign Operator for Coping with Heavy-Tailed Noise in Non-Convex Optimization: High Probability Bounds Under $(L_0,L_1)$-Smoothness}}, 
      author={Nikita Kornilov and Philip Zmushko and Andrei Semenov and Alexander Gasnikov and Alexander Beznosikov},
      year={2025},
      journal={arXiv preprint} 
}

@article{antoniadis2011penalized,
  title={{Penalized likelihood regression for generalized linear models with non-quadratic penalties}},
  author={Antoniadis, Anestis and Gijbels, Ir{\`e}ne and Nikolova, Mila},
  journal={Annals of the Institute of Statistical Mathematics},
  volume={63},
  number={3},
  pages={585--615},
  year={2011},
  publisher={Springer}
}

@ARTICLE{consensus+innovation2,
  author={Kar, Soummya and Moura, Jose M.F.},
  journal={IEEE Signal Processing Magazine}, 
  title={{Consensus + innovations distributed inference over networks: cooperation and sensing in networked systems}}, 
  year={2013},
  volume={30},
  number={3},
  pages={99-109},
  doi={10.1109/MSP.2012.2235193}
}

@ARTICLE{consensus+innovation1,
  author={Kar, Soummya and Moura, José M. F. and Ramanan, Kavita},
  journal={IEEE Transactions on Information Theory}, 
  title={{Distributed Parameter Estimation in Sensor Networks: Nonlinear Observation Models and Imperfect Communication}}, 
  year={2012},
  volume={58},
  number={6},
  pages={3575-3605},
  doi={10.1109/TIT.2012.2191450}
}

@ARTICLE{armacki2026sharp,
  author={Armacki, Aleksandar and Bajovi\'{c}, Dragana and Jakoveti\'{c}, Du\v{s}an and Kar, Soummya},
  journal={IEEE Transactions on Information Theory}, 
  title={{Sharp High-Probability Rates for Nonlinear SGD under Heavy-Tailed Noise via Symmetrization}}, 
  year={2026},
  volume={},
  number={},
  pages={1-23},
  doi={10.1109/TIT.2026.3682577}
}

@ARTICLE{dsgd-high-prob,
  author={Lu, Kaihong and Wang, Hongxia and Zhang, Huanshui and Wang, Long},
  journal={IEEE Transactions on Automatic Control}, 
  title={{Convergence in High Probability of Distributed Stochastic Gradient Descent Algorithms}}, 
  year={2024},
  volume={69},
  number={4},
  pages={2189-2204},
  doi={10.1109/TAC.2023.3327319}
}

@ARTICLE{dsgd-online-noncooperative-games,
  author={Lu, Kaihong},
  journal={IEEE Transactions on Automatic Control}, 
  title={{Online Distributed Algorithms for Online Noncooperative Games With Stochastic Cost Functions: High Probability Bound of Regrets}}, 
  year={2024},
  volume={69},
  number={12},
  pages={8860-8867}
}

@article{dsgd-online-stochastic,
title={{Online distributed nonconvex optimization with stochastic objective functions: High probability bound analysis of dynamic regrets}},
journal = {Automatica},
volume = {170},
pages = {111863},
year = {2024},
issn = {0005-1098},
doi = {https://doi.org/10.1016/j.automatica.2024.111863},
url = {https://www.sciencedirect.com/science/article/pii/S0005109824003571},
author = {Hang Xu and Kaihong Lu and Yu-Long Wang}
}

@ARTICLE{ran-improved,
  author={Xin, Ran and Khan, Usman A. and Kar, Soummya},
  journal={IEEE Transactions on Signal Processing}, 
  title={{An Improved Convergence Analysis for Decentralized Online Stochastic Non-Convex Optimization}}, 
  year={2021},
  volume={69},
  number={},
  pages={1842-1858}
}

@ARTICLE{nedic-subgrad,
  author={Nedi\'{c}, Angelia and Ozdaglar, Asuman},
  journal={IEEE Transactions on Automatic Control}, 
  title={{Distributed Subgradient Methods for Multi-Agent Optimization}}, 
  year={2009},
  volume={54},
  number={1},
  pages={48-61},
  doi={10.1109/TAC.2008.2009515}
}

@ARTICLE{diffusion3,
  author={Chen, Jianshu and Sayed, Ali H.},
  journal={IEEE Transactions on Signal Processing}, 
  title={{Diffusion Adaptation Strategies for Distributed Optimization and Learning Over Networks}}, 
  year={2012},
  volume={60},
  number={8},
  pages={4289-4305},
  doi={10.1109/TSP.2012.2198470}
}

@ARTICLE{diffusion2,
  author={Cattivelli, Federico S. and Sayed, Ali H.},
  journal={IEEE Transactions on Signal Processing}, 
  title={{Diffusion LMS Strategies for Distributed Estimation}}, 
  year={2010},
  volume={58},
  number={3},
  pages={1035-1048},
  doi={10.1109/TSP.2009.2033729}
}

@ARTICLE{diffusion1,
  author={Lopes, Cassio G. and Sayed, Ali H.},
  journal={IEEE Transactions on Signal Processing}, 
  title={{Diffusion Least-Mean Squares Over Adaptive Networks: Formulation and Performance Analysis}}, 
  year={2008},
  volume={56},
  number={7},
  pages={3122-3136},
  doi={10.1109/TSP.2008.917383}
}

@ARTICLE{ran-vr,
  author={Xin, Ran and Khan, Usman A. and Kar, Soummya},
  journal={IEEE Transactions on Signal Processing}, 
  title={{Variance-Reduced Decentralized Stochastic Optimization With Accelerated Convergence}}, 
  year={2020},
  volume={68},
  number={},
  pages={6255-6271},
  doi={10.1109/TSP.2020.3031071}
}

@book{Horn_Johnson_2012, 
    place={Cambridge}, 
    edition={2}, 
    title={{Matrix Analysis}}, 
    publisher={Cambridge University Press}, 
    author={Horn, Roger A. and Johnson, Charles R.}, 
    year={2012}
}

@inproceedings{yu2026decentralized,
  title={{Decentralized Nonconvex Optimization under Heavy-Tailed Noise: Normalization and Optimal Convergence}},
  author={Shuhua Yu and Du\v{s}an Jakoveti\'{c} and Soummya Kar},
  booktitle={The Fourteenth International Conference on Learning Representations},
  year={2026},
  url={https://openreview.net/forum?id=B0qUqxBOT6}
}

@InProceedings{pmlr-v119-koloskova20a,
  title= 	 {{A Unified Theory of Decentralized {SGD} with Changing Topology and Local Updates}},
  author =       {Koloskova, Anastasia and Loizou, Nicolas and Boreiri, Sadra and Jaggi, Martin and Stich, Sebastian},
  booktitle= 	 {{Proceedings of the 37th International Conference on Machine Learning}},
  pages = 	 {5381--5393},
  year = 	 {2020},
  volume = 	 {119},
  series = 	 {Proceedings of Machine Learning Research},
  publisher =    {PMLR},
  url = 	 {https://proceedings.mlr.press/v119/koloskova20a.html}
}

@ARTICLE{dist-learn-noncvx,
  author={Chang, Tsung-Hui and Hong, Mingyi and Wai, Hoi-To and Zhang, Xinwei and Lu, Songtao},
  journal={IEEE Signal Processing Magazine}, 
  title={{Distributed Learning in the Nonconvex World: From batch data to streaming and beyond}}, 
  year={2020},
  volume={37},
  number={3},
  pages={26-38},
  doi={10.1109/MSP.2020.2970170}
}

@article{wang-cooperative,
  author  = {Jianyu Wang and Gauri Joshi},
  title={{Cooperative SGD: A Unified Framework for the Design and Analysis of Local-Update SGD Algorithms}},
  journal = {Journal of Machine Learning Research},
  year    = {2021},
  volume  = {22},
  number  = {213},
  pages   = {1--50},
  url     = {http://jmlr.org/papers/v22/20-147.html}
}

@inproceedings{NIPS2017_f7552665,
 author = {Lian, Xiangru and Zhang, Ce and Zhang, Huan and Hsieh, Cho-Jui and Zhang, Wei and Liu, Ji},
 booktitle={{Advances in Neural Information Processing Systems}},
 pages = {},
 publisher = {Curran Associates, Inc.},
 title={{Can Decentralized Algorithms Outperform Centralized Algorithms? A Case Study for Decentralized Parallel Stochastic Gradient Descent}},
 url = {https://proceedings.neurips.cc/paper_files/paper/2017/file/f75526659f31040afeb61cb7133e4e6d-Paper.pdf},
 volume = {30},
 year = {2017}
}

@INPROCEEDINGS{conv-rates-dsgd-jakovetic,
  author={Jakoveti\'{c}, Du\v{s}an and Bajovi\'{c}, Dragana and Sahu, Anit Kumar and Kar, Soummya},
  booktitle={{2018 IEEE Conference on Decision and Control (CDC)}}, 
  title={{Convergence Rates for Distributed Stochastic Optimization Over Random Networks}}, 
  year={2018},
  volume={},
  number={},
  pages={4238-4245},
  doi={10.1109/CDC.2018.8619228}
}

@ARTICLE{sayed-networks,
  author={Sayed, Ali H.},
  journal={Proceedings of the IEEE}, 
  title={{Adaptive Networks}}, 
  year={2014},
  volume={102},
  number={4},
  pages={460-497},
  doi={10.1109/JPROC.2014.2306253}
}

@ARTICLE{vlaski_et_al,
  author={Vlaski, Stefan and Kar, Soummya and Sayed, Ali H. and Moura, José M.F.},
  journal={IEEE Signal Processing Magazine}, 
  title={{Networked Signal and Information Processing: Learning by multiagent systems}}, 
  year={2023},
  volume={40},
  number={5},
  pages={92-105},
  doi={10.1109/MSP.2023.3267896}
}

@InProceedings{pmlr-v54-mcmahan17a,
  title= 	 {{Communication-Efficient Learning of Deep Networks from Decentralized Data}},
  author = 	 {McMahan, Brendan and Moore, Eider and Ramage, Daniel and Hampson, Seth and Arcas, Blaise Aguera y},
  booktitle= 	 {{Proceedings of the 20th International Conference on Artificial Intelligence and Statistics}},
  pages = 	 {1273--1282},
  year = 	 {2017},
  volume = 	 {54},
  series = 	 {Proceedings of Machine Learning Research},
  publisher =    {PMLR},
  url = 	 {https://proceedings.mlr.press/v54/mcmahan17a.html}
}

@ARTICLE{yu-secure,
  author={Yu, Shuhua and Kar, Soummya},
  journal={IEEE Transactions on Signal Processing}, 
  title={{Secure Distributed Optimization Under Gradient Attacks}}, 
  year={2023},
  volume={71},
  number={},
  pages={1802-1816},
  doi={10.1109/TSP.2023.3277211}
}

@InProceedings{pmlr-v151-jee-cho22a,
  title= 	 {{Towards Understanding Biased Client Selection in Federated Learning}},
  author =       {Cho, Yae Jee and Wang, Jianyu and Joshi, Gauri},
  booktitle= 	 {{Proceedings of The 25th International Conference on Artificial Intelligence and Statistics}},
  pages = 	 {10351--10375},
  year = 	 {2022},
  volume = 	 {151},
  series = 	 {Proceedings of Machine Learning Research},
  publisher =    {PMLR},
  url = 	 {https://proceedings.mlr.press/v151/jee-cho22a.html}
}

@article{konecny2016federated,
  title={{Federated Learning: Strategies for Improving Communication Efficiency}},
  author={Kone{\v{c}}n{\`y}, Jakub and McMahan, H Brendan and Yu, Felix X and Richt{\'a}rik, Peter and Suresh, Ananda Theertha and Bacon, Dave},
  journal={NIPS Workshop on Private Multi-Party Machine Learning},
  year={2016},
  url={https://arxiv.org/abs/1610.05492}
}

@book{bullo2009distributed,
  title={{Distributed Control of Robotic Networks: A Mathematical Approach to Motion Coordination Algorithms}},
  author={Bullo, Francesco and Cort{\'e}s, Jorge and Martinez, Sonia},
  year={2009},
  publisher={Princeton University Press}
}

@inproceedings{stichlocal,
  title={{Local SGD Converges Fast and Communicates Little}},
  author={Stich, Sebastian U},
  booktitle={{International Conference on Learning Representations}},
  year={2019} 
}

@inproceedings{bonawitz2017practical,
  title={{Practical secure aggregation for privacy-preserving machine learning}},
  author={Bonawitz, Keith and Ivanov, Vladimir and Kreuter, Ben and Marcedone, Antonio and McMahan, H Brendan and Patel, Sarvar and Ramage, Daniel and Segal, Aaron and Seth, Karn},
  booktitle={{Proceedings of the 2017 ACM SIGSAC Conference on Computer and Communications Security}},
  pages={1175--1191},
  year={2017}
}

@ARTICLE{vlaski-dsgd-nonconv-1,
  author={Vlaski, Stefan and Sayed, Ali H.},
  journal={IEEE Transactions on Signal Processing}, 
  title={{Distributed Learning in Non-Convex Environments—Part I: Agreement at a Linear Rate}}, 
  year={2021},
  volume={69},
  number={},
  pages={1242-1256},
  doi={10.1109/TSP.2021.3050858}
}

@ARTICLE{vlaski-dsgd-nonconv-2,
  author={Vlaski, Stefan and Sayed, Ali H.},
  journal={IEEE Transactions on Signal Processing}, 
  title={{Distributed Learning in Non-Convex Environments— Part II: Polynomial Escape From Saddle-Points}}, 
  year={2021},
  volume={69},
  number={},
  pages={1257-1270},
  doi={10.1109/TSP.2021.3050840}
}

@ARTICLE{jakovetic-fast,
  author={Jakoveti\'{c}, Du\v{s}an and Xavier, João and Moura, José M. F.},
  journal={IEEE Transactions on Automatic Control}, 
  title={{Fast Distributed Gradient Methods}}, 
  year={2014},
  volume={59},
  number={5},
  pages={1131-1146},
  doi={10.1109/TAC.2014.2298712}
}

@ARTICLE{jakovetic-unification,
  author={Jakoveti\'{c}, Du\v{s}an},
  journal={IEEE Transactions on Signal and Information Processing over Networks}, 
  title={{A Unification and Generalization of Exact Distributed First-Order Methods}}, 
  year={2019},
  volume={5},
  number={1},
  pages={31-46},
  doi={10.1109/TSIPN.2018.2846183}
}

@article{nedic-tracking,
author = {Nedi\'{c}, Angelia and Olshevsky, Alex and Shi, Wei},
title={{Achieving Geometric Convergence for Distributed Optimization Over Time-Varying Graphs}},
journal = {SIAM Journal on Optimization},
volume = {27},
number = {4},
pages = {2597-2633},
year = {2017},
doi = {10.1137/16M1084316},
URL = {https://doi.org/10.1137/16M1084316}
}

@article{shi-extra,
author = {Shi, Wei and Ling, Qing and Wu, Gang and Yin, Wotao},
title={{EXTRA: An Exact First-Order Algorithm for Decentralized Consensus Optimization}},
journal = {SIAM Journal on Optimization},
volume = {25},
number = {2},
pages = {944-966},
year = {2015},
doi = {10.1137/14096668X},
URL = {https://doi.org/10.1137/14096668X}
}

@ARTICLE{scutari-next,
  author={Lorenzo, Paolo Di and Scutari, Gesualdo},
  journal={IEEE Transactions on Signal and Information Processing over Networks}, 
  title={{NEXT: In-Network Nonconvex Optimization}}, 
  year={2016},
  volume={2},
  number={2},
  pages={120-136},
  doi={10.1109/TSIPN.2016.2524588}
}

@article{kun-dgd,
author = {Yuan, Kun and Ling, Qing and Yin, Wotao},
title={{On the Convergence of Decentralized Gradient Descent}},
journal = {SIAM Journal on Optimization},
volume = {26},
number = {3},
pages = {1835-1854},
year = {2016},
doi = {10.1137/130943170},
URL = {https://doi.org/10.1137/130943170}
}

@article{swenson-dsgd,
author = {Swenson, Brian and Murray, Ryan and Poor, H. Vincent and Kar, Soummya},
title={{Distributed Stochastic Gradient Descent: Nonconvexity, Nonsmoothness, and Convergence to Local Minima}},
year = {2022},
issue_date = {January 2022},
publisher = {JMLR.org},
volume = {23},
number = {1},
issn = {1532-4435},
journal = {JMLR},
month = {jan},
articleno = {328},
numpages = {62}
}

@Article{Pu2021,
author={Pu, Shi
and Nedi{\'{c}}, Angelia},
title={{Distributed stochastic gradient tracking methods}},
journal={Mathematical Programming},
year={2021},
month={May},
day={01},
volume={187},
number={1},
pages={409-457},
issn={1436-4646},
doi={10.1007/s10107-020-01487-0},
url={https://doi.org/10.1007/s10107-020-01487-0}
}

@ARTICLE{qu-harnessing,
  author={Qu, Guannan and Li, Na},
  journal={IEEE Transactions on Control of Network Systems}, 
  title={{Harnessing Smoothness to Accelerate Distributed Optimization}}, 
  year={2018},
  volume={5},
  number={3},
  pages={1245-1260},
  doi={10.1109/TCNS.2017.2698261}
}

@ARTICLE{clipped-dual-avg,
  author={Qin, Yanfu and Lu, Kaihong and Xu, Hang and Chen, Xiangyong},
  journal={IEEE Transactions on Systems, Man, and Cybernetics: Systems}, 
  title={{High Probability Convergence of Clipped Distributed Dual Averaging With Heavy-Tailed Noises}}, 
  year={2025},
  volume={55},
  number={4},
  pages={2624-2632},
  doi={10.1109/TSMC.2024.3525011}
}

@article{clipped-dsgd1,
title={{Online distributed optimization with clipped stochastic gradients: High probability bound of regrets}},
journal = {Automatica},
volume = {182},
pages = {112525},
year = {2025},
issn = {0005-1098},
doi = {https://doi.org/10.1016/j.automatica.2025.112525},
url = {https://www.sciencedirect.com/science/article/pii/S0005109825004200},
author = {Yuchen Yang and Kaihong Lu and Long Wang}
}

@article{clipped-dsgd2,
  title={{High Probability Convergence of Distributed Clipped Stochastic Gradient Descent with Heavy-tailed Noise}},
  author={Yang, Yuchen and Lu, Kaihong and Wang, Long},
  journal = {Systems \& Control Letters},
  volume = {209},
  pages = {106358},
  year = {2026},
  issn = {0167-6911},
  doi = {https://doi.org/10.1016/j.sysconle.2026.106358},
  url = {https://www.sciencedirect.com/science/article/pii/S0167691126000174}
}

@Article{Cannelli2020,
author={Cannelli, Loris
and Facchinei, Francisco
and Kungurtsev, Vyacheslav
and Scutari, Gesualdo},
title={{Asynchronous parallel algorithms for nonconvex optimization}},
journal={Mathematical Programming},
year={2020},
volume={184},
number={1},
pages={121-154},
issn={1436-4646},
doi={10.1007/s10107-019-01408-w},
url={https://doi.org/10.1007/s10107-019-01408-w}
}

@ARTICLE{chen-decent-pareto,
  author={Chen, Jianshu and Sayed, Ali H.},
  journal={IEEE Journal of Selected Topics in Signal Processing}, 
  title={{Distributed Pareto Optimization via Diffusion Strategies}}, 
  year={2013},
  volume={7},
  number={2},
  pages={205-220},
  doi={10.1109/JSTSP.2013.2246763}
}

@article{adaptive-learning-ali,
url = {http://dx.doi.org/10.1561/2200000051},
year = {2014},
volume = {7},
journal = {Foundations and Trends® in Machine Learning},
title={{Adaptation, Learning, and Optimization over Networks}},
doi = {10.1561/2200000051},
issn = {1935-8237},
number = {4-5},
pages = {311-801},
author = {Ali H. Sayed}
}

@ARTICLE{matta-diffusion-1,
  author={Matta, Vincenzo and Braca, Paolo and Marano, Stefano and Sayed, Ali H.},
  journal={IEEE Transactions on Information Theory}, 
  title={{Diffusion-Based Adaptive Distributed Detection: Steady-State Performance in the Slow Adaptation Regime}}, 
  year={2016},
  volume={62},
  number={8},
  pages={4710-4732},
  doi={10.1109/TIT.2016.2580665}
}

@ARTICLE{matta-difussion-2,
  author={Matta, Vincenzo and Braca, Paolo and Marano, Stefano and Sayed, Ali H.},
  journal={IEEE Transactions on Signal and Information Processing over Networks}, 
  title={{Distributed Detection Over Adaptive Networks: Refined Asymptotics and the Role of Connectivity}}, 
  year={2016},
  volume={2},
  number={4},
  pages={442-460},
  doi={10.1109/TSIPN.2016.2613682}
}

@ARTICLE{nassif-multitask,
  author={Nassif, Roula and Vlaski, Stefan and Richard, Cedric and Chen, Jie and Sayed, Ali H},
  journal={IEEE Signal Processing Magazine}, 
  title={{Multitask Learning Over Graphs: An Approach for Distributed, Streaming Machine Learning}}, 
  year={2020},
  volume={37},
  number={3},
  pages={14-25},
  doi={10.1109/MSP.2020.2966273}
}

@inproceedings{liu2024revisiting,
title={{Revisiting the Last-Iterate Convergence of Stochastic Gradient Methods}},
author={Zijian Liu and Zhengyuan Zhou},
booktitle={The Twelfth International Conference on Learning Representations},
year={2024},
url={https://openreview.net/forum?id=xxaEhwC1I4}
}

@ARTICLE{unified-refined,
  author={Alghunaim, Sulaiman A. and Yuan, Kun},
  journal={IEEE Transactions on Signal Processing}, 
  title={{A Unified and Refined Convergence Analysis for Non-Convex Decentralized Learning}}, 
  year={2022},
  volume={70},
  number={},
  pages={3264-3279},
  doi={10.1109/TSP.2022.3184770}
  }

@book{cvetkovic_rowlinson_simic_1997,
    place={Cambridge},
    series={Encyclopedia of Mathematics and its Applications},
    title={{Eigenspaces of Graphs}},
    DOI={10.1017/CBO9781139086547},
    publisher={Cambridge University Press},
    author={Cvetkovi\'{c}, Drago\v{s} and Rowlinson, Peter and Simi\'{c}, Slobodan},
    year={1997},
    collection={Encyclopedia of Mathematics and its Applications}
}

@book{chung1997spectral,
  title={{Spectral graph theory}},
  author={Chung, Fan R K},
  volume={92},
  year={1997},
  publisher={American Mathematical Society}
}

@article{XIAO200465,
title = {{Fast linear iterations for distributed averaging}},
journal = {Systems \& Control Letters},
volume = {53},
number = {1},
pages = {65-78},
year = {2004},
issn = {0167-6911},
doi = {https://doi.org/10.1016/j.sysconle.2004.02.022},
url = {https://www.sciencedirect.com/science/article/pii/S0167691104000398},
author = {Lin Xiao and Stephen Boyd}
}

@article{chang2011libsvm,
  title={{LIBSVM: A library for support vector machines}},
  author={Chang, Chih-Chung and Lin, Chih-Jen},
  journal={ACM transactions on intelligent systems and technology (TIST)},
  volume={2},
  number={3},
  pages={1--27},
  year={2011},
  publisher={Acm New York, NY, USA}
}

@InProceedings{pmlr-v235-gorbunov24a,
  title = {{High-Probability Convergence for Composite and Distributed Stochastic Minimization and Variational Inequalities with Heavy-Tailed Noise}},
  author = {Gorbunov, Eduard and Sadiev, Abdurakhmon and Danilova, Marina and Horv\'{a}th, Samuel and Gidel, Gauthier and Dvurechensky, Pavel and Gasnikov, Alexander and Richt\'{a}rik, Peter},
  booktitle = {Proceedings of the 41st International Conference on Machine Learning},
  pages = {15951--16070},
  year = {2024},
  volume = {235},
  series = {Proceedings of Machine Learning Research},
  publisher = {PMLR},
  url = 	 {https://proceedings.mlr.press/v235/gorbunov24a.html}
}

@inproceedings{deepak2021llm-training,
author = {Narayanan, Deepak and Shoeybi, Mohammad and Casper, Jared and LeGresley, Patrick and Patwary, Mostofa and Korthikanti, Vijay and Vainbrand, Dmitri and Kashinkunti, Prethvi and Bernauer, Julie and Catanzaro, Bryan and Phanishayee, Amar and Zaharia, Matei},
title = {Efficient large-scale language model training on GPU clusters using megatron-LM},
year = {2021},
isbn = {9781450384421},
publisher = {Association for Computing Machinery},
address = {New York, NY, USA},
url = {https://doi.org/10.1145/3458817.3476209},
doi = {10.1145/3458817.3476209},
booktitle = {Proceedings of the International Conference for High Performance Computing, Networking, Storage and Analysis},
articleno = {58},
numpages = {15},
location = {St. Louis, Missouri},
series = {SC '21}
}

@unpublished{li2023sgd-timeseries,
    author = {Li, Jiaqi and Lou, Zhipeng and Richter, Stefan and Wu, Wei-Biao},
    title = {{The Stochastic Gradient Descent from a Nonlinear Time Series Perspective}},
    year = {2023},
    note = {Preprint},
    url = {https://stat.math.uni-heidelberg.de/SGD_GMC-3.pdf}
}
\bibliographystyle{icml2026}

\newpage
\appendix
\onecolumn

\section{Introduction}

Appendix contains results omitted from the main body. Section \ref{sup:mse-vs-hp} discusses the differences between MSE and HP guarantees, Section \ref{sup:useful-results} collects some important facts used in our proofs, Section \ref{sup:technical} provides some technical results, Section \ref{sup:analysis-setup} defines some notions used in the analysis, Sections \ref{sup:non-conv} and \ref{sup:str-cvx} provide proofs for non-convex and strongly convex costs, respectively, Section \ref{sup:num} provides additional numerical results, Section \ref{sup:tran-time} derives the transient times, Section \ref{sup:lin-speed-up} provides a discussion on the novel results needed for achieving linear speed-up, Section \ref{sup:time-var-net} discusses the extension to time-varying networks, Section \ref{sup:dimension} contains further discussion on dimension dependence in our bounds and the notion of sub-Gaussianity used in our work, while Section \ref{sup:gen-step} provides additional results for strongly convex costs under a more general step-size schedule.

\section{On MSE and HP Results}\label{sup:mse-vs-hp}

Although both HP and MSE results establish convergence, the nature of these guarantees is very different. In particular, MSE results quantify the \emph{average behaviour across many runs} of an algorithm, while HP results quantify the behaviour of \emph{an individual run}. As mentioned briefly in the introduction, this distinction is very important in huge-scale applications like LLM training, where it is often impossible to perform more than a single training run, both resource- and time-wise. For example, \citet{deepak2021llm-training} estimate that the GPT-3 model with 175 billion parameters was trained for 34 days, while the variant with 1 trillion parameters can be trained for approximately 84 days. This is further exacerbated by phenomena like heavy-tailed noise, frequently observed during training of deep learning models and transformers, e.g., \cite{zhang2020adaptive,simsekli2019tail}, which can cause the performance of an individual run to significantly deviate from the average performance. Hence, strong guarantees with respect to an individual run are very important in such applications. While MSE results can be used to provide guarantees with respect to an individual run, they can be quite loose compared to HP bounds. In particular, if we have a MSE bound of the form $\E[X_t^2] = 1/t$, Chebyshev's inequality implies $\Prob\big(X_t^2 > \nicefrac{1}{\delta t}\big) \leq \delta$, which shows an inversely proportional dependence on the confidence level $\delta \in (0,1)$. On the other hand, HP results establish guarantees of the type $\mathbb{P}(X_t^2 > \log(\nicefrac{1}{\delta})/t) \leq \delta$, having a much milder, logarithmic dependence on $1/\delta$. Equivalently, for any threshold $\epsilon > 0$, MSE + Chebyshev implies $\mathbb{P}(X_t^2 > \epsilon) \leq \nicefrac{1}{t\epsilon}$, while HP gives $\mathbb{P}(X_t^2 > \epsilon) \leq \exp(-t\epsilon)$, providing a much sharper bound on the tail probability. Since our goal is to show that $X_t^2 \leq \epsilon$ (e.g., if $X_t = \|x^t-x^\star\|$, or if $X_t = \min_{k \in [t]}\|\nabla f(x^k)\|$), then $\Prob(X_t^2 > \epsilon)$ represents the \emph{probability of failure}, with MSE + Chebyshev implying that the probability of failure decays polynomially in $t$, while HP bound implies that the probability of failure decays exponentially fast in $t$, giving much stronger guarantees on the convergence of an individual run.

\section{Useful Inequalities}\label{sup:useful-results}

In this section we outline some well-known inequalities and results used in our proofs. We start with Jensen's inequality for convex/concave functions.

\begin{proposition}[Jensen's inequality]\label{prop:Jensen}
    Let $X \in \R$ be an integrable random variable. Then, for any convex function $h: \R \mapsto \R$, we have
    \begin{equation*}
        h(\E[X]) \leq \E[h(X)]. 
    \end{equation*} Moreover, if $h$ is concave, the reverse inequality holds, i.e., we then have
    \begin{equation*}
        \E[h(X)] \leq h(\E[X]).
    \end{equation*}
\end{proposition}

\begin{proposition}[Cauchy-Schwartz inequality]\label{prop:Cauchy-Schwartz}
    For any $a,b \in \R^d$, we have
    \begin{equation*}
        |\langle a,b \rangle| \leq \|a\|\|b\|. 
    \end{equation*}
\end{proposition}

As a consequence of the Cauchy-Schwartz inequality, we have the following result.

\begin{proposition}[Young's inequality]\label{prop:Young}
    For any $a,b \in \R$ and any $\epsilon > 0$, we have
    \begin{equation*}
        ab \leq \frac{\epsilon a^2}{2} + \frac{b^2}{2\epsilon}.
    \end{equation*} Furthermore, for any $\theta > 0$, we have
    \begin{equation*}
        (a+b)^2 \leq (1+\theta)a^2 + (1+\theta^{-1})b^2.
    \end{equation*}
\end{proposition}

Young's inequality is also known as the Peter-Paul inequality.

\begin{proposition}[H\"{o}lder's inequality]\label{prop:Holder}
    For any random variables $X,Y \in \R$ and any $p,q \in [1,\infty]$, such that $\frac{1}{p} + \frac{1}{q} = 1$, we have
    \begin{equation*}
        \E|XY| \leq \sqrt[p]{\E|X|^p}\sqrt[q]{\E|Y|^q}.
    \end{equation*}
\end{proposition}

The coefficients $p,q \in [1,\infty]$ are known as H\"{o}lder coefficients. Note that H\"{o}lder's inequality recovers Cauchy-Schwartz inequality for $p = q = 2$. We have the following useful consequence of H\"{o}lder's inequality.

\begin{proposition}\label{prop:gen-Holder}
    For any $n \in \N$ and random variables $X_i \in \R$, $i \in [n]$, we have
    \begin{equation*}
        \E\Big[\prod_{i \in [n]}|X_i|\Big] \leq \prod_{i \in [n]}\sqrt[n]{\E|X_i|^n}.
    \end{equation*}
\end{proposition}

Next, we state a useful result from \cite{ran-improved}.

\begin{proposition}[\cite{ran-improved}, Lemma 21]\label{prop:Ran}
    For any $c,t_0 > 0$ and $0 \leq a \leq b$, we have
    \begin{equation*}
        \prod_{k = a}^b\Big(1 - \frac{c}{k+t_0} \Big) \leq \frac{(a+t_0)^c}{(b+1+t_0)^c}.
    \end{equation*}
\end{proposition}

The following result states some important consequences of conditions \textbf{(A2)} and \textbf{(A3)}, see, e.g., \cite{bertsekas-gradient,nesterov-lectures_on_cvxopt}.

\begin{proposition}\label{prop:descent-ineq}
    Let \textup{\textbf{(A3)}} hold. Then, for any $i \in [n]$ and $x,y \in \R^d$, the following statements are true.
    \begin{enumerate}
         \item $f_i(x) \leq f_i(y) + \langle \nabla f_i(y), x-y\rangle + \frac{L}{2}\|x-y\|^2.$

         \item $f$ has $L$-Lipschitz continuous gradients.

         \item $f(x) \leq f(y) + \langle \nabla f(y), x-y\rangle + \frac{L}{2}\|x-y\|^2.$

         \item If in addition \textup{\textbf{(A2)}} holds, then for any $x \in \R^d$, we have $\|\nabla f(x)\|^2 \leq 2L(f(x)-f^\star)$.
    \end{enumerate}
\end{proposition}

Next, we state some important consequences of \textbf{(A6)}, see, e.g., \cite{nesterov-lectures_on_cvxopt}. 

\begin{proposition}\label{prop:str-cvx}
    Let \textup{(A6)} hold. Then, the following are true, for all $i \in [n]$ and $x \in \R^d$.
    \begin{enumerate}
        \item $\nabla^2f_i(x) \succeq \mu I_d$.
        \item $f$ is $\mu$-strongly convex.
        \item $\|\nabla f(x)\|^2 \geq 2\mu\big(f(x)-f^\star\big)$.
        \item $f(x) - f^\star \geq \frac{\mu}{2}\|x-x^\star\|^2$.
    \end{enumerate}
\end{proposition}

Finally, we provide a useful martingale concentration result from \cite{jin2019short}.

\begin{proposition}\label{prop:concentration}
    Let $X_1,\ldots,X_n \in \R^d$ be random vectors and denote by $\F_n = \sigma\big(\{X_1,\ldots,X_n\}\big)$ the corresponding filtration. If for all $i \in [n]$, the vector $X_i$ satisfies $\E[X_i\:\vert\:\F_{i-1}] = 0$ and 
    \begin{equation*}
        \Prob\big(\|X_i\| > \epsilon\:\vert\:\F_{i-1}\big) \leq 2\exp\bigg(-\frac{t^2}{2\sigma_i^2}\bigg),
    \end{equation*} for some $\F_{i-1}$-measurable $\sigma_i > 0$ and all $\epsilon > 0$, then for any $\delta \in (0,1)$, we have
    \begin{equation*}
        \Prob\Big(\|\sum_{i \in [n]}X_i\|^2 > 2\sqrt{2}\sum_{i \in [n]}\sigma_i^2\log(\nicefrac{2d}{\delta})\Big) \leq \delta.    
    \end{equation*}
\end{proposition}

\section{Technical Results}\label{sup:technical}

In this section we prove Lemmas \ref{lm:noise-properties}, \ref{lm:mgf-bound-str-cvx} and provide another technical result used in the proofs. For the reader's convenience, we restate Lemmas \ref{lm:noise-properties} and \ref{lm:mgf-bound-str-cvx} below.

{
\renewcommand{\thetheorem}{3.2}
\begin{lemma}
    If \textup{\textbf{(A3)}} holds, then the following are true for any $t \geq 1$, $i \in [n]$ and $\Ft$-measurable $v \in \R^d$.
    \begin{enumerate}
        \item $\E\lbr\exp\lp\langle v,\zit \rangle\rp \: \vert \: \Ft\rbr \leq \exp\lp\frac{3\sigma_i^2\|v\|^2}{4}\rp$.
        
        \item $\E\lbr\exp\lp\langle v, \ozt \rangle \rp \: \vert\: \Ft \rbr \leq \exp\lp\frac{3\sigma^2\|v\|^2}{4n} \rp$.

        \item $\E\lbr \exp\bigg(\frac{n\| \ozt \|^2}{15\sigma^2} \bigg) \: \big\vert \: \Ft \rbr \leq 2d\exp(1)$.
    \end{enumerate}
\end{lemma}
}
\begin{proof}
\begin{enumerate}[leftmargin=*]
    \item To prove the first property, we follow steps similar to those from Lemma 1 in \cite{li2020high}. Let $y_i = \frac{\zit}{\sigma_i}$ and note that $\E\lbr\exp\lp \|y_i\|^2\rp \: \vert \: \Ft \rbr \leq \exp(1)$, from \textbf{(A3)}. Assume first that $v \in \R^d$ is such that $\|v\| \leq \frac{4}{3}$. Using the inequality $\exp(a) \leq a + \exp(\nicefrac{9a^2}{16})$, which holds for any $a \in \R$, we then have
    \begin{align*}
        \E\lbr\exp\lp\langle v,y_i \rangle\rp \: \vert \: \Ft\rbr &\leq \E\lbr \langle v,y_i \rangle + \exp\lp\frac{9\langle v,y_i \rangle^2}{16}\rp \: \big\vert \: \Ft\rbr \stackrel{(a)}{=} \E\lbr\exp\lp\frac{9\langle v,y_i \rangle^2}{16}\rp \: \big\vert \: \Ft\rbr \\
        &\stackrel{(b)}{\leq} \E\lbr\exp\lp\frac{9\|v\|^2\|y_i\|^2 }{16}\rp \: \big\vert \: \Ft\rbr \stackrel{(c)}{\leq} \lp\E\lbr\exp\lp\|y_i\|^2\rp \: \vert \: \Ft\rbr\rp^{\nicefrac{9\|v\|^2}{16}} \\
        &\stackrel{(d)}{\leq} \exp\lp \frac{9\|v\|^2}{16} \rp \leq \exp\lp \frac{3\|v\|^2}{4} \rp,
    \end{align*} where $(a)$ follows from \textbf{(A3)} and since $v$ is $\Ft$-measurable, $(b)$ follows from Proposition \ref{prop:Cauchy-Schwartz}, $(c)$ follows from the fact that $\frac{9\|v\|^2}{16} \leq 1$ and Proposition \ref{prop:Jensen}, while $(d)$ follows from $\E[\exp(\|y_i\|^2)\:\vert\:\Ft] \leq \exp(1)$. Next, if $\|v\| > \frac{4}{3}$, we have 
    \begin{align*}
        \E[\exp(\langle v,y_i\rangle) \: \vert \: \Ft] &\leq \exp\lp \frac{3\|v\|^2}{8} \rp\E\lbr\exp\lp \frac{2\|y_i\|^2}{3} \rp \: \big\vert \: \Ft \rbr \\
        &\leq \exp\lp \frac{2}{3} + \frac{3\|v\|^2}{8} \rp \leq \exp\lp \frac{3\|v\|^2}{4} \rp,
    \end{align*} where the first inequality follows by applying Proposition \ref{prop:Young} with $\epsilon = \frac{4}{3}$ and the fact that $v$ is $\Ft$-measurable, the second follows from Proposition \ref{prop:Jensen} and $\E[\exp(\|y_i\|^2)\:\vert\:\Ft] \leq \exp(1)$, while the third inequality follows from the fact that $\frac{2}{3} < \frac{3\|v\|^2}{8}$, since $\|v\| > \frac{4}{3}$. Combining both cases, we get $\E[\exp\lp\langle v,y_i \rangle \rp \:\vert\:\Ft] \leq \exp\lp\frac{3\|v\|^2}{4}\rp$, for any $\Ft$-measurable vector $v \in \R^d$. The proof is completed by applying this inequality to $\langle v,\zit \rangle = \langle \sigma_iv,y_i\rangle$.

    \item Recall that $\ozt = \frac{1}{n}\sum_{i \in [n]}\zit$ and $\sigma^2 = \frac{1}{n}\sum_{i \in [n]}\sigma_i^2$. We then have, for any $\Ft$-measurable $v \in \R^d$
    \begin{align*}
        \E[\exp(\langle v,\ozt \rangle] &= \E\bigg[\exp\bigg(\frac{1}{n}\sum_{i \in [n]}\langle v,\zit \rangle\bigg)\bigg] \stackrel{(a)}{=} \prod_{i \in [n]}\E\bigg[\exp\bigg(\Big\langle \frac{v}{n},\zit \Big\rangle\bigg)\bigg] \\
        &\stackrel{(b)}{\leq} \prod_{i \in [n]}\exp\bigg(\frac{3\sigma_i^2\|v\|^2}{4n^2} \bigg) \stackrel{(c)}{=} \exp\bigg(\frac{3\sigma^2\|v\|^2}{4n} \bigg),
    \end{align*} where $(a)$ follows from the fact that, conditioned on $\Ft$, the noise across users is independent according to \textbf{(A3)}, $(b)$ follows from the first part of the proof, while $(c)$ follows from the definition of $\sigma^2$.  
    
    \item From \textbf{(A4)} we know that conditioned on $\Ft$, the vectors $z_1^t,\ldots,z_n^t$ are independent and zero-mean. Moreover, for each $i \in [n]$ and any $\epsilon > 0$, we have
    \begin{align*}
        \Prob\big(\|\zit\| > &\epsilon \: \vert \: \Ft\big) = \Prob\bigg(\exp\bigg(\frac{\|\zit\|^2}{2\sigma_i^2}\bigg) > \exp\bigg(\frac{\epsilon^2}{2\sigma_i^2}\bigg) \: \Big\vert \: \Ft\bigg) \stackrel{(a)}{\leq} \exp\bigg(-\frac{\epsilon^2}{2\sigma_i^2}\bigg)\E\bigg[\exp\bigg(\frac{\|\zit\|^2}{2\sigma_i^2}\bigg) \: \Big\vert \: \Ft\bigg] \\
        &\stackrel{(b)}{\leq} \exp\bigg(-\frac{\epsilon^2}{2\sigma_i^2}\bigg)\Bigg(\E\bigg[\exp\bigg(\frac{\|\zit\|^2}{\sigma_i^2}\bigg)\:\Big\vert\:\Ft\bigg]\Bigg)^{1/2} \stackrel{(c)}\leq \exp\bigg(\frac{1}{2} - \frac{\epsilon^2}{2\sigma_i^2}\bigg) \stackrel{(d)}{\leq} 2\exp\bigg(-\frac{\epsilon^2}{2\sigma_i^2}\bigg), 
    \end{align*} where $(a)$ follows from Markov's inequality, $(b)$ follows from Proposition \ref{prop:Jensen}, in $(c)$ we use \textbf{(A4)}, while $(d)$ follows form the fact that $\exp(1/2) \leq 2$. Therefore, vectors $z_1^t,\ldots,z_n^t$ satisfy the conditions of Proposition \ref{prop:concentration} with respect to $\Ft$, and it follows that, for any $\delta \in (0,1)$
    \begin{equation*}
        \Prob\bigg(\|\ozt\|^2 > \frac{2\sqrt{2}\sigma^2\log(\nicefrac{2d}{\delta})}{n} \bigg) \leq \delta,
    \end{equation*} or equivalently, for any $\epsilon > 0$, we have
    \begin{equation}\label{eq:avg-noise-tail}
        \Prob\big(\|\ozt\|^2 > \epsilon\big) \leq 2d\exp\bigg(-\frac{n\epsilon}{2\sigma^2\sqrt{2}} \bigg) \leq 2d\exp\bigg(-\frac{n\epsilon}{3\sigma^2}\bigg).
    \end{equation} Next, define $Y = \frac{\sqrt{n}\|\ozt\|}{\sigma\sqrt{3}}$. Using the layer-cake formula for the expectation of a non-negative random variable, we then have, for any integer $p \geq 1$
    \begin{align}\label{eq:moment-bdd}
        \E\big[\|Y\|^{2p} \big] &= \int_{0}^\infty \Prob\big(\|Y\|^{2p} > y\big) dy \stackrel{(a)}{=} p\int_0^\infty s^{p-1}\Prob\big(\|Y\|^2 > s \big)ds \nonumber \\
        &= p\int_0^\infty s^{p-1}\Prob\bigg(\|\ozt\|^2 > \frac{3\sigma^2 s}{n} \bigg)ds \stackrel{(b)}{\leq} 2dp\int_0^\infty s^{p-1} \exp(-s)ds \stackrel{(c)}{\leq} 2dp^{p+1},
    \end{align} where $(a)$ follows by introducing the substitution $y = s^p$, $(b)$ follows from \eqref{eq:avg-noise-tail}, while $(c)$ follows from the definition of $\Gamma$ function and the fact that $\Gamma(p) \leq p^p$. For any $c > 0$, we then have
    \begin{align*}
        \E\bigg[\exp\bigg(\frac{cn\|\ozt\|^2}{3\sigma^2}\bigg)\bigg] &= 1 + \sum_{k = 1}^\infty\frac{c^k\E\big[\|Y\|^{2k}\big]}{k!} \stackrel{(i)}{\leq} 1 + \sum_{k = 1}^{\infty}\frac{2dk(ck)^k}{k!} \stackrel{(ii)}{\leq} 2d\bigg(1 + \frac{1}{e}\sum_{k = 1}^{\infty}k(ec)^k \bigg),
    \end{align*} where $(i)$ follows from \eqref{eq:moment-bdd}, while $(ii)$ follows from $k! \geq e(k/e)^k$. Next, note that the sequence $\sum_{k = 1}^\infty kp^k$ converges if $p < 1$, with $\sum_{k = 1}^{\infty}kp^k = \frac{p}{(1-p)^2}$, therefore, choosing $c < \frac{1}{e}$ in the relation above, we get
    \begin{align}\label{eq:semi-final}
        \E\bigg[\exp\bigg(\frac{cn\|\ozt\|^2}{3\sigma^2}\bigg)\bigg] \leq 2d\bigg(1 + \frac{c}{(1-ec)^2}\bigg).
    \end{align} It can be verified that $\frac{c}{(1-ec)^2} \leq e-1$ for any $c \leq \frac{e-1}{e^2}$. Choosing $c = \frac{1}{5}$ and plugging into \eqref{eq:semi-final} completes the proof.
\end{enumerate}
\end{proof}

We next restate Lemma \ref{lm:mgf-bound-str-cvx}.

{
\renewcommand{\thetheorem}{3.6}
\begin{lemma}
    Let $\{X^t\}_{t \geq 2}$ be a sequence of random variables initialized by a deterministic $X^1 > 0$, such that, for some $M \in \N$, $a,t_0,C_i > 0$, $i \in [M]$ and every $t \geq 1$ 
    \begin{equation}\label{eq:almost-decrease}
        \E[\exp(X^{t+1})] \leq \E\bigg[\exp\bigg(\Big(1-\frac{a}{t+t_0}\Big)X^t + \sum_{i \in [M]}\frac{C_i}{(t+t_0)^i} \bigg)\bigg].
    \end{equation} If $a \in (1,2]$ and $t_0 \geq a$, we then have
    \begin{align*}
        \E[\exp(X^{t+1})] &\leq \exp\bigg(\frac{(t_0+1)^aX_1}{(t+1+t_0)^a} + \frac{2^aC_1}{a} + \frac{2^aC_3\log(t+1+t_0)}{(t+1+t_0)^a} \bigg) \\
        &\times \exp\bigg( \frac{2^aC_2/(a-1)}{t+1+t_0} + \sum_{j = 4}^M\frac{2^at_0^{3-j}C_j}{(j-3)(t+1+t_0)^a} \bigg).
    \end{align*} 
\end{lemma}
}

\begin{proof}
    Starting from \eqref{eq:almost-decrease}, taking the logarithm, defining $Y_{t} \coloneqq \log\E[\exp(X^t)]$ and $b_t = 1-\frac{a}{t+t_0}$, we then have
    \begin{align}\label{eq:useful-recursion}
        Y^{t+1} \leq \sum_{i \in [M]}\frac{C_i}{(t+t_0)^i} + \log\E[\exp(b_tX^t)] &\leq \sum_{i \in [M]}\frac{C_i}{(t+t_0)^i} + \log\bigg[\Big(\E[\exp(X^t)]\Big)^{b_t}\bigg] \nonumber \\ 
        &= b_tY^t + \sum_{i \in [M]}\frac{C_i}{(t+t_0)^i},
    \end{align} where the second inequality follows from the fact that $b_t \in (0,1)$ and Proposition \ref{prop:Jensen}. Unrolling the recursion \eqref{eq:useful-recursion} and noting that $Y_1 = X_1$, since $X_1 > 0$ is deterministic, we get
    \begin{align}\label{eq:technical-intermed}
        Y^{t+1} &\leq X^1\prod_{k \in [t]}b_k + \sum_{i \in [M]}C_i\sum_{k \in [t]}\frac{1}{(k+t_0)^i}\prod_{s = k + 1}^tb_s \nonumber \\ 
        &\leq \frac{(t_0+1)^aX^1}{(t+1+t_0)^a} + \sum_{i \in [M]}C_i\sum_{k \in [t]}\frac{1}{(k+t_0)^i}\times\frac{(k+1+t_0)^a}{(t+1+t_0)^a} \nonumber \\
        &\leq \frac{(t_0+1)^aX^1}{(t+1+t_0)^a} + \sum_{i \in [M]}\frac{2^aC_i}{(t+1+t_0)^a}\sum_{k \in [t]}(k+t_0)^{a-i},
    \end{align} where the second inequality follows from Proposition \ref{prop:Ran}, while the third inequality follows from the fact that $\Big(\frac{k+1+t_0}{k+t_0}\Big)^a \leq 2^a$. We now proceed to analyze $\sum_{k \in [t]}(k+t_0)^{a-i}$, for different values of $i \in [M]$. Using Darboux sums and the fact that $a \in (1,2]$, it can be readily verified that
    \begin{equation}\label{eq:technical-cases}
        \sum_{k \in [t]}(k+t_0)^{a-i} \leq \begin{cases}
            \frac{(t+1+t_0)^a}{a}, & i = 1 \\
            \frac{(t+1+t_0)^{a-1}}{a-1}, & i = 2 \\
            \ln(t+t_0+1), & i = 3 \\
            \frac{t_0^{3-i}}{i-3}, & i \geq 4
        \end{cases}.
    \end{equation} Plugging \eqref{eq:technical-cases} into \eqref{eq:technical-intermed}, we get
    \begin{align*}
        Y^{t+1} &\leq \frac{(t_0+1)^aX^1}{(t+1+t_0)^a} + \frac{2^aC_1}{a} + \frac{2^aC_2}{(a-1)(t+t_0+1)}  \frac{2^aC_3\log(t+t_0+1)}{(t+1+t_0)^a} + \sum_{j = 4}^M\frac{2^at_0^{3-j}C_j}{(j-3)(t+1+t_0)^a}.
    \end{align*} Taking the exponent on both sides completes the proof.
\end{proof}

To complete this section, we provide a useful technical result.

\begin{lemma}\label{lm:technical-result}
    Let $\lambda \in [0,1)$ and $\alpha_t = \frac{a}{(t+t_0)^c}$, where $a,t_0 > 0$ and $c \geq \nicefrac{1}{2}$. If $t_0 \geq \frac{2c-1+\lambda}{1-\lambda}$, then for any $t \geq 1$, we have
    \begin{equation*}
        \sum_{k \in [t]}\alpha_k\lambda^{t-k} \leq \frac{3\alpha_t}{1-\lambda}.
    \end{equation*}
\end{lemma}

\begin{proof}
    Using the definition of $\alpha_t$, we note that
    \begin{equation*}
        \sum_{k \in [t]}\alpha_k\lambda^{t-k} = \alpha_t\sum_{k \in [t]}\frac{\alpha_k}{\alpha_t}\lambda^{t-k} = \alpha_t\sum_{k \in [t]}\lambda^{t-k}\bigg(\frac{t+t_0}{k+t_0}\bigg)^c = \alpha_t\sum_{k \in [t]}\lambda^{t-k}\bigg(1 + \frac{t-k}{k+t_0}\bigg)^c.  
    \end{equation*} Next, denote by $\tlambda \coloneqq 1 - \lambda \in (0,1]$ and use the substitution $s = t-k$, to get
    \begin{align}\label{eq:tech-res-proof1}
        \sum_{k \in [t]}\alpha_k\lambda^{t-k} &= \alpha_t\sum_{s = 0}^{t-1}(1 - \tlambda)^{s}\bigg(1 + \frac{s}{t-s+t_0}\bigg)^c \leq \alpha_t\sum_{s =0}^{t-1}\exp\bigg(- \tlambda s + \frac{cs}{t-s+t_0} \bigg) \nonumber \\
        &\leq \alpha_t\sum_{s = 0}^{t-1}\exp\bigg(-s\bigg(\tlambda - \frac{c}{1+t_0} \bigg)\bigg) \leq \alpha_t\sum_{s = 0}^{t-1}\exp\bigg(-\frac{\tlambda s}{2}\bigg),
    \end{align} where we used $1 \pm x \leq \exp(\pm x)$ in the first, the fact that $t-s \geq 1$ in the second and the choice of $t_0$ in the third inequality. Next, we use Darboux sums, to get
    \begin{equation}\label{eq:tech-res-proof2}
        \sum_{s = 0}^{t-1}\exp\bigg(-\frac{\tlambda s}{2}\bigg) = 1 + \sum_{s = 1}^{t-1}\exp\bigg(-\frac{\tlambda s}{2}\bigg) \leq 1 + \int_0^{t-1}\exp\bigg(-\frac{\tlambda s}{2}\bigg)ds \leq 1 + \frac{2}{\tlambda} \leq \frac{3}{\tlambda}.
    \end{equation} Plugging \eqref{eq:tech-res-proof2} into \eqref{eq:tech-res-proof1} completes the proof.
\end{proof}

\section{Analysis Setup}\label{sup:analysis-setup}

In this section we define some notation useful for the analysis. To begin, let $\git \coloneqq \nabla \ell(\xit;\xi_i^t)$ denote the stochastic gradient of user $i$ at time $t$. We can then represent the update rule \eqref{eq:dsgd-update} as
\begin{equation}\label{eq:dsgd-update-g}
    \xitp = \sum_{j \in \calN_i}w_{ij}\big(\xit - \alpha_t\git\big).
\end{equation} It can immediately be seen that $\zit = \git - \nabla f_i(\xit)$. Next, define the network average stochastic gradient $\ogt \coloneqq \frac{1}{n}\sum_{i \in [n]}\git$ and $\onab f_t \coloneqq \frac{1}{n}\sum_{i \in [n]}\nabla f_i(\xit)$, which represents the network average of user's gradients evaluated at their local models. It can then be seen that $\ogt = \onab f_t + \ozt$. Combining the definition of the network average model, the fact that the weight matrix is doubly stochastic and \eqref{eq:dsgd-update-g}, it follows that $\oxtp = \oxt - \alpha_t\ogt$. 

We now introduce some notation useful for the analysis of decentralized methods, see, e.g., \cite{nedic-subgrad,jakovetic-unification,ran-improved}. Let $\bx^t \coloneq col(x_1^t,\ldots,x_n^t) \in \R^{nd}$ denote the column vector stacking users' local models. Using this notation, we can then represent the update rule \eqref{eq:dsgd-update-g} compactly as 
\begin{equation}\label{eq:dsgd-vector}
    \bx^{t+1} = \bW(\bx^t - \alpha_t \bg^t), 
\end{equation} where $\bW = W \otimes I_d \in \R^{nd \times nd}$, $\otimes$ denotes the Kronecker product and $I_d \in \R^{d \times d}$ denotes the $d$-dimensional identity matrix, while $\bg^t \coloneqq col(g^t_1,\ldots,g^t_n)$. Next, define the matrix $J \coloneqq \frac{1}{n}\mathbf{1}_n\mathbf{1}_n^\top \in \R^{n \times n}$, where $\mathbf{1}_n \in \R^n$ is the vector of all ones. The matrix $J$ represent the ``ideal'' consensus matrix, where all users can communicate with each other.\footnote{This is equivalent to the client-server setup in terms of the update rule, as the server averages models from all clients in each iteration.} The interaction between matrices $W$ and $J$ represents an important part of any decentralized algorithm and we now list some known properties, see, e.g., \cite{jakovetic-unification,ran-improved} and references therein.

\begin{proposition}\label{prop:mixing-mx}
    Let \textup{\textbf{(A1)}} hold. Then, the following are true.
    \begin{enumerate}
        \item $W\mathbf{1}_n = J\mathbf{1}_n = \mathbf{1}_n$.
        \item $\lambda \coloneqq \|W - J\| \in [0,1)$.
        \item $WJ = JW = J$.
    \end{enumerate}
\end{proposition}

Next, define $\obxt \coloneqq \mathbf{1}_n \otimes \oxt \in \R^{nd}$, $\bJ \coloneqq J \otimes I_d \in \R^{nd \times nd}$ and note that $\obxt = \bJ\bxt$. Combined with $\eqref{eq:dsgd-vector}$, it follows that $\obxtp = \obxt - \alpha_t\obgt$. Denoting by $\tbW \coloneqq \bW - \bJ$, it follows from \eqref{eq:dsgd-vector} that
\begin{equation}\label{eq:dsgd-consensus-diff}
    \bxtp - \obxtp = \bW(\bxt - \alpha\bgt) - \bJ(\bxt - \alpha_t\bgt) = \tbW(\bxt - \alpha\bgt) = \tbW(\bxt - \obxt - \alpha\bgt).
\end{equation} Finally, recall that we denote the unique global minima of strongly convex costs by $x^\star \in \R^d$. We define two related concepts, namely the column stacking of $x^\star$, i.e., $\bxs \coloneqq \mathbf{1}_n \otimes x^\star$ and the stacking of users' gradients evaluated at the global optima, i.e., $\nbfs \coloneqq col(\nabla f_1(x^\star),\ldots,\nabla f_n(x^\star))$. Note that $\|\nbfs\|^2 = \sum_{i \in [n]}\|\nabla f_i(x^\star)\|^2$ is a useful measure of heterogeneity.

\section{Proofs for Non-convex Costs}\label{sup:non-conv}

In this section we prove Lemmas \ref{lm:descent-inequality} and \ref{lm:consensus-bound}, as well as Theorem \ref{thm:main-non-cvx}. For the reader's convenience, we restate the results below, starting with Lemma \ref{lm:descent-inequality}.

{
\renewcommand{\thetheorem}{3.1}
\begin{lemma}
    Let \textup{\textbf{(A3)}} hold. If $\alpha_t \leq \frac{1}{2L}$, we have
    \begin{align*}
        f(\oxtp) \leq f(\oxt) - \frac{\alpha_t}{2}\| \nabla f(\oxt)\|^2 - \alpha_t\langle \nabla f(\oxt), \ozt \rangle + \alpha_t^2L\|\ozt\|^2 + \frac{\alpha_t L^2}{2n}\sum_{i \in [n]}\|\xit - \oxt\|^2. 
    \end{align*}
\end{lemma}
}

\begin{proof}
    The proof follows similar steps as in, e.g., Lemma 3 from \cite{ran-improved}. Starting from Proposition \ref{prop:descent-ineq} and averaging across all costs $i \in [n]$, it readily follows that, for all $x,y \in \R^d$
    \begin{equation}\label{eq:descent-ineq}
        f(x) \leq f(y) + \langle \nabla f(y), x - y \rangle + \frac{L}{2}\|x - y\|^2.
    \end{equation} Setting $x = \oxtp$ and $y = \oxt$ in \eqref{eq:descent-ineq}, we get
    \begin{align}\label{eq:descent-intermed}
        f(\oxtp) &\leq f(\oxt) - \alpha_t\langle \nabla f(\oxt), \og_t \rangle + \frac{\alpha_t^2L}{2}\|\og_t\|^2 \nonumber \\ 
        &\stackrel{(a)}{\leq} f(\oxt) - \alpha_t\langle \nabla f(\oxt), \onab f_t \rangle - \alpha_t\langle \nabla f(\oxt), \ozt \rangle + \alpha_t^2L\|\onab f_t\|^2 +\alpha_t^2L\|\ozt\|^2 \nonumber \\ 
        &\stackrel{(b)}{=} f(\oxt) - \frac{\alpha_t}{2}\|\nabla f(\oxt)\|^2 - (1 - 2\alpha_t L)\frac{\alpha_t}{2}\|\onab f_t\|^2 \\ 
        &\quad\quad\quad\quad + \frac{\alpha_t}{2}\|\onab f_t - \nabla f(\oxt)\|^2 - \alpha_t\langle \nabla f(\oxt), \ozt \rangle + \alpha_t^2L\|\ozt\|^2 \nonumber \\
        &\stackrel{(c)}{\leq} f(\oxt) - \frac{\alpha_t}{2}\|\nabla f(\oxt)\|^2 + \frac{\alpha_t}{2}\|\onab f_t - \nabla f(\oxt)\|^2 - \alpha_t\langle \nabla f(\oxt), \ozt \rangle + \alpha_t^2L\|\ozt\|^2,
    \end{align} where $(a)$ follows by applying Proposition \ref{prop:Young} with $\theta = 1$, $(b)$ follows from the identity $\langle a,b \rangle = \frac{1}{2}\big(\|a\|^2 + \|b\|^2 - \|a-b\|^2 \big)$, while $(c)$ follows from the fact that $\alpha_t \leq \frac{1}{2L}$. Recalling the definition of $\onab f_t$, we get
    \begin{equation}\label{eq:bound-grad-diff}
        \|\onab f_t - \nabla f(\oxt)\|^2 = \Big\|\frac{1}{n}\sum_{i \in [n]}\big[\nabla f_i(\xit) - \nabla f_i(\oxt)\big]\Big\|^2 \leq \frac{L^2}{n}\sum_{i \in [n]}\|\xit - \oxt\|^2,
    \end{equation} where we used Proposition \ref{prop:Jensen} and \textbf{(A3)} in the last inequality. Plugging \eqref{eq:bound-grad-diff} in \eqref{eq:descent-intermed} gives the desired result.
\end{proof} 

Next, we restate Lemma \ref{lm:consensus-bound}.

{
\renewcommand{\thetheorem}{3.3}
\begin{lemma}
    Let \textup{\textbf{(A1)}} and \textup{\textbf{(A5)}} hold. We then have, for any $t \geq 1$
    \begin{align*}
        \frac{1}{n}\sum_{i \in [n]}&\|\xitp - \oxtp\|^2 \leq 2\lambda^{2t}\Delta_x + \frac{4\lambda^2}{n(1-\lambda)}\sum_{k \in [t]}\alpha_k^2\lambda^{t-k}\sum_{i \in [n]}\|\zik\|^2 \\
        &+ \frac{4\lambda^2A^2}{1-\lambda}\sum_{k \in [t]}\alpha_k^2\lambda^{t-k} + \frac{4\lambda^2B^2}{n(1-\lambda)}\sum_{k \in [t]}\alpha_k^2\lambda^{t-k}\sum_{i \in [n]}\|\nabla f(\xik)\|^2. 
    \end{align*}
\end{lemma} 
}

\begin{proof}
    We start by noting that $\sum_{i \in [n]}\|\xitp - \oxtp\|^2 = \|\bxtp - \obxtp\|^2$. Next, starting from \eqref{eq:dsgd-consensus-diff} and unrolling the recursion, we get
    \begin{align*}
        \bxtp - \obxtp = \tbW(\bxt - \obxt) - \alpha_t\tbW\bgt = \ldots = \tbW^{t}(\bx^1 - \obx^1) - \sum_{k \in [t]}\alpha_k\tbW^{t+1-k}\bgk, 
    \end{align*} Using Proposition \ref{prop:mixing-mx}, it follows that 
    \begin{equation*}
        \|\bxtp - \obxtp \| \leq \|\tbW^t(\bx^1-\obx^1)\| + \sum_{k \in [t]}\alpha_k\|\tbW^{t+1-k}\bgk\| \leq \lambda^t\|\bx^1 - \obx^1\| + \lambda\sum_{k \in [t]}\alpha_k\lambda^{t-k}\|\bgk\|, 
    \end{equation*} where we used the fact that $\|\tbW^k\| = \|\tbW\|^k$, for any integer $k \geq 0$. This readily implies
    \begin{align}
        \|\bxtp - \obxtp\|^2 &\leq 2\lambda^{2t}\|\bx^1 - \obx^1\|^2 + 2\lambda^2\lp\sum_{k \in [t]}\alpha_k\lambda^{t-k}\|\bg_k\|\rp^2 \nonumber \\
        &\leq 2\lambda^{2t}\|\bx^1 - \obx^1\|^2 + 2\lambda^2\sum_{s \in [t]}\lambda^{t-s}\sum_{k \in [t]}\alpha_k^2\lambda^{t-k}\|\bgk\|^2, \label{eq:consensus-bd-intermed}
    \end{align} where in the second inequality we use Proposition \ref{prop:Cauchy-Schwartz}, with $a = [a_1,\ldots,a_t]^\top$ and $b = [b_1,\ldots,b_t]^\top$, setting $a_k = \lambda^{(t-k)/2}$ and $b_k = \alpha_k\lambda^{(t-k)/2}\|\bgk\|$. Next, consider the quantity $\|\bgk\|^2$. Using the fact that $\|\bgk\|^2 = \sum_{i \in [n]}\|\gik\|^2$, we get
    \begin{align}
        \|\bgk\|^2 &= \sum_{i \in [n]}\|\gik\|^2 \stackrel{(a)}{\leq} 2\sum_{i \in [n]}\|\zik\|^2 + 2\sum_{i \in [n]}\|\nabla f_i(\xik)\|^2 \nonumber \\
        &\stackrel{(b)}{\leq} 2\sum_{i \in [n]}\|\zik\|^2 + 2nA^2 + 2B^2\sum_{i \in [n]}\|\nabla f(\xik)\|^2, \label{eq:grad-bound}  
    \end{align} where $(a)$ follows from $\gik = \nabla f_i(\xik) + \zik$ and using Proposition \ref{prop:Young} with $\theta = 1$, while $(b)$ follows from \textbf{(A5)}. Plugging \eqref{eq:grad-bound} into \eqref{eq:consensus-bd-intermed}, using the fact that $\sum_{k \in [t]}\lambda^{t-k} = \sum_{k = 0}^{t-1}\lambda^k \leq \frac{1}{1 - \lambda}$, the desired result follows.
\end{proof}

We are now ready to prove Theorem \ref{thm:main-non-cvx}, which we restate next, for convenience.

{
\renewcommand{\thetheorem}{3.4}
\begin{theorem}
    Let \textup{\textbf{(A1)}-\textbf{(A5)}} hold. If for any $T \geq 1$, the step-size is chosen such that $\alpha_T \equiv \alpha = \min\lcb C, \frac{\sqrt{n}}{\sigma\sqrt{15LT}} \rcb$, where $C > 0$ is a problem related constant satisfying $C \leq \min\lcb \frac{1}{2L}, \frac{n}{9\sigma^2}, \frac{1-\lambda}{\lambda LB\sqrt{48}}, \frac{\sqrt{n}}{3\sigma\sqrt{5L}}, \frac{\sqrt[3]{n}(1-\lambda)^{2/3}}{\osigma^{2/3}\lambda^2L^{2/3}\sqrt[3]{9}}\rcb$, then for any $\delta \in (0,1)$, with probability at least $1 - \delta$
    \begin{align*}
        \frac{1}{nT}&\sum_{t \in [T]}\sum_{i \in [n]}\|\nabla f(\xit)\|^2 = \bigO\Bigg( \frac{\sigma\sqrt{L}\big(\Delta_f + \log(\nicefrac{2d}{\delta})\big)}{\sqrt{nT}} + \frac{\Delta_f + \log(\nicefrac{1}{\delta})}{CT} + \frac{\Delta_xL^2}{(1-\lambda^2)T} + \frac{n\lambda^2L(A^2 + \sigma^2)}{\sigma^2(1-\lambda)^2T} \Bigg) .
    \end{align*}
\end{theorem}
}

\begin{proof}
    Using Lemma \ref{lm:descent-inequality}, rearranging and summing up the first $T$ terms, we get
    \begin{equation*}
        \sum_{t \in [T]}\frac{\alpha_t}{2}\| \nabla f(\oxt)\|^2 \leq \Delta_f - \sum_{t \in [T]}\alpha_t\langle \nabla f(\oxt), \ozt \rangle + L\sum_{t \in [T]}\alpha_t^2\|\ozt\|^2 + \frac{L^2}{2n}\sum_{t \in [T]}\alpha_t\sum_{i \in [n]}\|\xit - \oxt\|^2.
    \end{equation*} Next, to offset the effect of the inner product term, we subtract $\frac{9\sigma^2}{4n}\sum_{t \in [T]}\alpha_t^2\|\nabla f(\oxt)\|^2$ from both sides of the above inequality and note that, choosing $\alpha_t \leq \frac{n}{9\sigma^2}$, we have $\frac{\alpha_t}{2}-\frac{9\alpha_t^2\sigma^2}{4n} \geq \frac{\alpha_t}{4}$. Therefore, we obtain 
    \begin{align}
        \sum_{t \in [T]}\frac{\alpha_t}{4}\| \nabla f(\oxt)\|^2 &\leq \Delta_f + L\sum_{t \in [T]}\alpha_t^2\|\ozt\|^2 + \frac{L^2}{2n}\sum_{t \in [T]}\alpha_t\sum_{i \in [n]}\|\xit - \oxt\|^2 \nonumber \\
        &- \sum_{t \in [T]}\alpha_t\Big(\langle \nabla f(\oxt), \ozt \rangle + \frac{9\alpha_t\sigma^2}{4n}\|\nabla f(\oxt)\|^2 \Big). \label{eq:mgf-prelude}
    \end{align} Using Proposition \ref{prop:Young} with $\theta = 1$ and Lipschitz continuity of gradients of $f$, it readily follows that 
    \begin{equation*}
        \|\nabla f(\xit)\|^2 \leq 2\|\nabla f(\oxt)\|^2 + 2L^2\|\oxt - \xit\|^2,
    \end{equation*} implying 
    \begin{equation}\label{eq:grad-lower-bd}
        \|\nabla f(\oxt)\|^2 \geq \frac{1}{2n}\sum_{i \in [n]}\|\nabla f(\xit)\|^2 - \frac{L^2}{n}\sum_{i \in [n]}\|\oxt - \xit\|^2.
    \end{equation} Plugging \eqref{eq:grad-lower-bd} into \eqref{eq:mgf-prelude} and rearranging, we get
    \begin{align*}
        \frac{1}{8n}\sum_{t \in [T]}\sum_{i \in [n]}\alpha_t\|\nabla  f(\xit)\|^2 &\leq \Delta_f - \sum_{t \in [T]}\alpha_t\Big(\langle \nabla f(\oxt), \ozt \rangle + \frac{9\alpha_t\sigma^2}{4n}\|\nabla f(\oxt)\|^2 \Big) \\ 
        &+ L\sum_{t \in [T]}\alpha_t^2\|\ozt\|^2 + \frac{3L^2}{4n}\sum_{t \in [T]}\alpha_t\sum_{i \in [n]}\|\xit - \oxt\|^2. 
    \end{align*} Using Lemma \ref{lm:consensus-bound}, we then have
    \begin{align}
        &\frac{1}{8n}\sum_{t \in [T]}\sum_{i \in [n]}\alpha_t\|\nabla  f(\xit)\|^2 \leq \Delta_f - \sum_{t \in [T]}\alpha_t\Big(\langle \nabla f(\oxt), \ozt \rangle + \frac{9\alpha_t\sigma^2}{4n}\|\nabla f(\oxt)\|^2 \Big) \nonumber + L\sum_{t \in [T]}\alpha_t^2\|\ozt\|^2 \nonumber \\ &\:+ \frac{3L^2}{4}\sum_{t \in [T]}\alpha_t\Big(2\lambda^{2(t-1)}\Delta_x + \frac{4\lambda^2}{n(1-\lambda)}\sum_{k \in [t-1]}\alpha_k^2\lambda^{t-1-k}\sum_{i \in [n]}\Big[\|\zik\|^2 + A^2 + B^2\|\nabla f(\xik)\|^2\Big]\Big). \label{eq:mgf-prelude2} 
    \end{align} Next, note that, for any sequence $\{a_t\}_{t \in [T]}$, the following identity holds
    \begin{equation*}
        \sum_{t \in [T]}\sum_{k \in [t-1]}\lambda^{t-1-k}a_k = \sum_{t \in [T]}a_t\sum_{k \in [T-t]}\lambda^{k-1},
    \end{equation*} implying that $\sum_{t \in [T]}\sum_{k \in [t-1]}\lambda^{t-1-k}a_k \leq \frac{1}{1-\lambda}\sum_{t \in [T]}a_t$, if $\{a_t\}_{t \in [T]}$ is non-negative. Since the step-size is non-increasing, applying the previous identity to $\sum_{t \in [T]}\alpha_t\sum_{k \in [t-1]}\alpha_k^2\lambda^{t-1-k}\big[\|\zik\|^2 + A^2 + B^2\|\nabla f(\xik)\|^2\big]$, we get
    \begin{align}\label{eq:mgf-prelude3}
        &\sum_{t \in [T]}\alpha_t\sum_{k \in [t-1]}\alpha_k^2\lambda^{t-1-k}\Big[\|\zik\|^2 + A^2 + B^2\|\nabla f(\xik)\|^2\Big] \nonumber \\
        &\leq \sum_{t \in [T]}\sum_{k \in [t-1]}\alpha_k^3\lambda^{t-1-k}\Big[\|\zik\|^2 + A^2 + B^2\|\nabla f(\xik)\|^2\Big] \nonumber \\ 
        &\leq \frac{1}{1 - \lambda}\sum_{t \in [T]}\alpha_t^3\Big[\|\zit\|^2 + A^2 + B^2\|\nabla f(\xit)\|^2\Big].
    \end{align} Plugging \eqref{eq:mgf-prelude3} into \eqref{eq:mgf-prelude2}, it follows that
    \begin{align*}
        \frac{1}{8n}&\sum_{t \in [T]}\sum_{i \in [n]}\alpha_t\|\nabla  f(\xit)\|^2 \leq \Delta_f - \sum_{t \in [T]}\alpha_t\Big(\langle \nabla f(\oxt), \ozt \rangle + \frac{9\alpha_t\sigma^2}{4n}\|\nabla f(\oxt)\|^2 \Big) \nonumber + L\sum_{t \in [T]}\alpha_t^2\|\ozt\|^2 \nonumber \\ &+ \frac{3\Delta_xL^2}{2}\sum_{t \in [T]}\alpha_t\lambda^{2(t-1)} + \frac{3\lambda^2L^2}{n(1-\lambda)^2}\sum_{t \in [T]}\alpha_t^3\sum_{i \in [n]}\Big[\|\zit\|^2 + A^2 + B^2\|\nabla f(\xit)\|^2 \Big].
    \end{align*} Rearranging and choosing $\alpha_t \leq \frac{1-\lambda}{\lambda LB\sqrt{48}}$, we have
    \begin{align}
        \frac{1}{16n}\sum_{t \in [T]}&\sum_{i \in [n]}\alpha_t\|\nabla  f(\xit)\|^2 \leq \Delta_f - \sum_{t \in [T]}\alpha_t\Big(\langle \nabla f(\oxt), \ozt \rangle + \frac{9\alpha\sigma^2}{4n}\|\nabla f(\oxt)\|^2 \Big) + L\sum_{t \in [T]}\alpha^2_t\|\ozt\|^2 \nonumber \\ &+ \frac{3\Delta_xL^2}{2}\sum_{t \in [T]}\alpha_t\lambda^{2(t-1)} + \frac{3\lambda^2A^2L^2}{(1-\lambda)^2}\sum_{t \in [T]}\alpha_t^3 + \frac{3\lambda^2L^2}{n(1-\lambda)^2}\sum_{t \in [T]}\sum_{i \in [n]}\alpha_t^3\|\zit\|^2. \label{eq:mgf-final}
    \end{align} Define the process $M_T \coloneqq \frac{1}{16n}\sum_{t \in [T]}\sum_{i \in [n]}\alpha_t\|\nabla f(\xit)\|^2 - \Delta_f - \frac{3\Delta_x L^2}{2}\sum_{t \in [T]}\alpha_t\lambda^{2(t-1)} - \frac{3\lambda^2A^2L^2}{(1-\lambda)^2}\sum_{t \in [T]}\alpha_t^3$ and consider its' MGF. Using \eqref{eq:mgf-final}, we then get 
    \begin{align}
        &\E[\exp(M_T)] \leq \E\Bigg[\exp\bigg(\sum_{t \in [T]}\underbrace{-\alpha_t\Big(\langle \nabla f(\oxt), \ozt \rangle + \frac{9\alpha_t\sigma^2\|\nabla f(\oxt)\|^2}{4n}}_{b_{1,t}}\Big)\bigg) \exp\bigg(\sum_{t \in [T]}\underbrace{\alpha_t^2 L\|\ozt\|^2}_{b_{2,t}}\bigg) \nonumber \\ &\times \exp\bigg (\sum_{t \in [T]}\underbrace{\frac{3\alpha_t^3\lambda^2L^2}{n(1-\lambda)^2}\sum_{i \in [n]}\|\zit\|^2}_{b_{3,t}}\bigg)\Bigg] \leq \sqrt[3]{\E[\exp(B_{1,T})]\E[\exp(B_{2,T})]\E[\exp(B_{3,T})]}, \label{eq:bound-on-mgf}
    \end{align} where $B_{k,T} = 3\sum_{t \in [T]}b_{1,t}$ and $B_{k,0} = 0$ for all $k \in [3]$, while the last step follows by applying Proposition \ref{prop:gen-Holder}. We now analyze each quantity separately, starting with $\E[\exp(B_{1,T})]$. To that end, we have
    \begin{equation*}
        \E[\exp(B_{1,T})] = \E\Bigg[\exp\bigg(B_{1,T-1} -\frac{27\alpha_T^2\sigma^2\|\nabla f(\ox^T)\|^2}{4n}\bigg)\E\lbr\exp\lp-3\alpha_T \langle \nabla f(\ox^T), \oz^T \rangle \rp \: \vert \: \FT \rbr\Bigg].
    \end{equation*} Noting that $\nabla f(\ox^T)$ is $\FT$-measurable and applying Lemma \ref{lm:noise-properties}, we get  
    \begin{align*}
        \E[\exp(B_{1,T})] &\leq \E\Bigg[\exp\bigg(B_{1,T-1}-\frac{27\alpha_T^2\sigma^2\|\nabla f(\ox^T)\|^2}{4n} + \frac{27\alpha_T^2\sigma^2\|\nabla f(\ox^T)\|^2}{4n} \bigg)\Bigg] = \E[\exp(B_{1,T-1})].
    \end{align*} Unrolling the recursion, it follows that $\E[\exp(B_{1,T})] \leq 1$. Next, consider $\E[\exp(B_{2,T})]$. To that end, we have
    \begin{align*}
        \E[\exp(B_{2,T})] &= \E\lbr\exp(B_{2,T-1}) \E\lbr\exp\lp 3\alpha_T^2L\|\oz^T\|^2 \rp \: \vert \: \FT \rbr\rbr \\ 
        &= \E\lbr\exp(B_{2,T-1}) \E\lbr\exp\lp \frac{45\alpha_T^2\sigma^2L}{n}\times\frac{n\|\oz^T\|^2}{15\sigma^2} \rp \: \Big\vert \: \FT \rbr\rbr \\
        &\leq \E\lbr\exp(B_{2,T-1}) \lp\E\lbr\exp\lp \frac{n\|\oz^T\|^2}{15\sigma^2} \rp \: \Big\vert \: \FT \rbr \rp^{\nicefrac{45\alpha_T^2\sigma^2L}{n}}\rbr,
    \end{align*} where the last inequality follows from $\alpha_t \leq \frac{\sqrt{n}}{3\sigma\sqrt{5L}}$ and Proposition \ref{prop:Jensen}. From the third property in Lemma \ref{lm:noise-properties} we get
    \begin{equation*}
        \E[\exp(B_{2,T})] \leq \exp\lp\frac{45\alpha_T^2\sigma^2L}{n}\big(1 + \log(2d)\big)\rp\E\lbr\exp(B_{2,T-1}) \rbr.
    \end{equation*} Unrolling the recursion, it follows that $\E[\exp(B_{2,T})] \leq \exp\lp \frac{45\sigma^2L}{n}\big(1 + \log(2d)\big)\sum_{t \in [T]}\alpha_t^2\rp$. Finally, to bound $\E[\exp(B_{3,T})]$, we can proceed in the same way, using the conditional independence of noise across agents from \textbf{(A4)} to note that, if $\alpha_t \leq \frac{\sqrt[3]{n}(1-\lambda)^{2/3}}{\osigma^{2/3}\lambda^{2/3}L^{2/3}\sqrt[3]{9}}$, then $\E[\exp(B_{3,T})] \leq \exp\lp \frac{9\sigma^2\lambda^2L^2}{(1-\lambda)^2}\sum_{t \in [T]}\alpha_t^3 \rp$. Combining everything, we get
    \begin{equation}\label{eq:mgf-bounded}
        \E[\exp(M_T)] \leq \exp\bigg( \frac{15\sigma^2L}{n}\big(1 + \log(2d)\big)\sum_{t \in [T]}\alpha_t^2 + \frac{3\sigma^2\lambda^2L^2}{(1-\lambda)^2}\sum_{t \in [T]}\alpha_t^3 \bigg).
    \end{equation} Using Markov's inequality and \eqref{eq:mgf-bounded}, we get, for any $\epsilon > 0$
    \begin{equation*}
        \Prob\lp M_T > \epsilon \rp \leq \exp(-\epsilon)\E[\exp(M_T)] \leq \exp\bigg( -\epsilon + \frac{15\sigma^2L}{n}\big(1 + \log(2d)\big)\sum_{t \in [T]}\alpha_t^2 + \frac{3\sigma^2\lambda^2L^2}{(1-\lambda)^2}\sum_{t \in [T]}\alpha_t^3 \bigg),
    \end{equation*} or equivalently, for any $\delta \in (0,1)$, with probability at least $1 - \delta$
    \begin{equation*}
        M_T \leq \log(\nicefrac{1}{\delta}) + \frac{15\sigma^2L}{n}\big(1 + \log(2d)\big)\sum_{t \in [T]}\alpha_t^2 + \frac{3\sigma^2\lambda^2L^2}{(1-\lambda)^2}\sum_{t \in [T]}\alpha_t^3.
    \end{equation*} Using the definition of $M_T$ and that the sequence of step-sizes is non-increasing, we get with probability at least $1 - \delta$
    \begin{align*}
        \frac{\alpha_T}{16n}\sum_{t \in [T]}\sum_{i \in [n]}\|\nabla f(\xit)\|^2 &\leq  \Delta_f + \log(\nicefrac{1}{\delta}) + \frac{3\Delta_x L^2}{2}\sum_{t \in [T]}\alpha_t\lambda^{2(t-1)} \\
        &+ \frac{15\sigma^2L}{n}\big(1 + \log(2d)\big)\sum_{t \in [T]}\alpha_t^2 + \frac{3\lambda^2L^2(\sigma^2 + A^2)}{(1-\lambda)^2}\sum_{t \in [T]}\alpha_t^3. 
    \end{align*} Dividing both sides by $\frac{\alpha_T T}{16}$, it follows that, with probability at least $1 - \delta$
    \begin{align*}
        \frac{1}{nT}\sum_{t \in [T]}\sum_{i \in [n]}&\|\nabla f(\xit)\|^2 \leq  \frac{16\big(\Delta_f + \log(\nicefrac{1}{\delta})\big)}{\alpha_T T} + \frac{24\Delta_x L^2}{\alpha_TT}\sum_{t \in [T]}\alpha_t\lambda^{2(t-1)} \\
        &+ \frac{240\sigma^2L}{n\alpha_TT}\big(1 + \log(2d)\big)\sum_{t \in [T]}\alpha_t^2 + \frac{48\lambda^2L^2(\sigma^2 + A^2)}{\alpha_TT(1-\lambda)^2}\sum_{t \in [T]}\alpha_t^3. 
    \end{align*} Next, consider two regimes with respect to the time horizon $T$. 
    \begin{enumerate}[leftmargin=*]
        \item \emph{Known time horizon}. In this case, we choose a fixed step-size $\alpha_t \equiv \alpha$, for all $t \in [T]$. Noticing that $\sum_{t \in [T]}\alpha_t\lambda^{2(t-1)} \leq \frac{\alpha}{1-\lambda^2}$, we then have, with probability at least $1 - \delta$
        \begin{align*}
            \frac{1}{nT}\sum_{t \in [T]}\sum_{i \in [n]}&\|\nabla f(\xit)\|^2 \leq  \frac{16\big(\Delta_f + \log(\nicefrac{1}{\delta})\big)}{\alpha T} + \frac{24\Delta_x L^2}{(1-\lambda^2)T} + \frac{240\alpha\sigma^2L\big(1+\log(2d)\big)}{n} + \frac{48\alpha^2\lambda^2L^2(\sigma^2 + A^2)}{(1-\lambda)^2}. 
        \end{align*} If the step-size satisfies $\alpha \leq \frac{\sqrt{n}}{\sigma\sqrt{15LT}}$, then with probability at least $1 - \delta$
        \begin{align*}
            \frac{1}{nT}\sum_{t \in [T]}\sum_{i \in [n]}&\|\nabla f(\xit)\|^2 \leq  \frac{16\big(\Delta_f + \log(\nicefrac{1}{\delta})\big)}{\alpha T} + \frac{24\Delta_x L^2}{(1-\lambda^2)T} + \frac{16\sigma\sqrt{15L}\big(1 + \log(2d)\big)}{\sqrt{nT}} + \frac{4n\lambda^2L(\sigma^2 + A^2)}{\sigma^2(1-\lambda)^2T}. 
        \end{align*} Setting $C = \min\lcb \frac{1}{2L}, \frac{n}{9\sigma^2}, \frac{1-\lambda}{\lambda LB\sqrt{48}}, \frac{\sqrt{n}}{3\sigma\sqrt{5L}}, \frac{\sqrt[3]{n}(1-\lambda)^{2/3}}{\osigma^{2/3}\lambda^2L^{2/3}\sqrt[3]{9}}\rcb$ and $\alpha = \min\lcb C, \frac{\sqrt{n}}{\sigma\sqrt{15LT}} \rcb$ guarantees that all the step-size conditions are satisfied and it readily follows that $\frac{1}{\alpha} = \max\lcb \frac{1}{C}, \frac{\sigma\sqrt{15LT}}{\sqrt{n}} \rcb \leq \frac{1}{C} + \frac{\sigma\sqrt{15LT}}{\sqrt{n}}$, implying that $\frac{1}{\alpha T} \leq \frac{\sigma\sqrt{15L}}{\sqrt{nT}} + \frac{1}{CT}$. Therefore, for any $\delta \in (0,1)$, with probability at least $1 - \delta$, we finally have
        \begin{align*}
            &\frac{1}{nT}\sum_{t \in [T]}\sum_{i \in [n]}\|\nabla f(\xit)\|^2 \: \leq \:  \frac{16\sigma\sqrt{15L}\big(\Delta_f + \log(\nicefrac{2d}{\delta}) + 1\big)}{\sqrt{nT}} + \frac{16\big(\Delta_f + \log(\nicefrac{1}{\delta})\big)}{CT} + \frac{24\Delta_x L^2}{(1-\lambda^2)T} + \frac{4n\lambda^2L(\sigma^2 + A^2)}{\sigma^2(1-\lambda)^2T}.
        \end{align*}

    \item \emph{Unknown time horizon}. In this case, we choose a time-varying step-size $\alpha_t = \frac{C^\prime}{\sqrt{t+1}}$, for all $t \geq 1$ and $C^\prime = \sqrt{2}C$, where $C = \min\lcb \frac{1}{2L}, \frac{n}{9\sigma^2}, \frac{1-\lambda}{\lambda LB\sqrt{48}}, \frac{\sqrt{n}}{3\sigma\sqrt{5L}}, \frac{\sqrt[3]{n}(1-\lambda)^{2/3}}{\osigma^{2/3}\lambda^2L^{2/3}\sqrt[3]{9}}\rcb$, again guaranteeing that all the step-size conditions are satisfied. Noting that $\alpha_t \leq C$, we get $\sum_{t \in [T]}\alpha_t\lambda^{2(t-1)} \leq \frac{C}{(1-\lambda^2)}$, hence, with probability at least $1 - \delta$
    \begin{align*}
        \frac{1}{nT}\sum_{t \in [T]}&\sum_{i \in [n]}\|\nabla f(\xit)\|^2 \leq  \frac{16\sqrt{2}\big(\Delta_f + \log(\nicefrac{1}{\delta})\big)}{C\sqrt{T+1}} + \frac{24\sqrt{2}\Delta_x L^2}{C\sqrt{T+1}(1-\lambda^2)} \\ 
        &+ \frac{240\sqrt{2}\sigma^2LC\log(T+1)\big(1 + \log(2d)\big)}{n\sqrt{T+1}} + \frac{384\lambda^2L^2(\sigma^2 + A^2)C^2}{(1-\lambda)^2\sqrt{T+1}}. 
    \end{align*} Note that in the unknown time horizon we lose linear speed-up, as even when $C = \frac{\sqrt{n}}{3\sigma\sqrt{5L}}$, we get
    \begin{align*}
        \frac{1}{nT}&\sum_{t \in [T]}\sum_{i \in [n]}\|\nabla f(\xit)\|^2 = \bigO\Bigg(\frac{\sigma\sqrt{L}\big(\Delta_f + \log(\nicefrac{2dT}{\delta})\big)}{\sqrt{n(T+1)}} + \frac{\sigma\Delta_x L^{3/2}}{(1-\lambda^2)\sqrt{n(T+1)}} + \frac{n\lambda^2L(\sigma^2 + A^2)}{\sigma^2(1-\lambda)^2\sqrt{T+1}}\Bigg), 
    \end{align*} since the term which does not attain linear speed-up is no longer of higher order. We note that linear speed-up in the MSE sense is also achieved under a known time horizon and fixed step-size, e.g., \cite{pmlr-v119-koloskova20a,ran-improved}. 
    \end{enumerate}
\end{proof}

\section{Proofs for Strongly Convex Costs}\label{sup:str-cvx}

In this section we prove Lemma \ref{lm:consensus-str-cvx} and Theorem \ref{thm:main-dsgd-str-cvx}. To do so, we follow a similar strategy to the one in, e.g., \cite{conv-rates-dsgd-jakovetic}, where it is shown that the sequence of iterates generated by \dsgd is bounded in the MSE sense. However, to establish HP guarantees we instead work with the MGF of the iterates and have the following result.

\begin{lemma}\label{lm:bdd-iterates}
    Let assumptions \textup{\textbf{(A1)}}-\textup{\textbf{(A4)}} and \textup{\textbf{(A6)}} hold and let $a,t_0,K > 0$ and $\nu \in (0,1]$ be positive constants. If for all $t \geq 1$ the step-size satisfies $\alpha_t \leq \min\lcb \frac{1}{\osigma\sqrt{2(t+t_0+2)K}}, \frac{1}{\mu} \rcb$, with $\nu \leq \min\lcb 1, \frac{\mu}{24a\osigma^2K} \rcb$ and $K_{t+1} = (t+t_0+2)K$, then 
    \begin{equation*}
        \E[\exp(\nu K_{t+1}\|\bxtp - \bxs\|^2)] \leq \exp\bigg(\nu K_{t+1}\bigg(\frac{4an\sigma^2\alpha_{t+1}}{a\mu-1} + \frac{9\|\nbfs\|^2}{\mu^2} + \frac{(1+t_0)^{a\mu}\|\bx^1-\bxs\|^2}{(t+1+t_0)^{a\mu}}\bigg)\bigg).
    \end{equation*}
\end{lemma}

\begin{proof}
    Consider the update rule \eqref{eq:dsgd-vector}. We then have
    \begin{equation}\label{eq:opt-gap-expansion}
        \bxtp - \bxs = \bW(\bxt - \bxs - \alpha_t\bg_t) = \bW(\bxt - \bxs - \alpha_t\nbf^t - \alpha_t\bzt), 
    \end{equation} where in the first equality we used $\bW\bxs = \bxs$. Using Taylor's expansion, for each $i \in [n]$ and $x \in \R^d$, we have
    \begin{equation}\label{eq:taylor-expansion}
        \nabla f_i(x) = \nabla f_i(x^\star) + \int_{0}^1 \nabla^2 f_i(x^\star + \tau(x-x^\star))d\tau(x-x^\star) = \nabla f_i(x^\star) + H_i(x)(x-x^\star). 
    \end{equation} Denote by $\bHt \coloneqq diag(H_1(x^t_1),\ldots,H_n(x^t_n)) \in \R^{nd \times nd}$ the block diagonal matrix and recall that $\nbfs = col(\nabla f_1(x^\star),\ldots,\nabla f_n(x^\star)) \in \R^{nd}$. Using \eqref{eq:taylor-expansion}, we can readily see that $\nbft = \nbfs + \bHt(\bxt - \bx^\star)$, therefore, plugging in \eqref{eq:opt-gap-expansion}, we get
    \begin{equation*}
        \bxtp - \bxs = \bW(\bI-\alpha_t\bHt)(\bxt - \bxs) - \alpha_t\bW\nbfs - \alpha_t\bW\bzt = C_t + \alpha_t\bW\bzt, 
    \end{equation*} where $C_t \coloneqq \bW(\bI-\alpha_t\bHt)(\bxt - \bxs) - \alpha_t\bW\nbfs$. Therefore, we have
    \begin{align*}
        \|\bxtp &- \bxs\|^2 = \|C_t\|^2 - 2\alpha_t\langle \bW C_t, \bzt \rangle +\alpha_t^2\|\bW\bzt\|^2 \leq \|C_t\|^2 -2\alpha_t\langle \bW C_t, \bzt \rangle +\alpha_t^2\|\bzt\|^2 \\
        &\stackrel{(i)}{\leq} (1+\theta)\|\bW(\bI - \alpha_t\bHt)(\bxt - \bxs)\|^2 + (1 + \theta^{-1})\alpha_t^2\|\bW\nbfs\|^2 -2\alpha_t\langle \bW C_t, \bzt \rangle +\alpha_t^2\|\bzt\|^2 \\
        &\stackrel{(ii)}{\leq} (1+\theta)(1 - \alpha_t\mu)^2\|\bxt - \bxs\|^2 + (1 + \theta^{-1})\alpha_t^2\|\nbfs\|^2 -2\alpha_t\langle \bW C_t, \bzt \rangle +\alpha_t^2\|\bzt\|^2,
    \end{align*} where in $(i)$ we used Proposition \ref{prop:Young}, for some $\theta > 0$ (to be specified later), while $(ii)$ follows from Proposition \ref{prop:mixing-mx} and the fact that $\|\bI - \alpha_t\bHt\| \leq (1-\alpha_t\mu)$ (as a consequence of Proposition \ref{prop:str-cvx}). Define $D_t \coloneqq (1+\theta)(1 - \alpha_t\mu)^2\|\bxt - \bxs\|^2 + (1 + \theta^{-1})\alpha_t^2\|\nbfs\|^2$ and consider the MGF of $\nu K_{t+1}\|\bxtp - \bxs\|^2$ conditioned on $\Ft$, where we recall that $K_{t+1} = (t+t_0+2)K$ for some $K > 0$ and $\nu \in (0,1]$. We then have
    \begin{align}\label{eq:mgf-opt-gap}
        \E_t[\exp(\nu K_{t+1}&\|\bxtp - \bxs\|^2)] \stackrel{(a)}{\leq} \exp(\nu K_{t+1} D_t)\E_t\Big[\exp\Big(\nu K_{t+1}\big(-2\alpha_t\langle\bW C_t,\bzt\rangle + \alpha_t^2\|\bzt\|^2\big)\Big)\Big] \nonumber \\
        &\stackrel{(b)}{\leq} \exp(\nu K_{t+1} D_t)\sqrt{\E_t[\exp(-4\alpha_t\nu K_{t+1}\langle\bW C_t,\bzt\rangle)\E_t[\exp(2 \alpha_t^2\nu K_{t+1}\|\bzt\|^2)]} \nonumber \\
        &\stackrel{(c)}{\leq} \exp(\nu K_{t+1} D_t)\sqrt{\exp(12\alpha_t^2\osigma^2\nu^2K_{t+1}^2\|\bW C_t\|^2 + 2 \alpha_t^2n\sigma^2\nu K_{t+1})} \nonumber \\
        &\stackrel{(d)}{\leq} \exp(\nu K_{t+1}D_t + 6\alpha_t^2\osigma^2\nu^2K_{t+1}^2\| C_t\|^2 + \alpha_t^2n\sigma^2\nu K_{t+1}) \nonumber \\
        &\stackrel{(e)}{\leq} \exp(\nu K_{t+1}(1+6\alpha_t^2\osigma^2\nu K_{t+1})D_t + \alpha_t^2n\sigma^2\nu K_{t+1}), 
    \end{align} where $(a)$ follows from the fact that $D_t$ is $\Ft$-measurable, in $(b)$ we used Proposition \ref{prop:Holder}, $(c)$ follows from Lemma \ref{lm:noise-properties} and $\alpha_t \leq \frac{1}{\osigma\sqrt{2(t+t_0+2)K}}$, in $(d)$ we used Proposition \ref{prop:mixing-mx}, while $(e)$ follows from the definition of $D_t$ and the fact that $\|C_t\|^2 \leq D_t$. We now analyze $(1+6\alpha_t^2\osigma^2\nu K_{t+1})D_t$. To that end, if we choose $\theta = \frac{\alpha_t\mu}{2}$, it follows that
    \begin{align}\label{eq:mgf-opt-gap-intermed}
        (1 + 6\alpha_t^2\osigma^2\nu K_{t+1})D_t &= (1 + 6\alpha_t^2\osigma^2\nu K_{t+1})\big[(1+\nicefrac{\alpha_t\mu}{2})(1 - \alpha_t\mu)^2\|\bxt - \bxs\|^2 + (1 + \nicefrac{2}{\alpha_t\mu})\alpha_t^2\|\nbfs\|^2\big] \nonumber \\
        &\stackrel{(i)}{\leq} (1+6\alpha_t^2\osigma^2\nu K_{t+1})\big[(1-\nicefrac{\alpha_t\mu}{2})(1 - \alpha_t\mu)\|\bxt - \bxs\|^2 + (\alpha_t + \nicefrac{2}{\mu})\alpha_t\|\nbfs\|^2\big] \nonumber \\
        &\stackrel{(ii)}{\leq} (1 - \alpha_t\mu)\|\bxt - \bxs\|^2 + 9\alpha_t\|\nbfs\|^2/2\mu,
    \end{align} where in $(i)$ we used the fact that $(1 + \frac{a}{2})(1 - a) \leq (1-\frac{a}{2})$ for any $a > 0$, while $(ii)$ follows by setting $\nu \leq \frac{\mu}{24a\osigma^2K}$ and from the step-size choice $\alpha_t \leq \frac{1}{\mu}$. Plugging \eqref{eq:mgf-opt-gap-intermed} in \eqref{eq:mgf-opt-gap}, using the shorthand $E_t = \alpha_t^2n\sigma^2 + 9\alpha_t\|\nbfs\|^2/2\mu$ and taking the full expectation, we get
    \begin{align*}
        \E[\exp(\nu &K_{t+1}\|\bxtp - \bxs\|^2)] \leq \exp(\nu K_{t+1}E_t)\E[\exp((1-\alpha_t\mu)\nu K_{t+1}\|\bxt - \bxs\|^2)] \\ 
        &\leq \exp(\nu K_{t+1}E_t)\Big(\E[\exp(\nu K_t\|\bxt - \bxs\|^2)]\Big)^{(1-\alpha_t\mu)\frac{t+t_0+2}{t+t_0+1}} \\
        &\leq \exp\big(\nu K_{t+1}\big(E_t + (1-\alpha_t\mu)E_{t-1}\big)\big)\Big(\E[\exp((1-\alpha_{t-1}\mu)\nu K_t\|\bx_{t-1} - \bxs\|^2)]\Big)^{(1-\alpha_t\mu)\frac{t+t_0+2}{t+t_0+1}} \\
        &\leq \ldots \leq \exp\Big(\nu K_{t+1}\sum_{k = 1}^{t}E_{k}\prod_{s = k+1}^t(1-\alpha_k\mu) + \nu K_1\|\bx^1 - \bxs\|^2\prod_{k = 1}^{t}(1-\alpha_k\mu)\frac{t+t_0+2}{t_0+2} \Big) \\
        &= \exp\bigg(\nu K_{t+1}\bigg(\sum_{k = 1}^tE_k\prod_{s = k+1}^t(1-\alpha_s\mu) + \|\bx^1 - \bxs\|^2\prod_{k = 1}^t(1 - \alpha_k\mu)\bigg) \bigg),
    \end{align*} where the second inequality follows from Proposition \ref{prop:Jensen} and the fact that $0<(1-\alpha_t\mu)\frac{t+t_0+2}{t+t_0+1} \leq 1$, for any $t \geq 1$, whenever $0 < \alpha_t \leq \frac{1}{\mu}$. Next, we use Proposition \ref{prop:Ran}, to get
    \begin{align*}
        \sum_{k = 1}^tE_k&\prod_{s=k+1}^t(1-\alpha_s\mu) \leq \sum_{k = 1}^t\Big(\alpha_k^2n\sigma^2 + 9\alpha_k\|\nbfs\|^2/2\mu\Big)\frac{(k+1+t_0)^{a\mu}}{(t+1+t_0)^{a\mu}} \\
        &\stackrel{(i)}{\leq} \frac{1}{(t+t_0+1)^{a\mu}}\sum_{k = 1}^t\Big(4a^2n\sigma^2(k+1+t_0)^{a\mu-2} + \frac{9a\|\nbfs\|^2}{\mu}(k+t_0+1)^{a\mu-1}\Big) \\
        &\stackrel{(ii)}{\leq} \frac{4a^2n\sigma^2}{(a\mu-1)(t+t_0+1)} + \frac{9\|\nbfs\|^2}{\mu^2},
    \end{align*} where $(i)$ follows from the choice $\alpha_k \leq \frac{1}{\mu}$, while in $(ii)$ we use the lower Darboux sum. Combining everything, we get
    \begin{equation*}
        \E[\exp(\nu K_{t+1}\|\bxtp-\bxs\|^2)] \leq \exp\bigg(\nu K_{t+1}\bigg(\frac{4an\sigma^2\alpha_{t+1}}{a\mu-1} + \frac{9\|\nbfs\|^2}{\mu^2} + \frac{(1+t_0)^{a\mu}\|\bx^1-\bxs\|^2}{(t+t_0+1)^{a\mu}}\bigg)\bigg).
    \end{equation*}
\end{proof}

Lemma \ref{lm:bdd-iterates} is an important building block for bounding the consensus gap. We next restate and prove Lemma \ref{lm:consensus-str-cvx}.

{
\renewcommand{\thetheorem}{3.5}
\begin{lemma}
    Let \textup{\textbf{(A1)}}-\textup{\textbf{(A4)}} and \textup{\textbf{(A6)}} hold, let $a,t_0,K > 0$ and the step-size be given by $\alpha_t = \frac{a}{t+t_0}$, and let $x^1_i = x^1_j$, for all $i,j\in[n]$. If $a = \frac{6}{\mu}$ and $t_0 \geq \max\Big\{6, \frac{288\osigma^2K}{\mu^2},\frac{3456\osigma^2\lambda^2K}{\mu^2(1-\lambda)},\frac{12\lambda L\sqrt{10}}{\mu(1-\lambda)}\Big\}$, then for $K_{t+1} = (t+t_0+2)K$ and any $\nu \leq \min \lcb 1, \frac{\mu^2}{144\sigma^2K} \rcb$, we have 
    \begin{align*}
        &\E\Big[\exp\Big(\nu K_{t+1}\sum_{i \in [n]}\|\xitp - \oxtp \|^2\Big)\Big] \leq \exp\bigg(\nu K_{t+1}\Big(\sum_{k \in [t]}\lambda^{t-k}S_{k} + \sum_{k \in [t]}\lambda^{t-k}D_{k}\Big) \bigg),
    \end{align*} where $S_k \coloneqq \alpha_k^2\lambda^2\Big(n\sigma^2 + \frac{5\|\nbfs\|^2}{1-\lambda}\Big)$ and $D_k \coloneqq \frac{5\alpha_k^2\lambda^2L^2}{1-\lambda}\Big(\frac{4an\sigma^2\alpha_{k}}{5} + \frac{9\|\nbfs\|^2}{\mu^2} + \frac{(1+t_0)^{6}\|\bx^1-\bxs\|^2}{(k+t_0)^{6}}\Big)$.
\end{lemma}
}

\begin{proof}
    Note that $\sum_{i \in [n]}\|\xitp-\oxtp\|^2 = \|\bxtp - \obxtp\|^2$ and recall $\tbW = \bW - \bJ$ and the update rule \eqref{eq:dsgd-vector}. Then
    \begin{equation*}
        \bxtp - \obxtp = \tbW(\bxt - \obxt - \alpha_t\nbft - \alpha_t\bzt). 
    \end{equation*} Denote the consensus difference by $\tbxtp \coloneqq \bxtp - \obxtp$ and let $C_t \coloneqq \tbW(\tbx_t - \alpha_t\nbft)$. Noting that $\|\tbxtp\|^2 = \|C_t\|^2 - 2\alpha_t\langle \tbW C_t,\bzt \rangle + \alpha_t\|\tbW\bzt\|^2$, we then consider the MGF of $\nu K_{t+1}\|\tbxtp\|^2$ conditioned on $\Ft$, where we recall that $K_{t+1} = (t+t_0+2)K$, for some $K > 0$ and $\nu \in (0,1]$. If $\alpha_t \leq \frac{1}{\osigma\lambda\sqrt{2(t+t_0+2)K}}$, we have
    \begin{align}\label{eq:mgf-cons-gap-intermed}
        \E_t[\exp(\nu K_{t+1}\|\tbxtp\|^2)] &\leq \exp(\nu  K_{t+1}\|C_t\|^2)\sqrt{\E_t\big[\exp(-4\alpha_t\nu K_{t+1}\langle \tbW C_t,\bzt \rangle)\big]\E_t\big[\exp(2\alpha_t^2\lambda^2\nu K_{t+1}\|\bzt\|^2)\big]} \nonumber \\
        &\leq \exp\Big(\nu K_{t+1}\Big((1+6\alpha_t^2\osigma^2\lambda^2\nu K_{t+1})\|C_t\|^2 + \alpha_t^2\lambda^2n\sigma^2\Big)\Big),  
    \end{align} where the first inequality follows from the fact that $C_t$ is $\Ft$-measurable and using Proposition \ref{prop:Holder}, while the second follows from Lemma \ref{lm:noise-properties}. Next, using Proposition \ref{prop:mixing-mx} and defining $\tlambda \coloneqq 1-\lambda \in (0,1]$, we get
    \begin{align}\label{eq:cons-gap-intermed}
        \|C_t\|^2 &\leq \lambda^2\|\tbx_t - \alpha_t\nbft\|^2 \stackrel{(i)}{\leq} \lambda^2(1+\theta)\|\tbx_t\|^2 + \alpha_t^2\lambda^2(1 + \theta^{-1})\|\nbft\|^2 \nonumber \\
        &= (1-\tlambda)(1+\theta)\lambda\|\tbx_t\|^2 + \alpha_t^2\lambda^2(1 + \theta^{-1})\|\nbft\|^2 \nonumber \\
        &\stackrel{(ii)}{\leq} (1-\nicefrac{\tlambda}{2})\lambda\|\tbx_t\|^2 + \alpha_t^2\lambda^2(1 + \nicefrac{2}{\tlambda})\|\nbft\|^2 \nonumber \\
        &\stackrel{(iii)}{\leq} (1-\nicefrac{\tlambda}{2})\lambda\|\tbx_t\|^2 + \frac{6\alpha_t^2\lambda^2}{1-\lambda}\|\nbfs\|^2 + \frac{6\alpha_t^2\lambda^2L^2}{1-\lambda}\|\bxt - \bxs\|^2,
    \end{align} where $(i)$ follows from Proposition \ref{prop:Young}, in $(ii)$ we set $\theta = \frac{\tlambda}{2}$, while $(iii)$ follows from the fact that $\|\nbft\|^2 \leq 2L^2\|\bxt-\bxs\|+2\|\nbfs\|^2$ (recall Proposition \ref{prop:descent-ineq}). Choosing $\alpha_t \leq \frac{\sqrt{1-\lambda}}{2\osigma \lambda\sqrt{6(t+t_0+2)K}}$ and plugging \eqref{eq:cons-gap-intermed} in \eqref{eq:mgf-cons-gap-intermed}, we get
    \begin{align*}
        &\E_t[\exp(\nu K_{t+1}\|\tbxtp\|^2)] \leq \exp\bigg(\nu K_{t+1}\Big(1+\frac{\tlambda}{4}\Big)\Big[\lambda\Big(1-\frac{\tlambda}{2}\Big)\|\tbx_t\|^2 \\ &\quad\quad\quad\quad+ \frac{4\alpha_t^2\lambda^2}{1-\lambda}\|\nbfs\|^2 + \frac{4\alpha_t^2\lambda^2L^2}{1-\lambda}\|\bxt - \bxs\|^2\Big] + \alpha_t^2n\sigma^2\lambda^2\nu K_{t+1} \bigg) \\
        &\leq \exp\bigg(\nu K_{t+1}\Big(\lambda\Big(1 - \frac{\tlambda}{4}\Big)\|\tbx_t\|^2 + \frac{5\alpha_t^2\lambda^2}{1-\lambda}\|\nbfs\|^2 + \frac{5\alpha_t^2\lambda^2L^2}{1-\lambda}\|\bxt - \bxs\|^2 + \alpha_t^2n\sigma^2\lambda^2\Big) \bigg).
    \end{align*} Taking the full expectation and introducing the shorthand $S_t \coloneqq \alpha_t^2\lambda^2\Big(n\sigma^2 + \frac{5\|\nbfs\|^2}{1-\lambda}\Big)$, we get
    \begin{align}\label{eq:cons-gap-pt1}
        &\E[\exp(\nu K_{t+1}\|\tbxtp\|^2)] \leq \exp(S_t\nu K_{t+1})\E\bigg[\exp\bigg(\lambda\nu K_{t+1}\Big(1 - \frac{\tlambda}{4} \Big)\|\tbx_t\|^2 + \frac{5\alpha_t^2\lambda^2L^2\nu K_{t+1}}{1-\lambda}\|\bxt - \bxs\|^2 \bigg)\bigg] \nonumber \\
        &\leq \exp(S_t\nu K_{t+1})\bigg(\E\big[\exp(\lambda\nu K_{t+1}\|\tbx_t\|^2)\big]\bigg)^{1-\nicefrac{\tlambda}{4}}\Bigg(\E\bigg[\exp\bigg( \frac{20\alpha_t^2\lambda^2L^2\nu K_{t+1}}{(1-\lambda)^2}\|\bxt - \bxs\|^2 \bigg)\bigg]\Bigg)^{\nicefrac{\tlambda}{4}} \nonumber \\
        &\leq \exp(S_t\nu K_{t+1})\bigg(\E\big[\exp(\nu K_t\|\tbx_t\|^2)\big]\bigg)^{\lambda(1-\nicefrac{\tlambda}{4})\frac{t+t_0+2}{t+t_0+1}}\Bigg(\E\bigg[\exp\bigg( \frac{20\alpha_t^2\lambda^2L^2\nu K_{t+1}}{(1-\lambda)^2}\|\bxt - \bxs\|^2 \bigg)\bigg]\Bigg)^{\nicefrac{\tlambda}{4}},
    \end{align} where the second inequality follows by applying Proposition \ref{prop:Holder} with $p = (1-\nicefrac{\tlambda}{4})^{-1}$ and $q = \frac{4}{\tlambda}$, while the third follows from the fact that $\lambda\frac{t+t_0+2}{t+t_0+1} \leq 1$ for $t_0 \geq \frac{1}{1-\lambda}$ and applying Proposition \ref{prop:Jensen}. If $\alpha_t \leq \frac{1-\lambda}{2\lambda L\sqrt{10}}$, we get
    \begin{align}\label{eq:cons-gap-pt2}
        \Bigg(\E\bigg[\exp\bigg( \frac{20\alpha_t^2\lambda^2L^2\nu K_{t+1}}{(1-\lambda)^2}\|\bxt - \bxs\|^2 \bigg)\bigg]\Bigg)^{\nicefrac{\tlambda}{4}} &\leq \bigg(\E\big[\exp(\nu K_t\|\bxt-\bxs\|^2) \big]\bigg)^{\frac{5\alpha_t^2(t+t_0+2)\lambda^2L^2}{(1-\lambda)(t+t_0+1)}} \nonumber \\
        &\leq \exp(\nu K_{t+1}D_t),
    \end{align} where we used Proposition \ref{prop:Jensen} in the first and Lemma \ref{lm:bdd-iterates} in the second inequality, with $D_t = \frac{5\alpha_t^2\lambda^2L^2}{1-\lambda}\Big(\frac{4an\sigma^2\alpha_{t}}{a\mu-1} + \frac{9\|\nbfs\|^2}{\mu^2} + \frac{(1+t_0)^{a\mu}\|\bx^1-\bxs\|^2}{(t+t_0)^{a\mu}}\Big)$. Similarly, we use \eqref{eq:cons-gap-pt1}, to get
    \begin{equation}\label{eq:cons-gap-pt3}
    \begin{aligned}
        \bigg(\E\big[\exp(\nu K_t\|\tbx_t\|^2)\big]\bigg)^{\lambda(1-\nicefrac{\tlambda}{4})\frac{t+t_0+2}{t+t_0+1}} &\leq \exp(\lambda(1-\nicefrac{\tlambda}{4})S_{t-1}\nu K_{t+1})\bigg(\E\big[\exp(\nu K_{t-1}\|\tbx_{t-1}\|^2)\big]\bigg)^{\lambda^2(1-\nicefrac{\tlambda}{4})^2\frac{t+t_0+2}{t+t_0}} \\
        &\times\Bigg(\E\bigg[\exp\bigg( \frac{20\alpha_t^2\lambda^2L^2\nu K_t}{(1-\lambda)^2}\|\bx_{t-1} - \bxs\|^2 \bigg)\bigg]\Bigg)^{\nicefrac{\lambda(1-\nicefrac{\tlambda}{4})\tlambda}{4}}.
        \end{aligned}
    \end{equation} Plugging \eqref{eq:cons-gap-pt2} and \eqref{eq:cons-gap-pt3} into \eqref{eq:cons-gap-pt1}, we get
    \begin{align*}
        \E[\exp(\nu K_{t+1}\|\tbxtp\|^2)] &\leq \exp\Big(\nu K_{t+1}\big(S_t + S_{t-1}\lambda(1-\nicefrac{\tlambda}{4})\big) + D_t + \lambda(1-\nicefrac{\tlambda}{4}) D_{t-1}\big)\Big) \\ &\times \Big(\E\big[\exp(\nu K_{t-1}\|\tbx_{t-1}\|^2)\big]\Big)^{\lambda^2(1-\nicefrac{\tlambda}{4})^2\frac{t+t_0+2}{t+t_0}}.
    \end{align*} Unrolling the recursion, it follows that
    \begin{align*}
        \E[\exp(\nu K_{t+1}\|\tbxtp\|^2)] &\leq \exp\bigg(\nu K_{t+1}\Big(\sum_{k = 1}^{t}\lambda^{t-k}S_{k} + \sum_{k = 1}^{t}\lambda^{t-k}D_{k} + \lambda^t\|\tbx_1\|^2\Big) \bigg) \\
        &= \exp\bigg(\nu K_{t+1}\Big(\sum_{k = 1}^{t}\lambda^{t-k}S_{k} + \sum_{k = 1}^{t}\lambda^{t-k}D_{k}\Big) \bigg),
    \end{align*} where the last equality follows from the fact that $x^1_i = x^1_j$, for all $i,j \in [n]$. 
\end{proof}

We are now ready to prove Theorem \ref{thm:main-dsgd-str-cvx}. Prior to that, we restate it, for convenience. 

{
\renewcommand{\thetheorem}{3.7}
\begin{theorem}
    Let \textup{\textbf{(A1)}-\textbf{(A4)}} and \textup{\textbf{(A6)}} hold, the step-size be given by $\alpha_t = \frac{a}{t+t_0}$ and let $x^1_i = x^1_j$, for all $i,j\in[n]$. If $a =\frac{6}{\mu}$, $t_0 \geq \max\Big\{ 6,\frac{3+\lambda}{1-\lambda},\frac{1960\sigma^2\kappa}{\mu},\frac{432\osigma^2\kappa^2}{\mu},\frac{12\kappa\lambda\sqrt{10}}{1-\lambda}, \frac{5184\osigma^2\lambda^2\kappa^2}{\mu(1-\lambda)}\Big\}$ and $\nu = \min\Big\{ 1,\frac{\mu}{432\sigma^2\kappa^2}, \frac{\mu}{72\kappa} \Big\}$, then for any $\delta \in (0,1)$ and $T \geq 1$, with probability at least $1 - \delta$, it holds that 
    \begin{align*}
        &\frac{1}{n}\sum_{i \in [n]}\big(f(\xit)-f^\star\big) = \bigO\bigg(\frac{\nu^{-1}\log(\nicefrac{2}{\delta}) + \sigma^2\kappa\big(1+\log(2d)\big)/\mu}{n(t+t_0)} + \frac{\kappa\lambda^2(1+L)(n\sigma^2+\|\nbfs\|^2\nicefrac{(1+9\kappa^2)}{(1-\lambda))}}{\mu(1-\lambda)n(t+t_0)^2} \\
        &+ \frac{(2+t_0)^3\Delta_f}{n(t+t_0)^3} + \frac{\kappa^3\sigma^2\lambda^2(\kappa\log(t+t_0) + 1)}{\mu(1-\lambda)^2(t+t_0)^3} + \frac{\kappa^3\lambda^2L(1+t_0)\|\bx^1-\bxs\|^2}{(1-\lambda)^2n(t+t_0)^3} + \frac{\kappa^2\lambda^2L(1+t_0)^{6}\|\bx^1-\bxs\|^2}{(1-\lambda)^2n(t+t_0)^{8}} \bigg).
    \end{align*} 
\end{theorem} 
}

\begin{proof}
    Starting from Lemma \ref{lm:descent-inequality} and using property 3 of Proposition \ref{prop:str-cvx} with $x = \oxt$, we have
    \begin{equation*}
        f(\oxtp) \leq f(\oxt) - \alpha_t\mu\big(f(\oxt) - f^\star \big) - \alpha_t\langle \nabla f(\oxt), \oz_t \rangle + \alpha_t^2L\|\oz_t\|^2 + \frac{\alpha_t L^2}{2n}\sum_{i \in [n]}\|\xit - \oxt\|^2.
    \end{equation*} Subtracting $f^\star$ from both sides and defining $F_t = n(t+t_0)(f(\ox_{t})-f^\star)$, it then follows that
    \begin{align*}
        F_{t+1} &\leq (1 - \alpha_t\mu)\frac{t+t_0+1}{t+t_0}F_t - \alpha_t n(t+t_0+1)\langle \nabla f(\oxt), \oz_t \rangle \\
        &+ \alpha_t^2n(t+t_0+1)L\|\oz_t\|^2 + \frac{\alpha_t(t+t_0+1)L^2}{2}\|\bxt - \obxt\|^2.
    \end{align*} Next, consider the MGF of $F_{t+1}$ conditioned on $\Ft$. Let $\nu \in (0,1]$ be a positive constant, we then have
    \begin{align*}
        \E_t\big[\exp(\nu F_{t+1})\big] &\stackrel{(a)}{\leq} \exp\bigg((1 - \alpha_t\mu)\frac{t+t_0+1}{t+t_0}\nu F_t + \frac{\alpha_t\nu(t+t_0+1) L^2}{2}\|\bxt - \obxt\|^2 \bigg) \\
        &\times \E_t\Big[\exp(- \alpha_t\nu n(t+t_0+1)\langle \nabla f(\oxt), \oz_t \rangle + \alpha_t^2\nu n(t+t_0+1) L\|\oz_t\|^2) \Big] \\
        &\stackrel{(b)}{\leq} \exp\bigg((1 - \alpha_t\mu)\frac{t+t_0+1}{t+t_0}\nu F_t + \frac{\alpha_t\nu(t+t_0+1) L^2}{2}\|\bxt - \obxt\|^2 \bigg) \\
        &\times \sqrt{\E_t\Big[\exp(-2 \alpha_t\nu n(t+t_0+1)\langle \nabla f(\oxt), \oz_t \rangle)\Big]\E_t\Big[\exp(2\alpha_t^2\nu n(t+t_0+1) L\|\oz_t\|^2) \Big]} \\
        &\stackrel{(c)}{\leq} \exp\bigg((1 - \alpha_t\mu)\frac{t+t_0+1}{t+t_0}\nu F_t + \frac{\alpha_t\nu(t+t_0+1) L^2}{2}\|\bxt - \obxt\|^2 \bigg) \\
        &\times \exp\bigg(\frac{3 \alpha_t^2\nu^2n(t+t_0+1)^2\sigma^2\|\nabla f(\oxt)\|^2}{2} + 15\alpha_t^2\nu\sigma^2(t+t_0+1)L\big(1+\log(2d)\big)\bigg) \\
        &\stackrel{(d)}{\leq} \exp\bigg(\nu\Big(b_tF_t + c_t\|\bxt - \obxt\|^2 + d_t\Big) \bigg),
    \end{align*} where in $(a)$ we used the fact that $F_t$ and $\|\bxt-\obxt\|^2$ are $\Ft$-measurable, $(b)$ follows from Proposition \ref{prop:Holder}, in $(c)$ we use Lemma \ref{lm:noise-properties}, Proposition \ref{prop:Jensen} and impose the condition $\alpha_t \leq \frac{1}{\sigma\sqrt{30(t+t_0+1)L}}$, while $(d)$ follows from Proposition \ref{prop:descent-ineq} and the definition of $F_t$, with $b_t = \Big(1 - \alpha_t\mu + 3\alpha_t^2\nu(t+t_0+1) L\Big)\frac{t+t_0+1}{t+t_0}$, $c_t = \frac{\alpha_t(t+t_0+1)L^2}{2}$ and $d_t = 15\alpha_t^2\sigma^2(t+t_0+1)L\big(1+\log(2d)\big)$. Taking the full expectation and applying Proposition \ref{prop:Holder}, we get
    \begin{align}\label{eq:bdd-mgf-str-cvx}
        \E\big[\exp(\nu F_{t+1})\big] &\leq \exp\big(\nu d_t\big) \E\Big[\exp\Big(\nu b_tF_t + \nu c_t\|\bxt-\obxt\|^2 \Big) \Big] \nonumber \\
        &\leq \exp\big(\nu d_t\big) \sqrt[p]{\E\big[\exp\big(\nu pb_tF_t\big)\big]}\sqrt[q]{\E\big[\exp\big(\nu qc_t\|\bxt-\obxt\|^2 \big)\big]}
    \end{align} for some $p,q \in [1,\infty]$. We next analyze the expression $pb_t$. Recalling the definition of $b_t$, we get
    \begin{equation*}
        pb_t = p\bigg(1 - \frac{a\mu}{t+t_0} + \frac{3a^2\nu(t+t_0+1)L}{(t+t_0)^2}\bigg)\frac{t+t_0+1}{t+t_0} \leq p\bigg(1 - \frac{a(\mu - 6a\nu L)}{t+t_0}\bigg)\frac{t+t_0+1}{t+t_0}.
    \end{equation*} Choosing $\nu \leq \frac{\mu}{12aL}$ and $p = 1 + \frac{\alpha_t\mu}{4}$, it follows that
    \begin{equation}\label{eq:pbt-bound}
        pb_t \leq p\bigg(1-\frac{a\mu}{2(t+t_0)} \bigg)\frac{t+t_0+1}{t+t_0} \leq \bigg(1 - \frac{a\mu}{4(t+t_0)}\bigg)\bigg(1 +\frac{1}{t+t_0}\bigg) \leq 1,
    \end{equation} where the last inequality follows since $a\mu > 4$. Next, note that the choice of $p = 1+\frac{\alpha_t\mu}{4}$ implies that $q = 1 + \frac{4}{\alpha_t\mu}$. From the definition of $c_t$, we then have
    \begin{align}\label{eq:qct-bound}
        qc_t = \bigg(1 + \frac{4}{\alpha_t\mu}\bigg)\frac{\alpha_t(t+t_0+1)L^2}{2} = \bigg(\frac{\alpha_t}{2} + \frac{2}{\mu} \bigg)(t+t_0+1)L^2 \leq \frac{3L^2}{\mu}(t+t_0+1), 
    \end{align} where the first inequality follows from $\alpha_t \leq \frac{1}{\mu}$. Using \eqref{eq:pbt-bound} and \eqref{eq:qct-bound} in \eqref{eq:bdd-mgf-str-cvx}, we get
    \begin{align}\label{eq:setup-for-Lemma-5}
        \E\big[\exp(\nu F_{t+1})\big] &\leq \exp\big(d_t\nu\big)\sqrt[p]{\big(\E\big[\exp(\nu F_{t})\big]\big)^{pb_t}}\sqrt[q]{\E\big[\exp\big(\nu(t+t_0+1)3\kappa L\|\bxt-\obxt\|^2 \big)\big]} \nonumber \\
        &= \exp\big(d_t\nu)\big(\E\big[\exp(\nu F_{t})\big]\big)^{b_t}\sqrt[q]{\E\big[\exp\big(\nu (t+t_0+1)3\kappa L\|\bxt-\obxt\|^2 \big)\big]}.
    \end{align} Using Lemma \ref{lm:consensus-str-cvx} with $K = 3\kappa L$, we get
    \begin{equation*}
        \E\big[\exp\big(\nu q c_t\|\bxt - \obxt\|^2\big)\big] \leq \E\big[\exp\big(\nu K_t\|\bxt - \obxt\|^2\big)\big] \leq \exp\bigg(\nu K_{t}\Big(\sum_{k = 1}^{t-1}\lambda^{t-1-k}S_{k} + \sum_{k = 1}^{t-1}\lambda^{t-1-k}D_{k}\Big) \bigg),
    \end{equation*} where we recall that $S_k = \alpha_k^2\lambda^2\Big(n\sigma^2 + \frac{5\|\nbfs\|^2}{1-\lambda}\Big)$ and $D_k = \frac{5\alpha_k^2\lambda^2L^2}{1-\lambda}\Big(\frac{4an\sigma^2\alpha_{k}}{a\mu-1} + \frac{9\|\nbfs\|^2}{\mu^2} + \frac{(1+t_0)^{a\mu}\|\bx^1-\bxs\|^2}{(k+t_0)^{a\mu}}\Big)$. To further bound the above expression, we use Lemma \ref{lm:technical-result}, to get
    \begin{align*}
        \sum_{k = 1}^{t-1}\lambda^{t-1-k}(S_k &+ D_k) \leq \frac{4a^2\lambda^2(n\sigma^2 + \nicefrac{5\|\nbfs\|^2(1 + \nicefrac{9L^2}{\mu^2})}{(1-\lambda)})}{(1-\lambda)(t+t_0)^2} \\ &+ \frac{32a^4n\sigma^2\lambda^2L^2}{(1-\lambda)^2(t+t_0)^3} +\frac{20a^2\lambda^2L^2(1+t_0)^{6}\|\bx^1-\bxs\|^2}{(1-\lambda)^2(t+t_0)^{8}}. 
    \end{align*} Noting that $\frac{1}{q} = \frac{\alpha_t\mu}{4+\alpha_t\mu} \leq \frac{\alpha_t\mu}{4}$, we finally get
    \begin{equation*}
        \sqrt[q]{\E\big[\exp\big(\nu q c_t\|\bxt - \obxt\|^2\big)\big]} \leq \exp\bigg(\frac{\nu K_t\alpha_t\mu}{4} N_t\bigg) \leq \exp\bigg(\frac{3aL^2\nu N_t}{2} \bigg), 
    \end{equation*} where $N_t \coloneqq \frac{4a^2\lambda^2(n\sigma^2 + \nicefrac{5\|\nbfs\|^2(1 + \nicefrac{9L^2}{\mu^2})}{(1-\lambda)})}{(1-\lambda)(t+t_0)^2} + \frac{32a^4n\sigma^2\lambda^2L^2}{(1-\lambda)^2(t+t_0)^3} +\frac{20a^2\lambda^2L^2(1+t_0)^{6}\|\bx^1-\bxs\|^2}{(1-\lambda)^2(t+t_0)^{8}}$. Define $G_1 \coloneqq \frac{1080\kappa\sigma^2(1 + \log(2d))}{\mu}$, $G_2 \coloneqq \frac{1296\kappa^2\lambda^2(n\sigma^2+\nicefrac{5\|\nbfs\|^2(1 + 9\kappa^2)}{(1-\lambda)})}{\mu(1-\lambda)}$, $G_4 \coloneqq \frac{6480\kappa^3\lambda^2L(1+t_0)^{6}\|\bx^1-\bxs\|^2}{(1-\lambda)^2}$ and $G_3 \coloneqq \frac{373248n\kappa^4\sigma^2\lambda^2}{\mu(1-\lambda)^2}$, and plug into \eqref{eq:setup-for-Lemma-5}, to get
    \begin{equation*}
        \E\big[\exp\big(\nu F_{t+1}\big)\big] \leq \Big(\E\big[\exp\big(\nu F_t \big)\big]\Big)^{b_t}\exp\bigg(\sum_{i \in [3]}\frac{\nu G_i}{(t+t_0)^i} + \frac{\nu G_4}{(t+t_0)^{8}}\bigg).
    \end{equation*} Recalling the definition of $b_t$ and \eqref{eq:pbt-bound}, it follows that $b_t \leq 1 - \frac{a\mu/2-1}{t+t_0} = 1 - \frac{2}{t+t_0}$, therefore we can bound the MGF of $\nu F_{t+1}$ using Lemma \ref{lm:mgf-bound-str-cvx} with $a = 2$, $M = 8$, $C_i = G_i$ for $i \in [3]$, $C_8 = G_4$ and $C_j = 0$, for $j \in \{4,\ldots,7\}$, to  finally get
    \begin{align}\label{eq:bdd-mgf}
        \E\big[\exp\big( \nu F_{t+1} \big) \big] &\leq \exp\bigg(\frac{(t_0+2)^{3}\nu \Delta_f}{(t+1+t_0)^2} + 4\nu G_1 + \frac{4\nu G_2}{t+1+t_0}\bigg) \nonumber \\
        &\times \exp\bigg(\frac{4\nu G_3\log(t+t_0+1)}{(t+t_0+1)^2} +\frac{4\nu G_4}{5(t_0+1)^5(t+1+t_0)^2} \bigg).
    \end{align} Applying Markov's inequality, we then get, for any $\epsilon > 0$
    \begin{equation*}
        \Prob\lp f(\oxtp) - f^\star > \epsilon \rp = \Prob\big( \nu F_{t+1} > \nu n(t+1+t_0)\epsilon \big) \leq \exp\lp -\nu n(t+1+t_0)\epsilon \rp\E\big[\big(\nu F_{t+1} \big) \big].
    \end{equation*} Using \eqref{eq:bdd-mgf}, it can be readily verified that for any $\delta \in (0,1)$, choosing 
    \begin{align}\label{eq:eps-1}
        \epsilon_t^1 &= \frac{\nu^{-1}\log(\nicefrac{1}{\delta}) + 4G_1}{n(t+t_0)} + \frac{4G_2}{n(t+t_0)^2} + \frac{(t_0+2)^3\Delta_f + 4G_3\log(t+t_0) +4G_4/5(t_0+1)^5}{n(t+t_0)^3}, 
    \end{align} results in $\Prob\lp f(\oxt) - f^\star > \epsilon_t^1 \rp \leq \delta$. Next, using Proposition \ref{prop:descent-ineq} with $x = \xit$ and $y = \oxt$, we get
    \begin{align*}
        f(\xit) &\leq f(\oxt) + \langle \nabla f(\oxt), \xit - \oxt\rangle + \frac{L}{2}\|\xit-\oxt\|^2 \\
        &\stackrel{(i)}{\leq} f(\oxt) + \frac{1}{2L}\|\nabla f(\oxt)\|^2 + \frac{L}{2}\|\xit - \oxt\|^2 + \frac{L}{2}\|\xit-\oxt\|^2 \\
        &\stackrel{(ii)}{\leq} f(\oxt) + f(\oxt) - f^\star + L\|\xit-\oxt\|^2,
    \end{align*} where in $(i)$ we used Proposition \ref{prop:Young} with $\epsilon = L$, while $(ii)$ follows from Proposition \ref{prop:descent-ineq}. Subtracting $f^\star$ from both sides and averaging over all users $i \in [n]$, we get
    \begin{equation}\label{eq:global-local-cost}
        \frac{1}{n}\sum_{i \in [n]}\big(f(\xit)-f^\star\big) \leq 2\big(f(\oxt)-f^\star\big) + \frac{L}{n}\|\bxt-\obxt\|^2.
    \end{equation} We now consider two events, $A_{t,\epsilon} \coloneqq \lcb \omega: f(\oxt)-f^\star > \epsilon \rcb$ and $B_{t,\epsilon} \coloneqq \lcb \omega: \frac{L}{n}\|\bxt - \oxt\|^2 > \epsilon \rcb$. From the previous analysis, we know that, for any $\delta \in (0,1)$ and $\epsilon_t^1$ from \eqref{eq:eps-1}, we have $\Prob\big( A_{t,\epsilon_t^1}\big) \leq \delta$. Similarly, using Markov's inequality and Lemma \ref{lm:consensus-str-cvx} with $K = L$, we have, for any $\epsilon > 0$
    \begin{align*}
        \Prob\bigg( \frac{L}{n}\|\bxt - \obxt\|^2 > \epsilon \bigg) &= \Prob\lp \nu K_{t}\| \bxt - \obxt\|^2 > \nu n(t+t_0+1)\epsilon\rp \\
        &\leq \exp(-\epsilon\nu n(t+t_0+1))\E\big[\exp\big(\nu K_{t}\|\bxt - \obxt\|^2 \big) \big] \\
        &\leq \exp\Bigg( -\epsilon\nu n(t+t_0+1) + K_t\nu\bigg(\frac{4a^2\lambda^2(n\sigma^2 + \nicefrac{5\|\nbfs\|^2(1 + \nicefrac{9L^2}{\mu^2})}{(1-\lambda)})}{(1-\lambda)(t+t_0)^2}\bigg)\Bigg) \\
        &\times \exp\Bigg(K_t\nu\bigg(\frac{32a^4n\sigma^2\lambda^2L^2}{(1-\lambda)^2(t+t_0)^3} +\frac{20a^2\lambda^2L^2(1+t_0)^{6}\|\bx^1-\bxs\|^2}{(1-\lambda)^2(t+t_0)^{8}} \bigg) \Bigg).
    \end{align*} Therefore, it can be readily seen that for any $\delta \in (0,1)$, choosing 
    \begin{align}\label{eq:eps-2}
        \epsilon_t^2 = \frac{\nu^{-1}\log(\nicefrac{1}{\delta})}{n(t+t_0+1)} &+ \frac{144\kappa\lambda^2(n\sigma^2 + \nicefrac{5\|\nbfs\|^2(1 + 9\kappa^2)}{(1-\lambda)})}{\mu(1-\lambda)n(t+t_0)^2} \nonumber \\ 
        &+ \frac{41472\kappa^3\sigma^2\lambda^2}{\mu(1-\lambda)^2(t+t_0)^3} +\frac{720\kappa^2\lambda^2L(1+t_0)^{6}\|\bx^1-\bxs\|^2}{(1-\lambda)^2n(t+t_0)^{8}},  
    \end{align} results in $\Prob\big(B_{t,\epsilon_t^2}\big) \leq \delta$. Finally, let $C_t \coloneqq \lcb \omega: \frac{1}{n}\sum_{i \in [n]}\big(f(\xit)-f^\star\big) > 2\epsilon_t^1+\epsilon_t^2 \rcb$. From \eqref{eq:global-local-cost} it readily follows that, for any $\delta \in (0,\nicefrac{1}{2})$, we have
    \begin{equation*}
        \Prob(C_t) \leq \Prob\big(A_{t,\epsilon_t^1}\cap B_{t,\epsilon_t^2}\big) \leq \Prob\big(A_{t,\epsilon_t^1}\big) + \Prob\big(B_{t,\epsilon_t^2}\big) \leq 2\delta.
    \end{equation*} Therefore, for any $\delta \in (0,1)$, with probability at least $1 - \delta$, we get
    \begin{align*}
        &\frac{1}{n}\sum_{i \in [n]}\big(f(\xit)-f^\star\big) = \bigO\bigg(\frac{\nu^{-1}\log(\nicefrac{2}{\delta}) + \sigma^2\kappa(1+\log(2d))/\mu}{n(t+t_0)} + \frac{\kappa\lambda^2(1+L)(n\sigma^2+\|\nbfs\|^2\nicefrac{(1+9\kappa^2)}{(1-\lambda))}}{\mu(1-\lambda)n(t+t_0)^2} \\
        &+ \frac{(2+t_0)^3\Delta_f}{n(t+t_0)^3} + \frac{\kappa^3\sigma^2\lambda^2(\kappa\log(t+t_0) + 1)}{\mu(1-\lambda)^2(t+t_0)^3} + \frac{\kappa^3\lambda^2L(1+t_0)\|\bx^1-\bxs\|^2}{(1-\lambda)^2n(t+t_0)^3} + \frac{\kappa^2\lambda^2L(1+t_0)^{6}\|\bx^1-\bxs\|^2}{(1-\lambda)^2n(t+t_0)^{8}} \bigg).
    \end{align*} Finally, it can be verified that the conditions on $a$, $t_0$ and $\nu$ in the statement of the theorem ensure that all the step-size conditions are satisfied, completing the proof.
\end{proof}

We remark that, similarly to the proof of Theorem \ref{thm:main-non-cvx}, one can analyze the strongly convex case with a fixed step-size, resulting in the dependence on some terms, e.g., optimality and iterate gaps $\Delta_f$ and $\|\bx^1 - \bx^\star\|$, decaying exponentially fast, i.e., $\bigO\big((\|\bx^1-\bxs\|^2 + \Delta_f)e^{-CT}\big)$, for some $C > 0$, as shown in, e.g., \cite{pmlr-v119-koloskova20a} for MSE guarantees in decentralized, or \cite{liu2024revisiting} for HP guarantees in centralized settings. For simplicity, we omit this analysis.

\section{Numerical Experiments}\label{sup:num}

In this section we provide some further numerical results and details omitted from the main body. Subsection \ref{subsec:synth} presents results on synthetic data, while Subsection \ref{subsec:real} presents results on real data. 

\subsection{Synthetic Data}\label{subsec:synth}

\textbf{Methodology.} We consider an instance of \eqref{eq:decentr-opt}, with $f_i(x) = \frac{1}{2}x^\top A_ix + b_i^\top x$, where $A_i \in \R^{d \times d}$ is positive definite, making each $f_i$ strongly convex, with \emph{unbounded gradients}. We consider an undirected communication network $G = (V,E)$, corresponding to a randomly generated Erd\H{o}s--R\'enyi graph with connectivity parameter $p = 0.8$, while the weight matrix $W \in \R^{n \times n}$ is computed using the Metropolis-Hastings weight scheme, e.g., \cite{XIAO200465}. When queried by user $i \in [n]$ in iteration $t \geq 1$, the \sfo returns a noisy gradient of $f_i$ evaluated at $\xit$, i.e., $\git = A_i\xit + b_i + \zit$, where $\zit \in \R^d$ is a zero-mean Gaussian random vector, making the noise consistent with assumption \textbf{(A4)}. We use the \dsgd method outlined in Algorithm \ref{alg:dsgd}, with the time-varying step-size schedule $\alpha_t = \frac{1}{t+1}$ and shared initializatio $x^1_i = 0$, for all $i \in [n]$. We are interested in testing the following two facets of our theory.
\begin{enumerate}
    \item \emph{Exponentially decaying tails} - we want to verify that the tail probability decays at an exponential scale, as defined in \eqref{eq:hp-definition} and predicted in Theorem \ref{thm:main-dsgd-str-cvx}.

    \item \emph{Linear speed-up} - we want to verify that the tail probability decays faster as the number of users increases, as predicted in Theorem \ref{thm:main-dsgd-str-cvx}.  
\end{enumerate}

To verify these two facets, we measure the performance in terms of the \emph{empirical tail probability} $\Prob^t_{n,\varepsilon}$, computed as follows. We first run \dsgd for $T$ iterations and repeat it over $R$ runs. Next, for each $t \in [T]$, we use Monte-Carlo sampling to create a set $S_t$ of indices from $[R]$, of some fixed length $|S_t| = S$. For any $\varepsilon > 0$, the empirical probability is computed as
\begin{equation*}
    \Prob^t_{n,\varepsilon} = \frac{1}{S}\sum_{r \in S_t}\mathbb{I}\bigg(\frac{1}{n}\sum_{i \in [n]}\|x_i^{t,r} - x^\star\|^2 > \varepsilon\bigg),
\end{equation*} where $x_i^{t,r} \in \R^d$ is the model of user $i$ in iteration $t$ and run $r$, with $x^\star = \argmin_{x \in \R^d}f(x)$ being the solution of the global problem and $\mathbb{I}(A)$ being the indicator of event $A$. The empirical tail probability is a proxy to the true tail probability and is the main metric in our experiments. To further illustrate our results, we also compute the \emph{empirical mean-squared error} $\E^t_n$, which is the optimality gap at time $t$, averaged across all users and runs, i.e., $\E^t_n = \frac{1}{nR}\sum_{i \in [n]}\sum_{r \in [R]}\|x_i^{t,r} - x^\star\|^2$. We next present the results.

\textbf{Exponentially decaying tails.} To verify that the (empirical) tail probability decays at an exponential rate, we consider a fixed network of $n = 30$ users. For local costs, each matrix $A_i$ is generated using python's \texttt{sklearn} library function \texttt{make\_sparse\_spd\_matrix}, with dimension $d = 50$ and value $\texttt{alpha} = 0.9$, while vectors $b_i$ are drawn from a multivariate normal distribution $\mathcal{N}(\mathbf{0}_d,\sigma_i^2 I_d)$, where $\sigma_i^2 = 1$ for $i \in \{1,\ldots,10\}$, $\sigma_i^2 = 2$ for $i \in \{11,\ldots,20\}$, and $\sigma^2_i = 4$ for $i \in \{21,\ldots,30\}$, ensuring that the data is heterogeneous across users. We run \dsgd for $T = 10000$ iterations and repeat across $R = 5000$ runs. To compute the empirical probability, in each iteration we draw $S = 3000$ Monte-Carlo samples and consider threshold values $\varepsilon = \big\{10^{-2},10^{-3},10^{-4}\big\}$. The results are presented in Figure \ref{fig:ex-dec-tail}, where the left plot shows the MSE behaviour, while the right plot shows the tail probability behaviour. We can see that the empirical tail probability decays exponentially fast for all values of $\varepsilon$, as predicted by our theory. Note that for the threshold value $\varepsilon = 10^{-4}$, the tail probability starts decaying exponentially after approximately $t = 6000$ iterations, which is consistent with the MSE behaviour on the left plot, where we can see that it takes \dsgd around the same number of iterations to reach the average accuracy $\E_n^t = 10^{-4}$.    

\begin{figure*}[!ht]
\centering
\begin{tabular}{cc}
\includegraphics[scale=0.45]{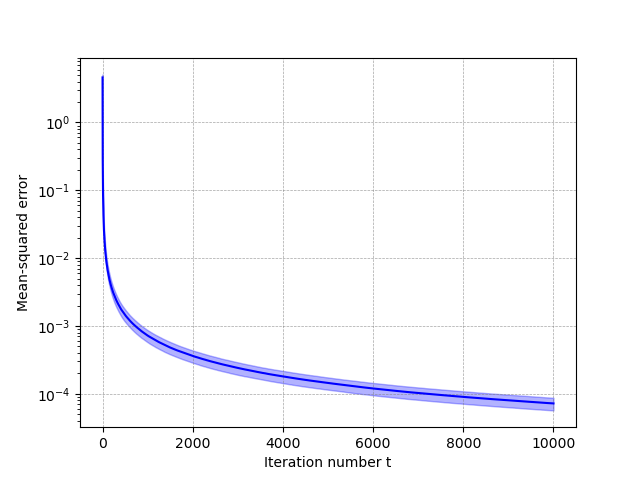}
&
\includegraphics[scale=0.45]{plots/tail_prob.png}
\end{tabular}
\caption{Performance of \dsgd, in the MSE sense (left) and HP sense (right). We can see that \dsgd achieves an exponential tail decay for all values of threshold $\varepsilon$. For the threshold $\varepsilon = 10^{-4}$, the tail probability starts decaying exponentially after approximately $t = 6000$ iterations, which is consistent with the MSE behaviour, where we can see that \dsgd takes around the same number of iterations to reach the average accuracy $\E^t_n = 10^{-4}$.}
\label{fig:ex-dec-tail}
\end{figure*}

\textbf{Linear speed-up.} To verify that linear speed-up in the number of users is achieved, we consider three networks, with $n = \{25,50,100\}$ users. In all three cases, we use Erd\H{o}s--R\'enyi communications graphs with Metropolis-Hastings weights that satisfy $\lambda = \|W - J\| \approx 0.6$. To ensure that the effect of heterogeneity is consistent across different networks, we fix the matrices $A_i$, i.e., $A_i \equiv A + I_d$, where $A \in \R^{d \times d}$ is generated using \texttt{sklearn}'s \texttt{make\_sparse\_spd\_matrix}, with $d = 50$ and $\texttt{alpha} = 0.9$, while vectors $b_i$ are given by $b_i = \beta_i\mathbf{1}_d$, where $\beta_i \in \{-2,-1,0,1,3\}$ are selected uniformly at random and in equal proportion across users (i.e., one fifth of users has $\beta_i = -2$, one fifth has $\beta_i = -1$, etc). Generating the networks and costs in this manner ensures that both network connectivity and user heterogeneity are constant across all three network settings, allowing us to properly capture the effect of linear speed-up. We run \dsgd for $T = 3000$ iterations and repeat across $R = 1000$ runs. To compute the empirical probability, in each iteration we draw $S = 600$ Monte-Carlo samples, with thresholds $\varepsilon = \big\{10^{-2},10^{-3},10^{-4}\big\}$. The results are presented in Figure \ref{fig:lin-speed}, where the plots left to right and top to bottom respectively show the MSE and tail probability behaviours for different values of $\varepsilon$. We can again see that the empirical tail probability decays exponentially fast for all values of $\varepsilon$, with the decay being consistently faster for larger number of users, demonstrating linear speed-up in the HP sense. Finally, note that in the lower right plot (i.e., threshold $\varepsilon = 10^{-4}$) the network with $n = 25$ users has a constant tail probability $\Prob^t_{n,\varepsilon} = 1$, while for the network with $n = 50$ the tail probability is slowly starting to decrease, which is again consistent with the MSE behaviour on the upper left plot, as the accuracy $\E^t_{n} = 10^{-4}$ for these two networks is not reached (on average) in the allocated number of iterations.    

\begin{figure*}[!ht]
\centering
\begin{tabular}{cc}
\includegraphics[scale=0.45]{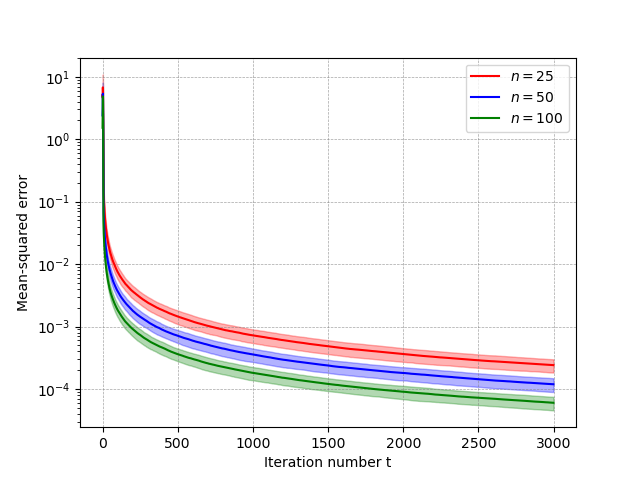}
&
\includegraphics[scale=0.45]{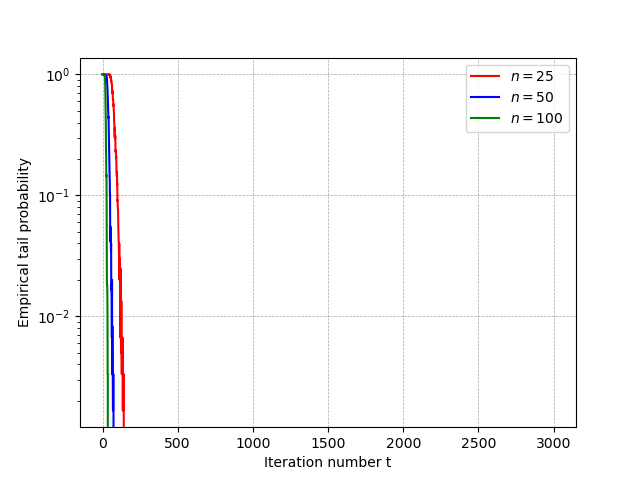} \\
\includegraphics[scale=0.45]{plots/tail_prob_eps=0.001.png}
&
\includegraphics[scale=0.45]{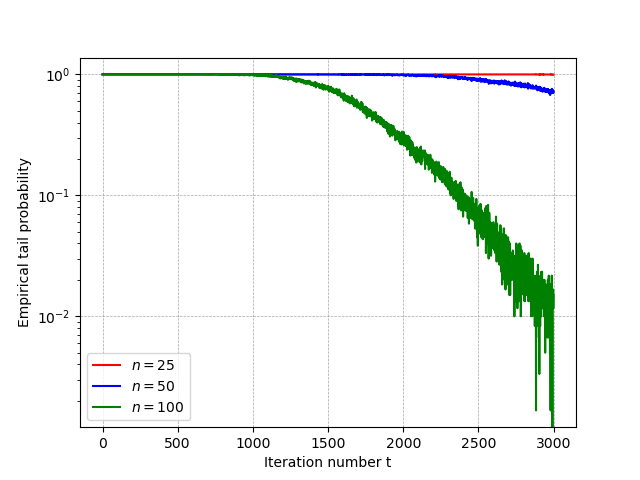}
\end{tabular}
\caption{Linear speed-up of \dsgd, in the MSE and HP sense. Left to right and top to bottom: MSE performance and tail decay with threshold $\varepsilon = \big\{10^{-2},10^{-3},10^{-4}\big\}$. We can see that \dsgd consistently achieves faster exponential tail decay for larger networks, across all values of threshold $\varepsilon$, illustrating the effect of the linear speed-up in the HP sense.}
\label{fig:lin-speed}
\end{figure*}

\subsection{Real Data}\label{subsec:real}

\paragraph{Methodology.} We validate our theory on a non-convex problem, considering a binary logistic regression classification task with a non-convex regularizer, e.g., \cite{antoniadis2011penalized}, with users' local costs given by
\begin{equation*}
    f_i(x) = \frac{1}{m_i}\sum_{r \in [m_i]} \log\big(1+\exp(-y_{i,r}\langle h_{i,r}, x\rangle)\big) + \eta \sum_{k \in [d]} \frac{[x]_k^2}{1+[x]_k^2},     
\end{equation*} where $h_{i,r} \in \mathbb{R}^{d}$ and $y_{i,r} \in \{+1, - 1\}$ are the feature vector and associated label, $\eta > 0$ is a user-specified penalty parameter, while $[x]_k$ denotes the $k$-th component of the model vector $x$. To evaluate the performance, we use the ``\texttt{mushroom}'', ``\texttt{a9a}'' and ``\texttt{ijcnn1}'' datasets from the LIBSVM library \cite{chang2011libsvm}, each providing a varying degree of heterogeneity. We split the data uniformly across agents, so that all agents have an equal-sized local dataset before training, i.e., $m_i = \frac{m}{n}$, where $m$ is the size of the original dataset. Similarly to synthetic data experiments, we use Erd\H{o}s-R\'enyi graphs with Metropolis-Hastings weights. For each experiment, we fix the following parameters: step-size $\alpha = 0.1$ and penalty parameter $\eta = 0.1$, giving a large learning rate and a non-trivial effect of the non-convex regularizer. Since the gap to the global minima is not computable in the non-convex case, we evaluate the performance using the average gradient norm-squared, i.e., $G^{t,r}_n = \frac{1}{nt}\sum_{\tau \in [t]}\sum_{i \in [n]}\|\nabla f(x_i^{\tau,r})\|^2$, so that the empirical tail probability is computed as $\Prob^t_{n,\epsilon} = \frac{1}{R}\sum_{r \in [R]}\mathbb{I}\big(G^{t,r}_n > \epsilon\big)$. Similarly to the previous section, we also compute and visualize the average performance across all runs, i.e., $\E^t_n = \frac{1}{R}\sum_{r \in [R]}G^{t,r}_n$. We note that in our experiments on real data, the stochastic noise comes from the mini-batch choice, hence the resulting noise is \emph{not necessarily sub-Gaussian}.

\begin{figure*}[t]
\centering
\begin{tabular}{cc}
\includegraphics[scale=0.45]{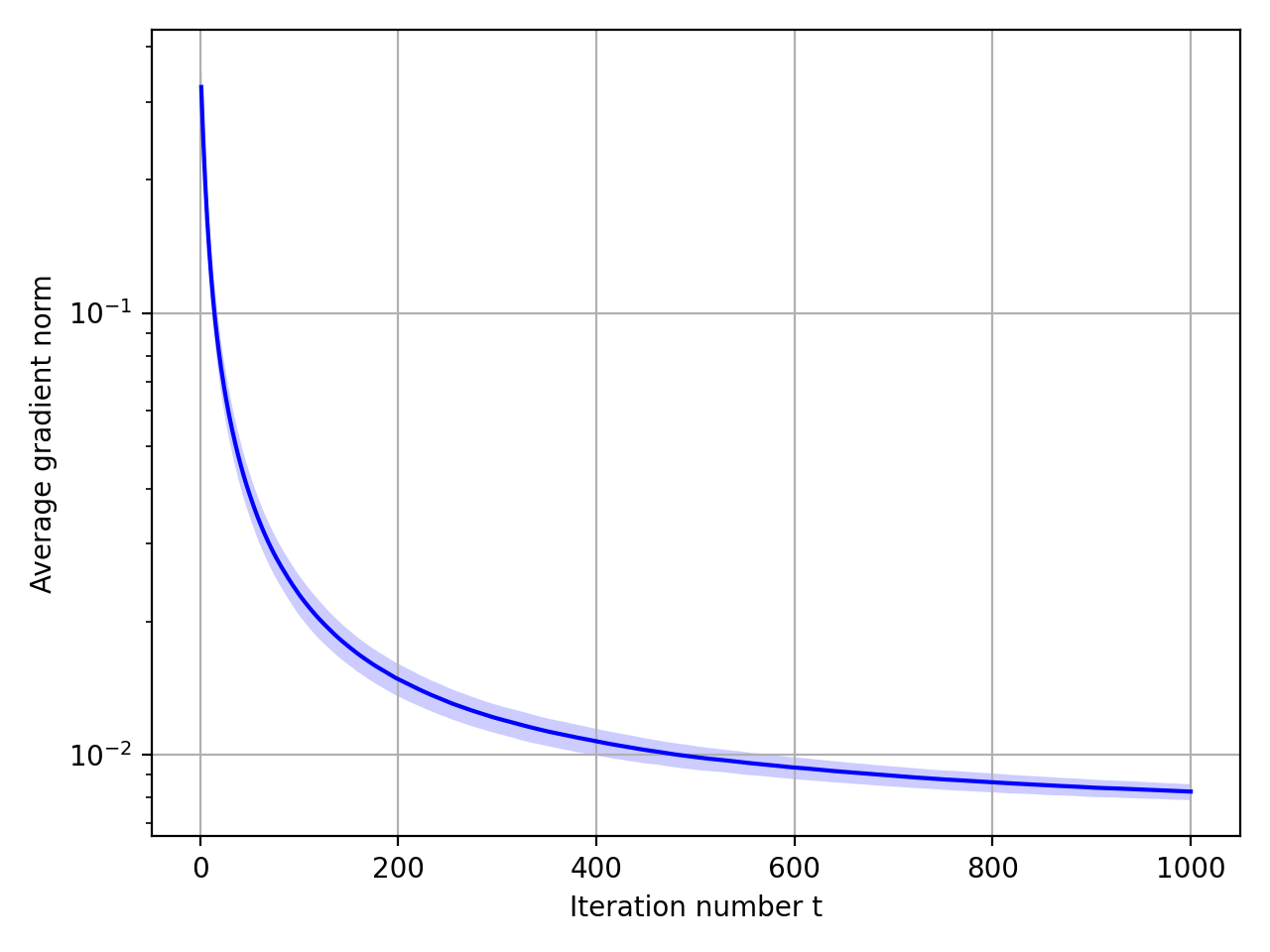}
&
\includegraphics[scale=0.45]{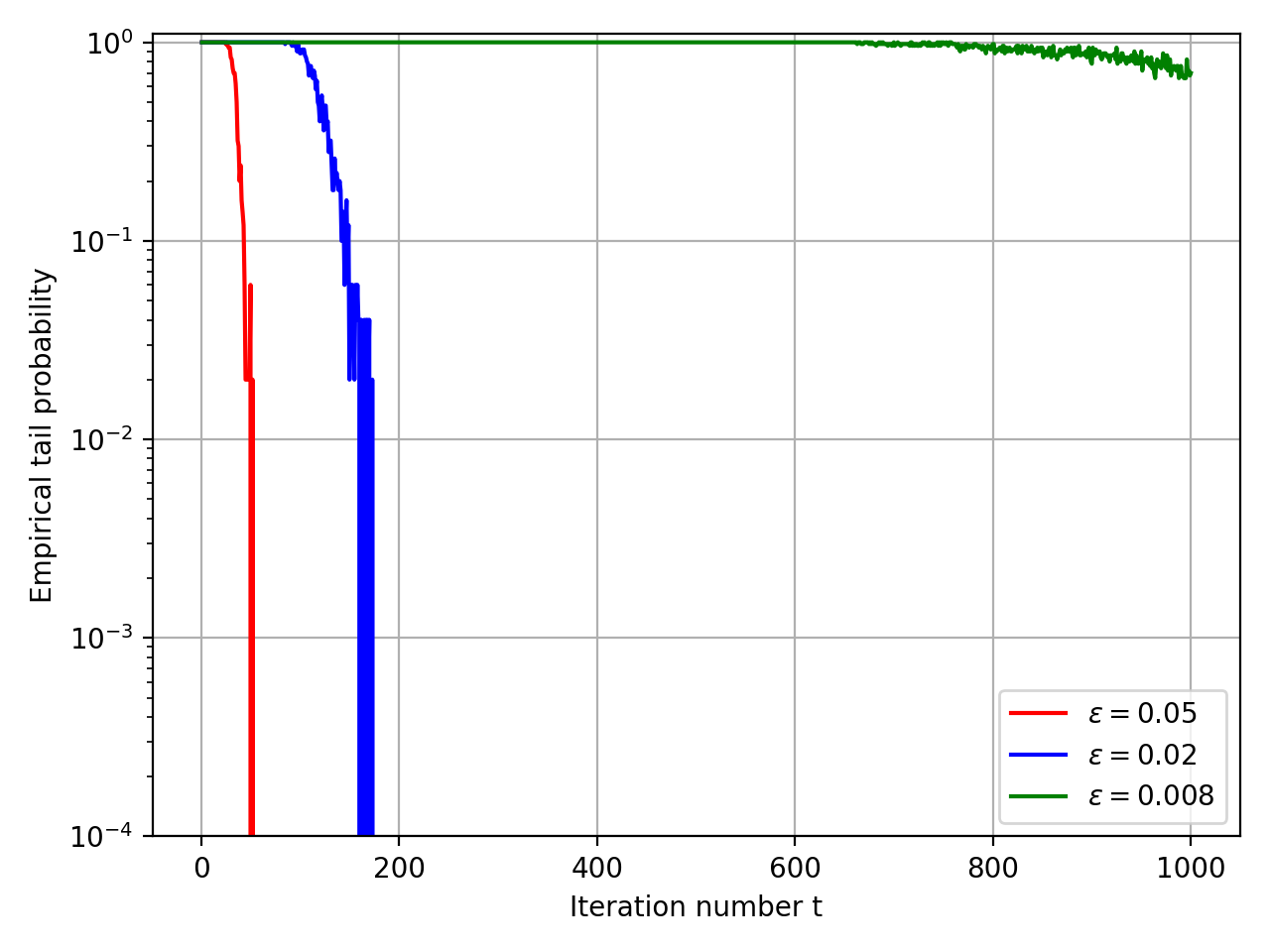}
\end{tabular}
\begin{tabular}{cc}
\includegraphics[scale=0.45]{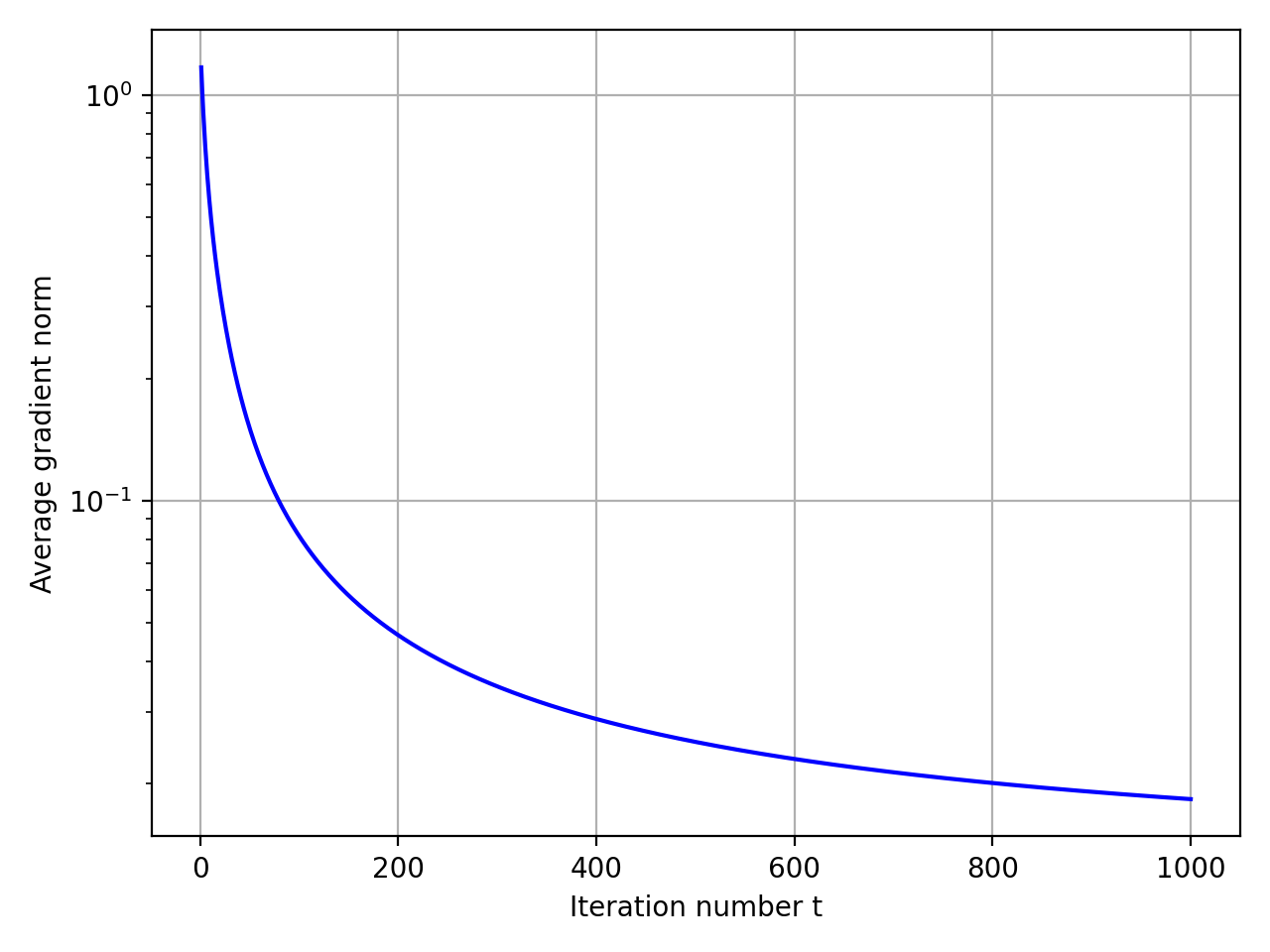}
&
\includegraphics[scale=0.45]{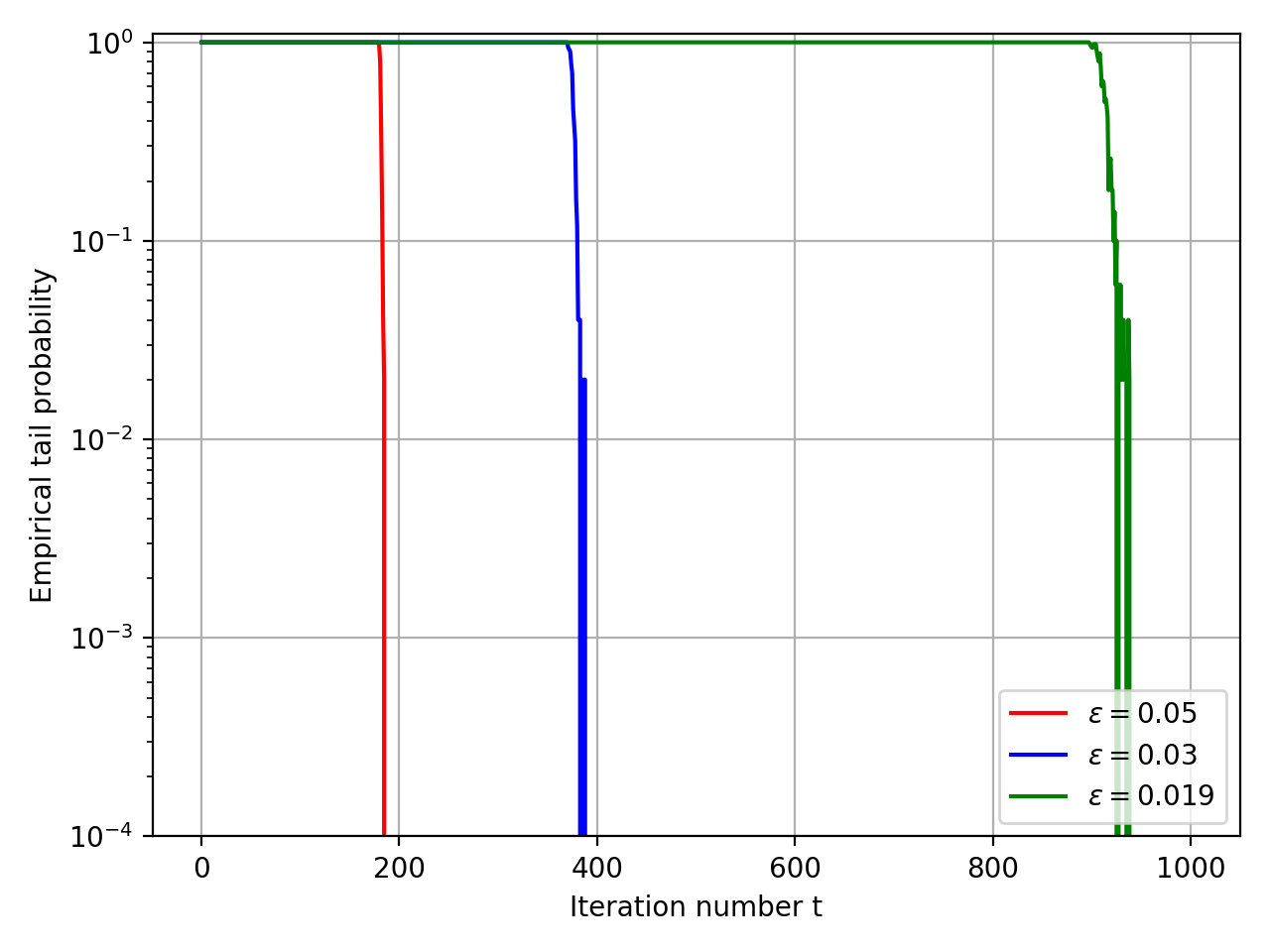}
\end{tabular}
\caption{Exponential tail decay on real data. Left to right: average gradient norm across all runs and its empirical tail probability. Top to bottom: performance on ``\texttt{a9a}'' and ``\texttt{mushroom}'' datasets. For the empirical tail probability, we use threshold values $\epsilon = \{0.01,0.05,0.01\}$ for the respective datasets. We can see that \dsgd consistently achieves exponentially decaying tails, across both datasets and different threshold values.}
\label{fig:real-exp}
\end{figure*}

\paragraph{Exponentially decaying tails.} We start by testing the decay rate of the empirical tail probability. We fix a network of $n = 30$ agents communicating over an Erd\H{o}s-R\'enyi graph. We run both methods for $T = 1000$ iterations, repeated across $R=100$ runs. The results are presented in Figure \ref{fig:real-exp}, where the top and bottom rows respectively correspond to ``\texttt{a9a}'' and ``\texttt{mushroom}'' datasets, while figures left to right visualizes the average error $\E^t_n$ and the empirical tail probability $\Prob^t_{n,\epsilon}$. The threshold values are chosen as $\varepsilon = \{0.05, 0.02, 0.008 \}$ for ``\texttt{a9a}'' and $\varepsilon = \{0.05, 0.03, 0.019 \}$ for ``\texttt{mushroom}'' datasets, based on the achieved average error. As can be seen from the figure, \dsgd achieves exponential tail decay across both datasets and all values of threshold $\varepsilon$.

\paragraph{Linear speedup.}  To verify that \dsgd achieves linear speed-up, we again consider three networks with $n = \{10, 30, 50\}$ agents, communicating over Erd\H{o}s--R\'enyi graphs. To ensure that the network connectivity is consistent, we enforce the condition $\lambda = \|W-J\| \approx 0.6$ on the resulting weight matrices, with $\alpha, T, R$ remaining unchanged. We consider the ``\texttt{a9a}'' and ``\texttt{ijcnn1}'' datasets, plotting the empirical tail probability versus the number of users for different error thresholds. In particular, we use the thresholds $\varepsilon = \{0.05,0.02,0.01\}$ for ``\texttt{a9a}'' and $\varepsilon = \{0.005, 0.0035, 0.003\}$ for ``\texttt{ijcnn1}'' dataset. The results are presented in Figures \ref{fig:real-speed-1} and \ref{fig:real-speed-2}, respectively. We can again see that the average gradient norm and empirical tail probability decay faster with larger $n$ across both datasets and different values of $\varepsilon$, with the speed-up effect more evident on the more heterogeneous ``\texttt{a9a}'' dataset.

\begin{figure*}[!ht]
\centering
\begin{tabular}{cc}
\includegraphics[scale=0.45]{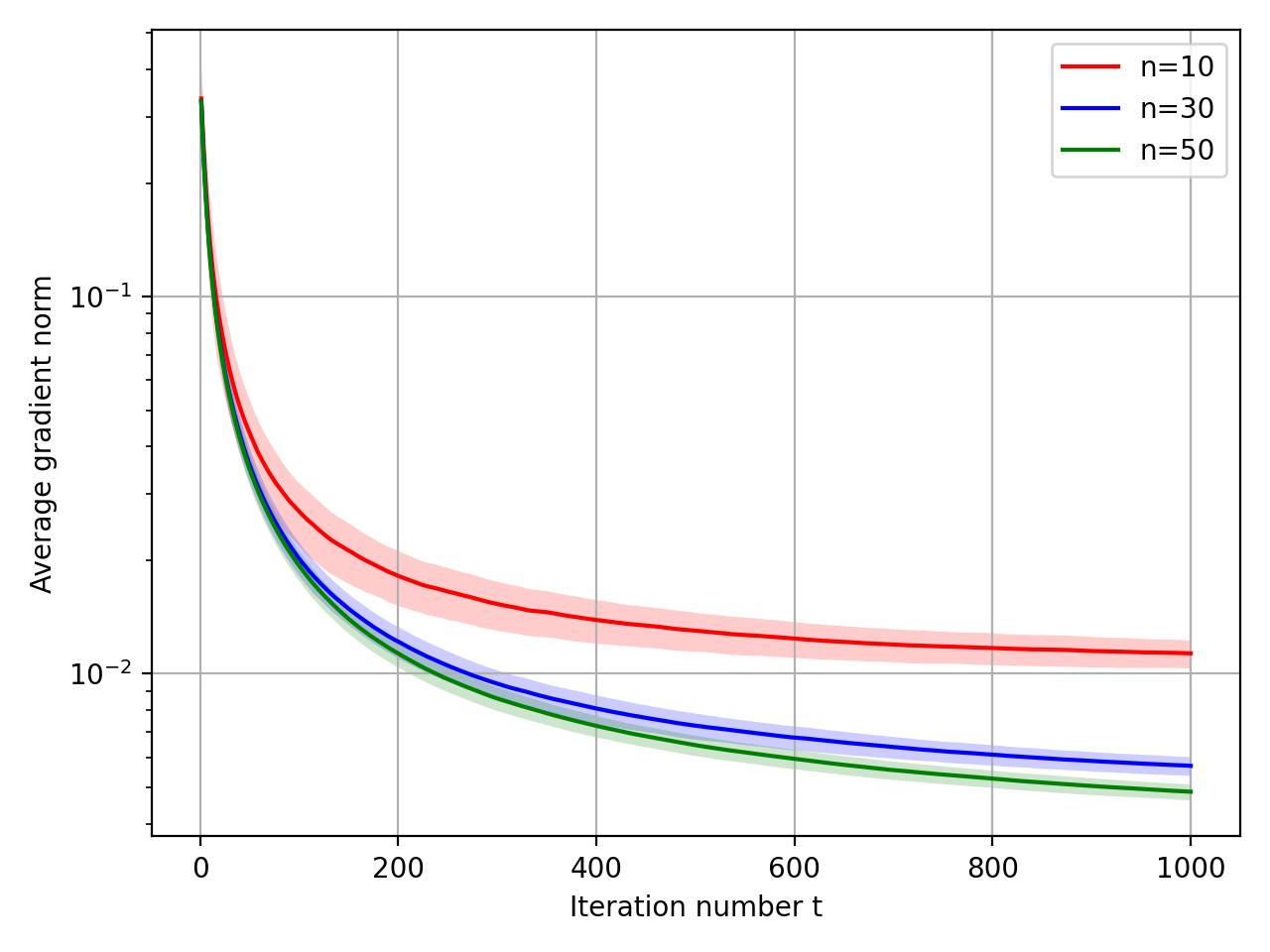}
&
\includegraphics[scale=0.45]{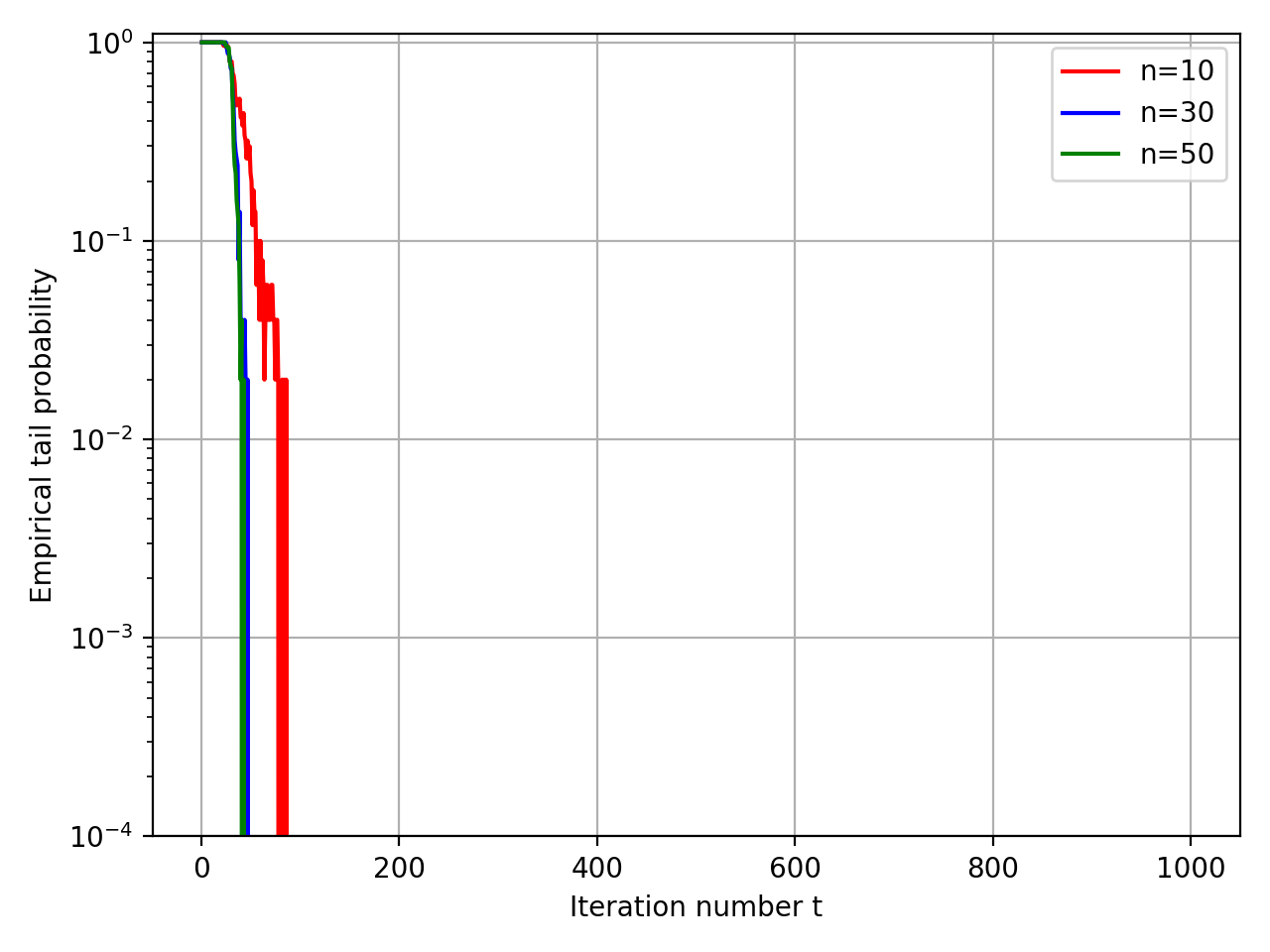} 
\\
\includegraphics[scale=0.45]{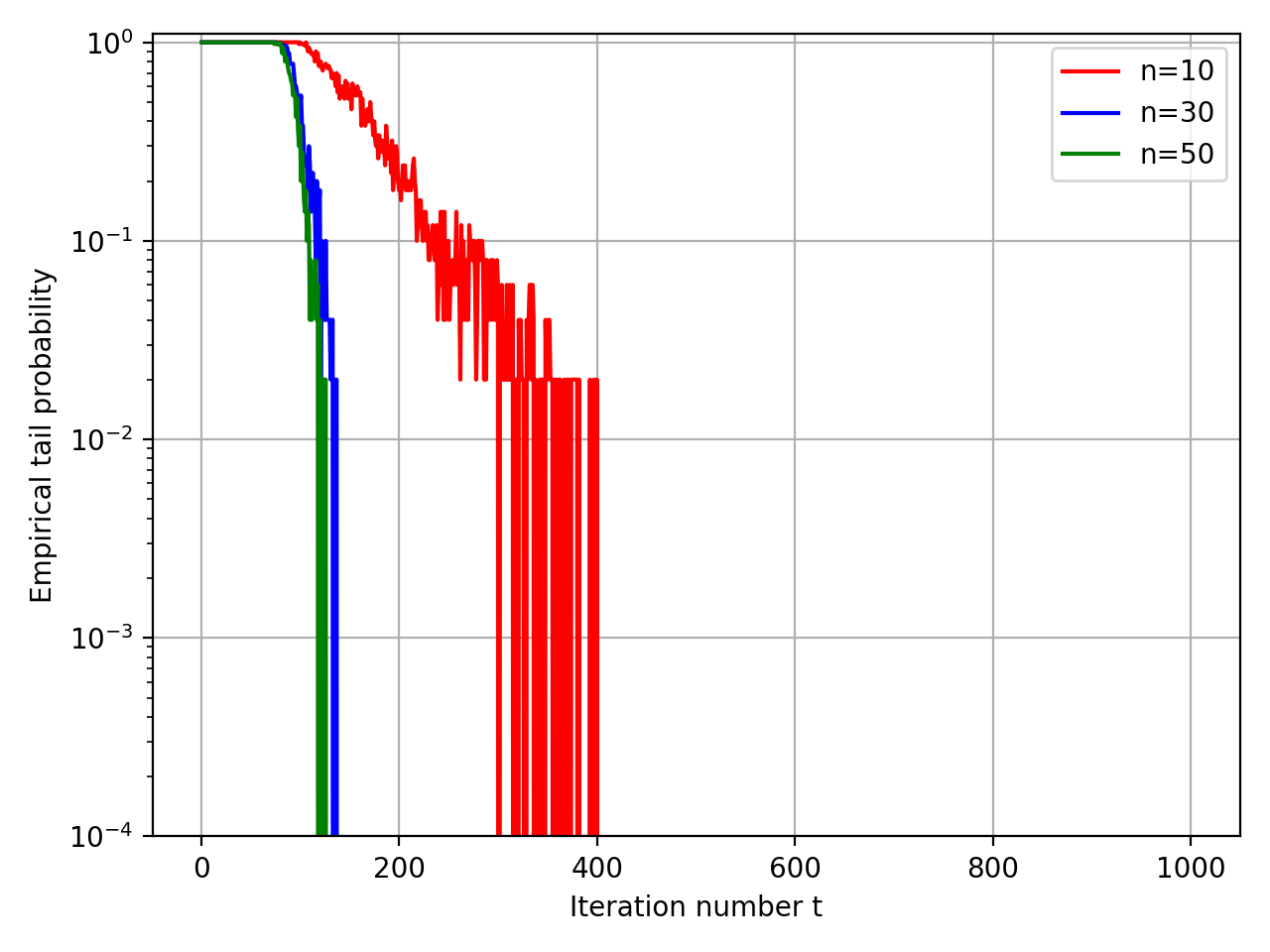}
&
\includegraphics[scale=0.45]{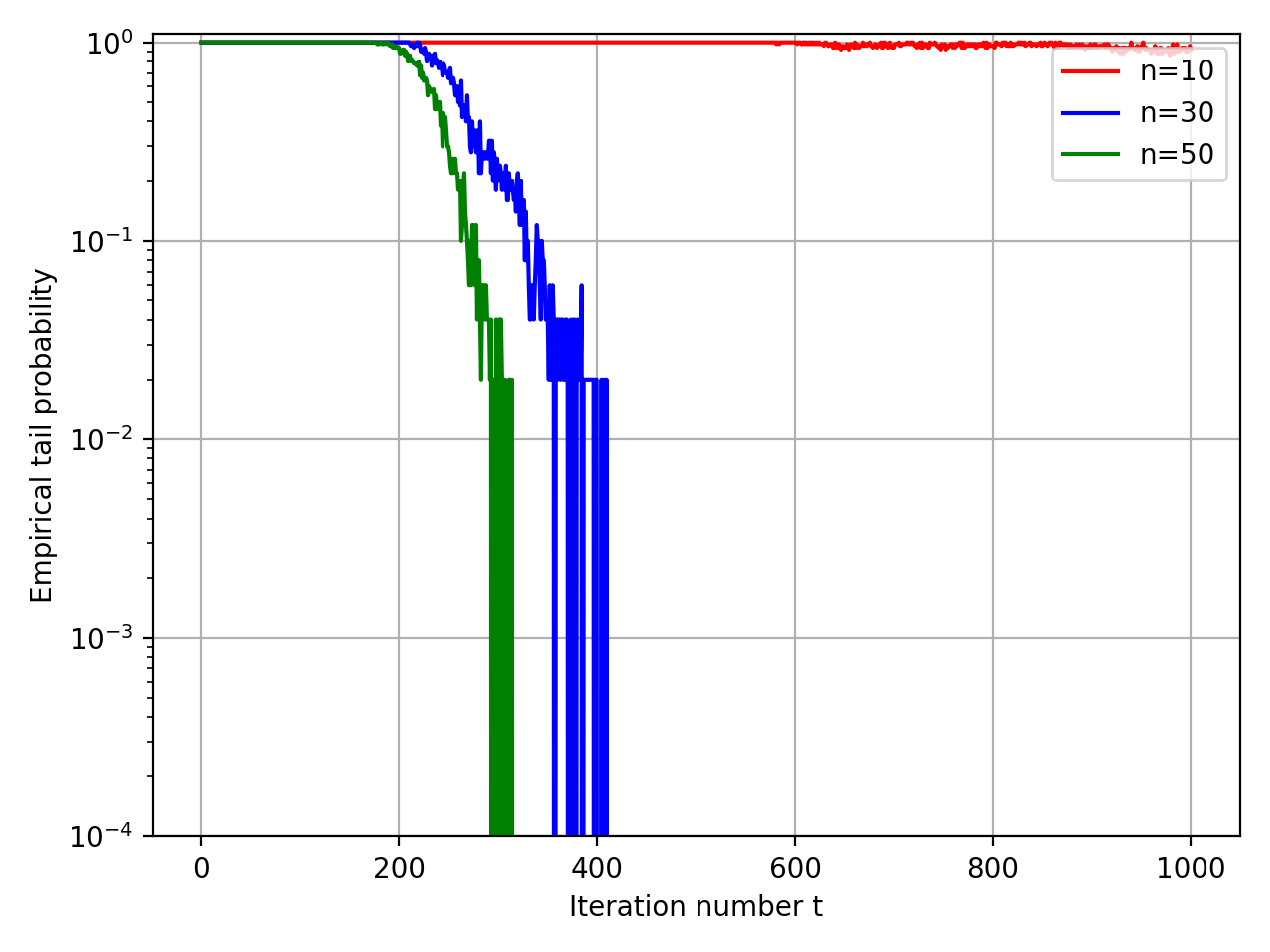}
\end{tabular}
\caption{Linear speed-up of \dsgd on the ``\texttt{a9a}'' dataset. Left to right and top to bottom: MSE performance and tail decay with threshold $\varepsilon = \big\{0.05,0.02,0.01\big\}$. We can see that \dsgd consistently achieves faster exponential tail decay for larger networks, across all values of threshold $\varepsilon$, illustrating the effect of the linear speed-up in the HP sense.}
\label{fig:real-speed-1}
\end{figure*}

\begin{figure*}[!ht]
\centering
\begin{tabular}{cc}
\includegraphics[scale=0.45]{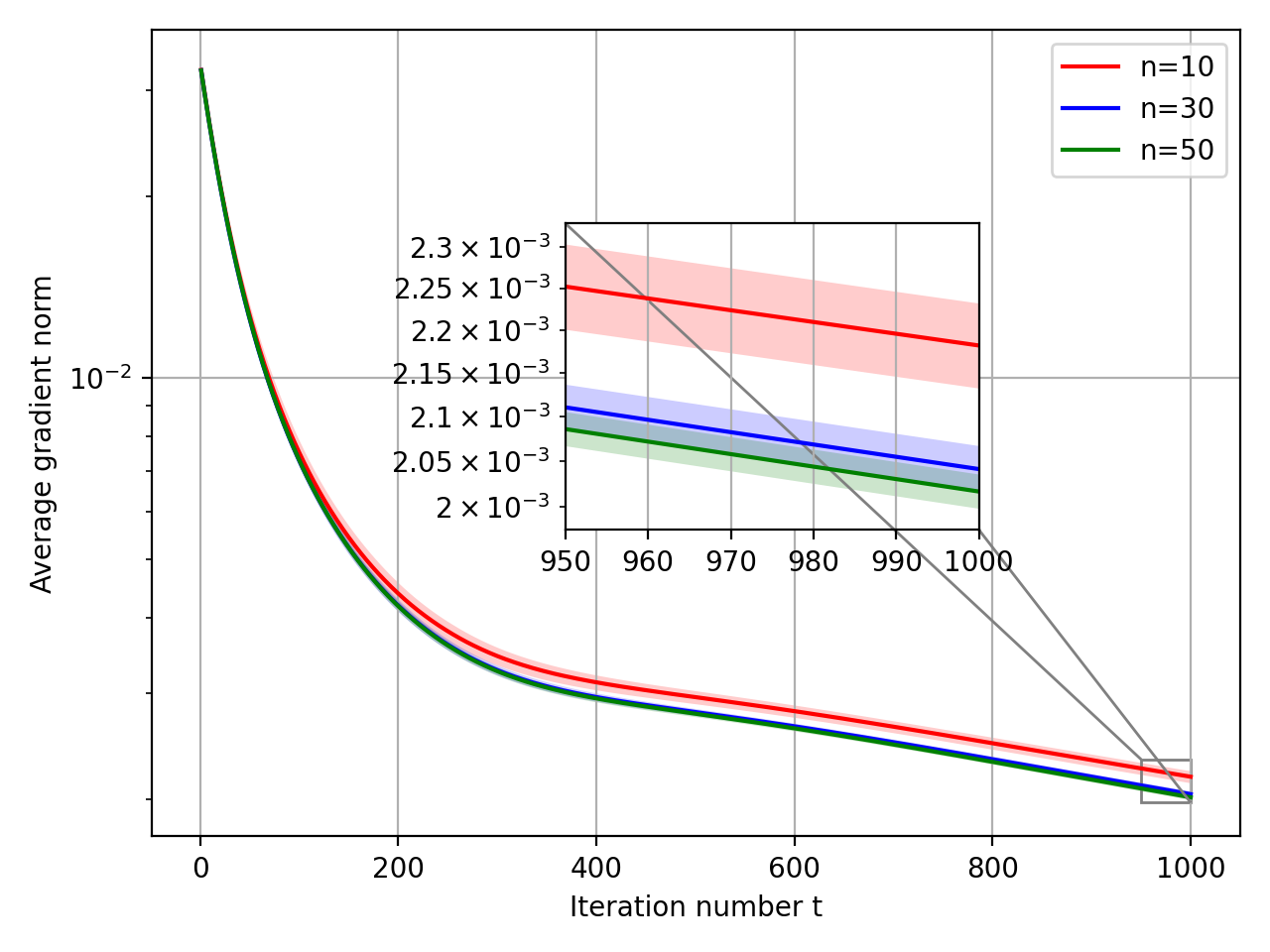}
&
\includegraphics[scale=0.45]{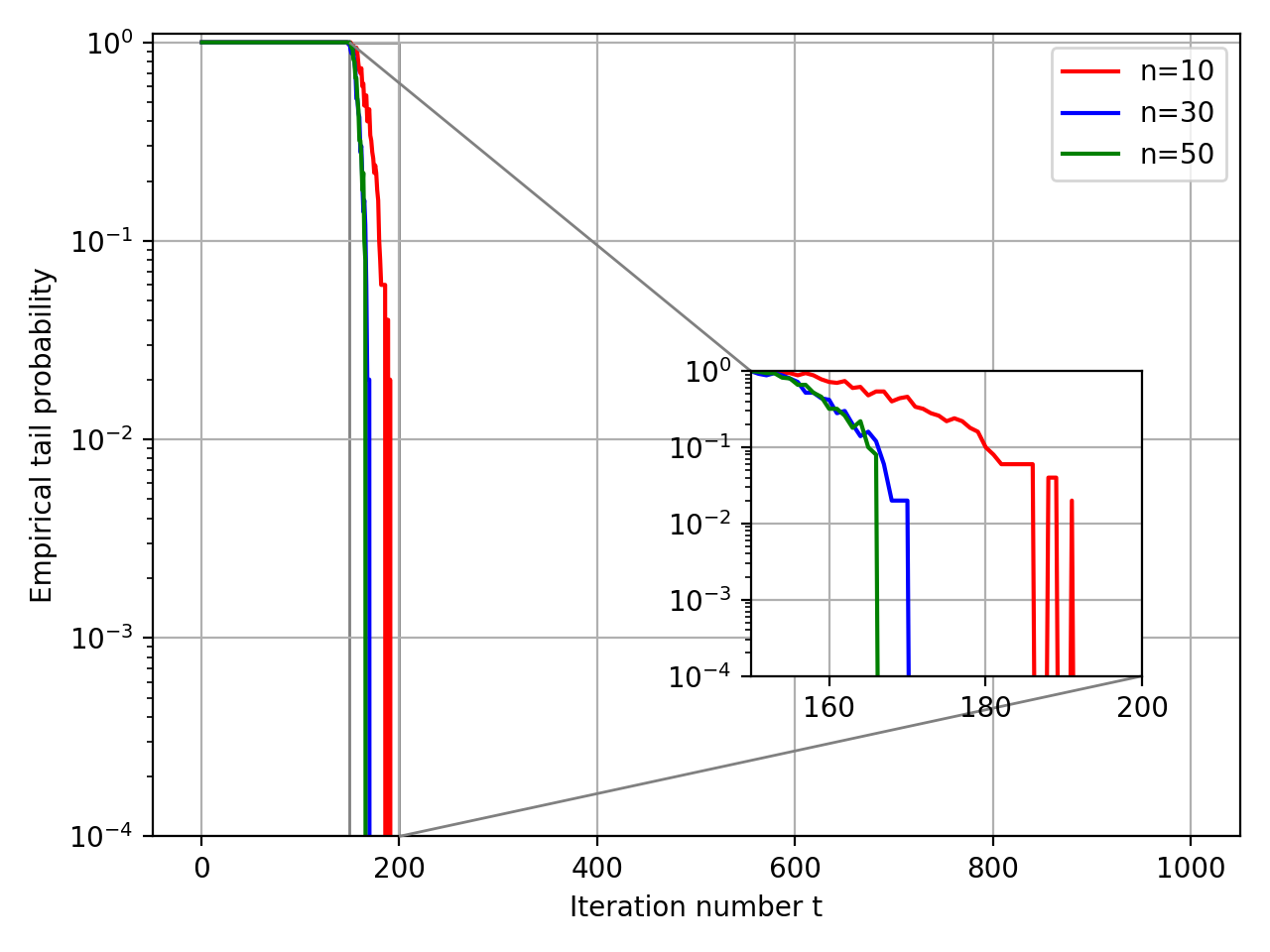} 
\\
\includegraphics[scale=0.45]{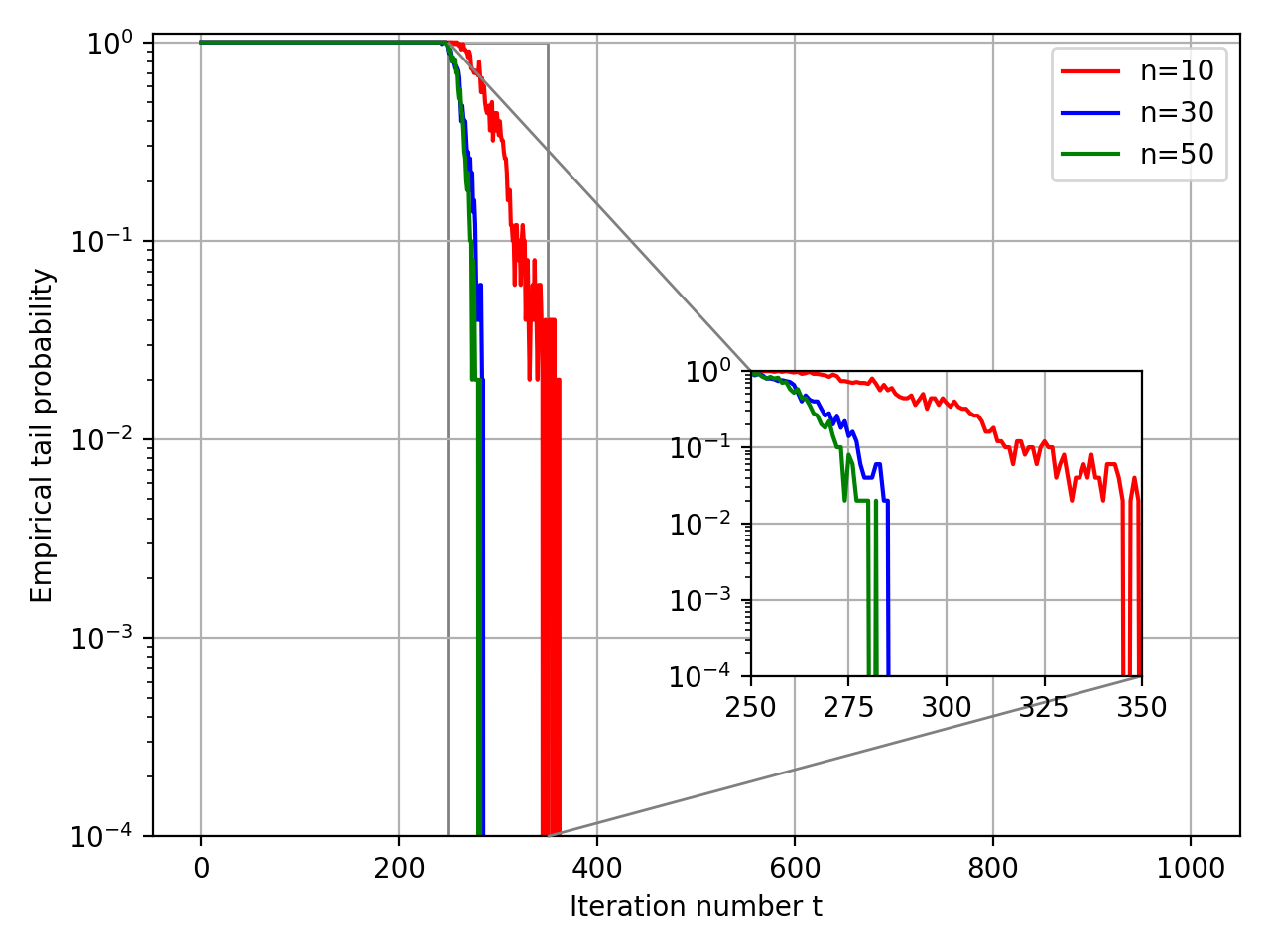}
&
\includegraphics[scale=0.45]{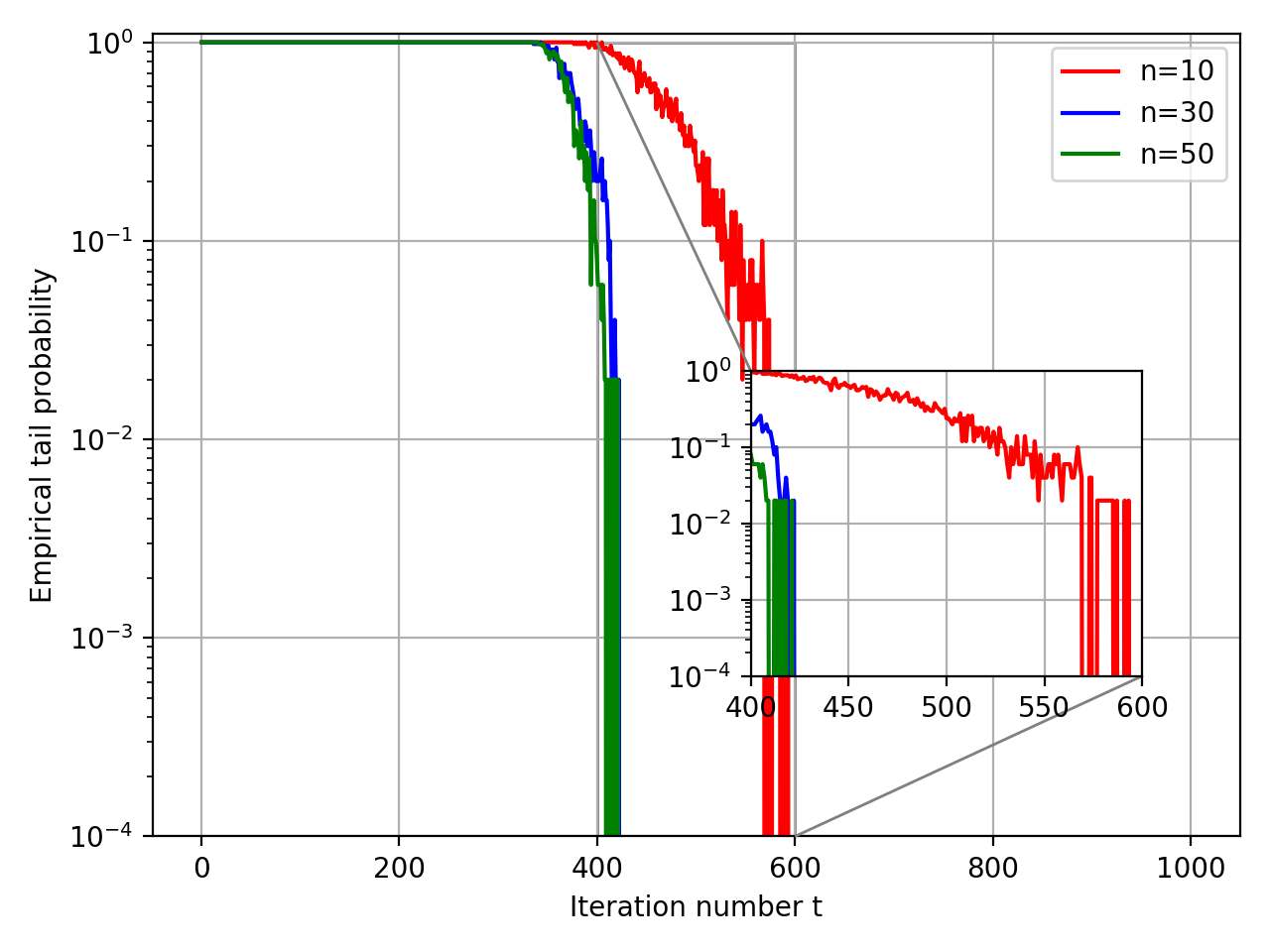}
\end{tabular}
\caption{Linear speed-up of \dsgd on the ``\texttt{ijcnn1}'' dataset. Left to right and top to bottom: MSE performance and tail decay with threshold $\varepsilon = \big\{0.005,0.0035,0.003\big\}$. We can see that \dsgd consistently achieves faster exponential tail decay for larger networks, across all values of threshold $\varepsilon$, illustrating the effect of the linear speed-up in the HP sense.}
\label{fig:real-speed-2}
\end{figure*}

\section{Deriving the Transient Time}\label{sup:tran-time}

In this section we provide details on the transient time resulting from our HP bounds. Similarly to \citet{unified-refined}, we focus on the dependence on network connectivity and number of users, ignoring other problem related constants. 

\textbf{Transient time for non-convex costs.} Recall that the bound in Theorem \ref{thm:main-non-cvx} is of the form
\begin{equation}\label{eq:tran-t-non-conv-1}
    \bigO\bigg( \frac{1}{\sqrt{nT}} + \frac{1}{CT} + \frac{\Delta_x}{(1-\lambda^2)T} + \frac{n}{(1-\lambda)^2T} \bigg),    
\end{equation} where the problem related constant $C$ depends on the network connectivity via the expression $C = \bigO\Big(\min\Big\{\frac{1-\lambda}{\lambda},\frac{\sqrt[3]{n}(1-\lambda)^{2/3}}{\lambda^2}\Big\}\Big)$. Since we are interested in the worst-case transient time (corresponding to poor network connectivity, i.e., $\lambda \approx 1$), it follows that $C = \bigO\Big(\frac{1-\lambda}{\lambda }\Big)$. For simplicity, assume that users have a shared initialization, i.e., $x^1_i = x_j^1$, for all $i,j \in [n]$, so that $\Delta_x = 0$. In this case, the bound in \eqref{eq:tran-t-non-conv-1} becomes
\begin{equation}\label{eq:tran-t-non-conv-2}
    \bigO\bigg( \frac{1}{\sqrt{nT}} + \frac{1}{(1-\lambda)T} + \frac{n}{(1-\lambda)^2T} \bigg).    
\end{equation} It is now evident that, for any $T \geq \max\Big\{\frac{n}{(1-\lambda)^2},\frac{n^3}{(1-\lambda)^4} \Big\} = \frac{n^3}{(1-\lambda)^4}$, the bound in \eqref{eq:tran-t-non-conv-2} becomes $\bigO\Big(\frac{1}{\sqrt{nT}}\Big)$, implying a transient time of at most $\bigO\Big(\frac{n^3}{(1-\lambda)^4}\Big)$. This matches the transient time of \dsgd obtained from the MSE rate in \cite{pmlr-v119-koloskova20a}, see, e.g., Table I in \cite{unified-refined}.

\textbf{Transient time for strongly convex costs.} Recalling the expressions in \eqref{eq:eps-1} and \eqref{eq:eps-2}, it follows that the full bound for strongly convex costs is of the form
\begin{align}\label{eq:tran-t-str-cvx-1}
    \bigO&\bigg(\frac{1}{n(T+t_0)} + \frac{1}{(1-\lambda)(T+t_0)^2} + \frac{1}{n(1-\lambda)^2(T+t_0)^2} + \frac{t_0^3}{n(T+t_0)^3} \nonumber \\ 
    &+ \frac{\log(T + t_0)}{(1-\lambda)^2(T+t_0)^3} + \frac{t_0}{n(1-\lambda)^2(T+t_0)^3} + \frac{t_0^6}{n(1-\lambda)^2(T+t_0)^8} \bigg),
\end{align} where $t_0$ depends on the network connectivity via the expression $t_0 = \Omega\Big(\max\Big\{\frac{\lambda}{1-\lambda},\frac{\lambda^2}{1-\lambda},\frac{1}{1-\lambda}\Big\}\Big)$. Since we are again interested in the worst-case transient time, it follows that $t_0 = \Omega\Big(\frac{1}{1-\lambda}\Big)$, so that the bound in \eqref{eq:tran-t-str-cvx-1} becomes
\begin{align}\label{eq:tran-t-str-cvx-2}
    &\quad\quad\bigO\bigg(\frac{1}{n(T+t_0)} + \frac{1}{(1-\lambda)(T+t_0)^2} + \frac{1}{n(1-\lambda)^2(T+t_0)^2} \nonumber \\
    &+ \frac{\log(T + t_0)}{(1-\lambda)^2(T+t_0)^3} + \frac{1}{n(1 - \lambda)^3(T+t_0)^3} + \frac{1}{n(1-\lambda)^8(T+t_0)^8} \bigg).
\end{align} It can be seen that for $T \geq \max\Big\{\frac{n}{1-\lambda}, \frac{1}{(1-\lambda)^2}, \frac{\sqrt{n}}{1-\lambda}, \frac{1}{(1-\lambda)^{\frac{3}{2}}},\frac{1}{(1-\lambda)^{\frac{8}{7}}} \Big\} = \max\Big\{\frac{n}{1-\lambda}, \frac{1}{(1-\lambda)^2} \Big\}$, the bound in \eqref{eq:tran-t-str-cvx-2} becomes $\mathcal{O}\Big(\frac{\log(T+t_0)}{n(T+t_0)}\Big)$, implying a transient time of at most $\mathcal{O}\Big(\max\Big\{\frac{n}{1-\lambda}, \frac{1}{(1-\lambda)^2}\Big\}\Big)$. Interestingly, this is strictly sharper than the transient time of \dsgd obtained from the MSE rate in \cite{pmlr-v119-koloskova20a}, which is of order $\mathcal{O}\Big(\frac{n}{(1-\lambda)^2}\Big)$, see, e.g., Table II in \cite{unified-refined}. As such, our HP analysis provides either matching or strictly sharper transient times for \dsgd compared to the transient times stemming from MSE results, further highlighting the tightness of our analysis.

\section{Achieving Linear Speed-up}\label{sup:lin-speed-up}

In this section we provide further discussions on the main results and techniques needed to obtain linear speed-up. 

\textbf{Linear speed-up for non-convex costs.} The crucial result for obtaining linear speed for non-convex costs is Lemma \ref{lm:noise-properties}, which establishes the variance reduction benefits of decentralized methods in the HP sense. To see why this is the case, note that using a fixed step-size $\alpha_t \equiv \alpha$ and following the same steps as in the proof of Theorem \ref{thm:main-non-cvx}, while relying only on assumption \textbf{(A4)} to control the noise, one would get a final bound where the leading term is of the form
\begin{equation}\label{eq:bdd-no-speedup}
    \bigO\bigg(\frac{1}{\alpha T} + \alpha\sigma^2\bigg).
\end{equation} It is obvious that no choice of $\alpha$ leading to a linear speed-up is possible, i.e., such that \eqref{eq:bdd-no-speedup} becomes $\bigO\Big(\frac{1}{\sqrt{nT}}\Big)$. On the other hand, using Lemma \ref{lm:noise-properties} to handle the noise, we are able to establish a bound where the leading term is of the form 
\begin{equation}\label{eq:bdd-lin-speedup}
    \bigO\bigg(\frac{1}{\alpha T} + \frac{\alpha\sigma^2}{n}\bigg),
\end{equation} where it is easy to see that choosing $\alpha = \Theta\Big(\frac{\sqrt{n}}{\sqrt{T}}\Big)$ leads to the desired linear speed-up. Works like \cite{dsgd-high-prob,dsgd-online-noncooperative-games,dsgd-online-stochastic} provide no results on the variance reduction benefit of decentralized learning in the HP sense, instead relying only on the definition of sub-Gaussianity.\footnote{\citet{dsgd-online-noncooperative-games} consider the problem of online noncooperative games and show HP convergence guarantees in terms of dynamic regret $R_i(t)$, for \emph{each individual $i \in [n]$}. While this is different from the metrics considered in our work (i.e., \emph{user-averaged} gradient norm/optimality gap), the analysis in \cite{dsgd-online-noncooperative-games} suffers from the same shortcomings as other decentralized works, i.e., no result showing variance reduction in the HP sense is provided, making linear speed-up impossible to achieve even if $\frac{1}{n}\sum_{i \in [n]}R_i(t)$ is considered.} As such, even if a fixed step-size was used, the authors would be unable to obtain a bound of the form \eqref{eq:bdd-lin-speedup} and achieve linear speed-up. As discussed in Appendix \ref{sup:non-conv}, using a time-varying step-size for non-convex costs inevitably leads to a loss of linear speed-up, as all terms decay at the same rate, leading to the bound
\begin{equation*}
    \bigO\bigg(\frac{1}{\sqrt{nT}} + \frac{n}{\sqrt{T}}\bigg).
\end{equation*} On the other hand, using a fixed step-size $\alpha = \Theta\Big(\frac{\sqrt{n}}{\sqrt{T}} \Big)$, we get the bound
\begin{equation*}
    \bigO\bigg(\frac{1}{\sqrt{nT}} + \frac{n}{T}\bigg),
\end{equation*} where the term not achieving linear speed-up is of higher order. This is consistent with MSE guarantees, e.g., \cite{pmlr-v119-koloskova20a,ran-improved,unified-refined}, where a fixed step-size is also needed to achieve linear speed-up.

\textbf{Linear speed-up for strongly convex costs.} In addition to Lemma \ref{lm:noise-properties}, a tighter control of the MGF of the process of interest is required, which is achieved in Lemma \ref{lm:mgf-bound-str-cvx}. To see why this is the case, consider the existing techniques in centralized settings, e.g., \cite{harvey2019tight,pmlr-v206-bajovic23a}, where the authors show via induction that the MGF of the process of interest is uniformly bounded, i.e., that there exists a $B > 0$, such that, for all $t \geq 1$ and $\nu \in (0,B]$
\begin{equation}\label{eq:harvey-mgf-sup}
    \E[\exp(\nu F^t)] \leq \exp\Big(\frac{\nu}{B}\Big).    
\end{equation} Recalling that $F^t = n(t+t_0)\big(f(\oxt) - f^\star\big)$, it then follows from \eqref{eq:harvey-mgf-sup} and Markov's inequality that, for any $\epsilon > 0$
\begin{equation*}
    \Prob\bigg(f(\oxt) - f^\star > \frac{\epsilon + 1}{n\nu(t+t_0)} \bigg) = \Prob\big(\nu F^t > \epsilon + 1 \big) \leq \exp(-\epsilon).
\end{equation*} Choosing $\nu = B$, we then get for any $\delta \in (0,1)$, with probability at least $1 - \delta$
\begin{equation}\label{eq:harvey-bound}
    f(\oxt) - f^\star \leq \frac{\log(\nicefrac{1}{\delta}) + 1}{nB(t+t_0)}.
\end{equation} While the above bound seemingly achieves linear speed-up, we now show that the constant $B$ must depend on $n$, effectively canceling the linear speed-up. To see that this is the case, we go back to \eqref{eq:harvey-mgf-sup} and note that for $t = 1$, we have 
\begin{equation*}
    \E[\exp(\nu F^1)] = \exp\Big(\nu n(t_0+1)\big(f(\ox^1) - f^\star\big)\Big) \leq \exp\Big(\frac{\nu}{B}\Big),    
\end{equation*} where the equality follows from the definition of $F^t$ and the fact that the initial models are chosen deterministically, while the inequality follows if and only if $B \leq \frac{1}{n(t_0+1)(f(\ox^1) - f^\star)}$. As such, in order for \eqref{eq:harvey-mgf-sup} to hold, the constant $B$ has to decay with $n$, canceling the linear speed-up in \eqref{eq:harvey-bound} and resulting in a HP bound of the form
\begin{equation*}
    f(\oxt) - f^\star \leq \frac{\log(\nicefrac{1}{\delta}) + 1}{B^\prime(t+t_0)},
\end{equation*} where $B^\prime > 0$ is independent of $n$ and the linear speed-up is lost. On the other hand, Lemma \ref{lm:mgf-bound-str-cvx} implies that
\begin{align*}
    \E\big[\exp\big( \nu F_{t+1} \big) \big] &\leq \exp\bigg(\frac{(t_0+2)^{3}\nu \Delta_f}{(t+1+t_0)^2} + 4\nu G_1 + \frac{4\nu G_2}{t+1+t_0} +\frac{4\nu G_3\log(t+t_0+1)}{(t+t_0+1)^2} +\frac{4\nu G_4}{5(t_0+1)^5(t+1+t_0)^2} \bigg),
\end{align*} for some problem related constants $G_i > 0$, $i \in [4]$, providing a tighter and more fine-grained control of the MGF compared to \eqref{eq:harvey-mgf-sup} and crucially allowing for a bound where the leading term achieves linear speed-up in the number of users (as $G_1$ is independent of $n$). On the other hand, using a different approach to \citet{harvey2019tight,pmlr-v206-bajovic23a}, \citet{liu2024revisiting} establish the following HP bound for the last iterate of \sgd in centralized setting and smooth, strongly convex costs
\begin{equation*}
    f(x^t) - f^\star = \bigO\bigg(\frac{\log(t)\log(\nicefrac{1}{\delta})}{t+t_0} + \frac{1}{(t+t_0)^2} \bigg).
\end{equation*} While such a result could potentially be leveraged to maintain linear speed-up (provided that the constant growing linearly with $n$ affects only the higher-order term), it still accrues a factor of $\log(t)$ in the leading term, whereas the $\log$ term in our bound only affects higher-order terms. As such, Lemma \ref{lm:mgf-bound-str-cvx} provides the most fine-grained control of the MGF and is of independent interest when studying ``almost decreasing'' type processes, even in centralized settings.

\section{Extending to Time-varying Networks}\label{sup:time-var-net}

As discussed in the main body, to model the communication network, we use assumption \textbf{(A1)}, which subsumes fixed undirected networks, as well as a class of fixed directed networks. The authors in \cite{dsgd-high-prob} consider a more general case of directed, time-varying networks, facilitated by the following condition.

\textbf{(A1$^\prime$)} At any time $t \geq 1$, the network is given by a time-varying directed graph $G^t = (V,E^t,W^t)$, where $E^t = \{(i,j): \: \text{user } i \text{ can send a message to user } j \text{ in iteration } t\}$ and $W^t \in \R^{n \times n}$ is such that $[W^t]_{ij} \geq l > 0$ if $(i,j) \in E^t$, otherwise $[W^t]_{ij} = 0$, where $l \in (0,1)$. Moreover, $W^t$ is row-stochastic and balanced for any $t \geq 1$, i.e., $W^t\mathbf{1}_n = \mathbf{1}_n$ and $\sum_{j \in [n]}[W^t]_{ij} = \sum_{j \in [n]}[W^t]_{ji}$, for all $i \in V$. Finally, for some $U \in \N$ and all $t \geq 1$, the graph $G^t_U = (V,E_U^t)$, where $E_U^t = \bigcup_{k = tU}^{(t+1)U-1}E^k$, is strongly connected.\footnote{Recall that a directed graph $G = (V,E)$ is strongly connected if for any two vertices $i,j \in V$, there exists a path from $i$ to $j$ (and vice versa), see, e.g., \cite{cvetkovic_rowlinson_simic_1997,chung1997spectral}.}

Our work can be readily extended to the setting of time-varying networks, with \textbf{(A1)} replaced by \textbf{(A1$^\prime$)}. To see why this is the case, note that in our proofs we only use the fact that under \textbf{(A1)}, we have $\lambda = \|W - J\| \in [0,1)$. \citet{dsgd-high-prob} use a similar result for time-varying in  networks (see Lemma 1 in \cite{dsgd-high-prob}), namely that under \textbf{(A1$^\prime$)}, for any $t \geq s$
\begin{equation}\label{eq:time-varying}
    \Big|[\Phi(t,s)]_{ij} - \frac{1}{n}\Big| \leq M\rho^{t-s},
\end{equation} where $\Phi(t,s) = W^{t-1}\times \ldots \times W^s$ is the transition matrix from time $s$ to $t$ (with $\Phi(t,t) = I_n$), while $M > 0$ and $\rho \in (0,1)$ are graph related constants. As discussed in \cite{dsgd-high-prob}, the constant $\rho$ plays the same role of the network connectivity parameter, as $\lambda$ does in our work. It can be readily verified that replacing \textbf{(A1)} with \textbf{(A1$^\prime$)} and using \eqref{eq:time-varying} to control the consensus gap, our results readily go through, with $\lambda$ being replaced by $\rho$. As such, our results can be easily extended to the more general (directed) time-varying network setting considered in \cite{dsgd-high-prob}.

\section{On Sub-Gaussian Vectors and Dimension Dependence}\label{sup:dimension}

As mentioned in Section \ref{sec:proof-and-discuss} in the main body, our rates in Theorems \ref{thm:main-non-cvx} and \ref{thm:main-dsgd-str-cvx} exhibit a mild dependence on the problem dimension, of order $\log(d)$, with the dependence stemming from Lemma \ref{lm:noise-properties}. This is a consequence of working with random vectors, where we simultaneoulsy need to show a bound on the MGF of the \emph{inner product} $\langle \ozt,v \rangle$ for any $\Ft$-measurable vector $v \in \R^d$ (shown in point 2 of Lemma \ref{lm:noise-properties}), as well as on the MGF of the \emph{squared norm} $\|\ozt\|^2$ (shown in point 3 of Lemma \ref{lm:noise-properties}). While the inner product maintains some desirable properties, such as being zero-mean and linear (in the sense that $\langle \ozt, v \rangle = \frac{1}{n}\sum_{i \in [n]}\langle \zit,v\rangle $), this is not the case with the squared norm, which is neither zero-mean, nor linear. Therefore, trying to directly establish the variance reduction benefit of decentralized learning, i.e., that $\ozt$ is $\bigO\big(\frac{\sigma}{\sqrt{n}} \big)$-sub-Gaussian, in the sense of condition \textbf{(A4)}, fails to yield the desired result. In particular, recalling that $\sigma^2 = \frac{1}{n}\sum_{i \in [n]}\sigma_i^2$, we then have
\begin{align}\label{eq:sedrakyan}
    \E\bigg[\exp\bigg(\frac{\|\ozt\|^2}{\sigma^2} \bigg)\: \big\vert \: \Ft\bigg] &\stackrel{(i)}{\leq} \E\Bigg[\exp\Bigg(\frac{\big(\sum_{i \in [n]}\|\zit\|\big)^2}{n\sum_{i \in [n]}\sigma_i^2} \Bigg)\: \Big\vert \: \Ft\Bigg] \stackrel{(ii)}{\leq} \E\bigg[\exp\bigg(\frac{1}{n}\sum_{i \in [n]}\frac{\|\zit\|^2}{\sigma_i^2} \bigg)\: \big\vert \: \Ft\bigg] \nonumber \\
    &\stackrel{(iii)}{\leq} \prod_{i \in [n]}\Bigg(\E\bigg[\exp\bigg(\frac{\|\zit\|^2}{\sigma_i^2} \bigg)\: \big\vert \: \Ft\bigg]\Bigg)^{1/n} \stackrel{(iv)}{\leq} \prod_{i \in [n]}\exp\bigg(\frac{1}{n} \bigg) = \exp(1),
\end{align} where $(i)$ follows from Proposition \ref{prop:Jensen}, $(ii)$ follows from Sedrakyan's inequality, namely that
\begin{equation*}
    \frac{(\sum_{i \in [n]}a_i)^2}{\sum_{i \in [n]}b_i} \leq \sum_{i \in [n]}\frac{a_i^2}{b_i},
\end{equation*} which holds for any $n \in \N$, $a_i \in \R$ and $b_i > 0$, in $(iii)$ we used Proposition \ref{prop:Jensen} and the fact that noise is conditionally independent across users, while $(iv)$ follows from \textbf{(A4)}. Therefore, a direct approach yields that $\ozt$ is $\sigma$-sub-Gaussian, failing to show the variance reduction benefit. To circumvent this issue, we use a different argument, namely that condition \textbf{(A4)} is equivalent to \emph{norm-sub-Gausiannity}, i.e., that for any $\epsilon > 0$, we have 
\begin{equation*}
    \Prob\big(\|\zit\|^2 > \epsilon \: \vert \: \Ft \big) \leq 2\exp\bigg(\frac{\epsilon^2}{2\sigma_i^2}\bigg),
\end{equation*} which, combined with Proposition \ref{prop:concentration}, implies that
\begin{equation*}
    \Prob\big(\|\ozt\|^2 > \epsilon\big) \leq 2d\exp\bigg(\frac{n\epsilon}{3\sigma^2} \bigg).
\end{equation*} This result can be further leveraged to show the desired bound on the MGF of $\ozt$, while introducing the mild $\log(d)$ dependence. Note that when there is no noise, i.e., $\sigma = 0$, the dimension dependence in our bounds vanishes, highlighting that this dependence stems from the noise. As mentioned in \cite{jin2019short}, while mild, it is not clear if the dependence on $d$ can be completely removed. We now highlight three important facts. First, if a different definition of sub-Gaussian vectors is used, namely the one via inner products, which states that, for any $\Ft$-measurable vector $v$, we have
\begin{equation*}
    \E\big[\exp\big(\langle v,\zit\rangle\big) \: \vert \: \Ft\big] \leq \exp\bigg(\frac{\sigma_i^2\|v\|^2}{2} \bigg),
\end{equation*} this would inevitably lead to a worse dependence on the problem dimension, as in order to bound the MGF of the squared norm of the average noise, i.e., $\|\ozt\|^2$, one would need to use a similar argument as the one for Lemma 1 in \cite{jin2019short}, via covers of the unit sphere, picking up a factor of $d$ in the final rate, much worse than our $\log(d)$ dependence. Second, while the notion of sub-Gaussian random vectors used in \textbf{(A4)} is equivalent to the notion of norm-sub-Gaussian vectors in \cite{jin2019short} and stronger than the notion of sub-Gaussian vectors defined via inner products above, it is in fact the standard notion of sub-Gaussianity used in almost all HP results, in both centralized and decentralized settings, see \cite{nemirovski2009robust,ghadimi2013stochastic,li2020high,liu2023high,liu2024revisiting,dsgd-high-prob,dsgd-online-noncooperative-games,dsgd-online-stochastic} and references therein. Finally, we note that in the extreme regime of $n \ll \log(d)$, we can simply use the argument in \eqref{eq:sedrakyan}, to conclude that $\ozt$ is $\sigma$-sub-Gaussian, completely removing the dependence on dimension $d$, at the cost of losing linear speed-up in the number of users.

\section{General Step-Size for Strongly Convex Costs}\label{sup:gen-step}

In this section we extend the results for strongly convex costs under a more general step-size schedule. To that end, we consider a time-varying step-size schedule $\alpha_t \propto t^{-\eta}$, where $\eta \in (1/2,1]$. We start by providing some technical result. The first extends Proposition \ref{prop:Ran} for a general step-size schedule.

\begin{lemma}\label{lm:Ran}
    For any $c,t_0 > 0$, $\eta \in (1/2,1]$ and $0 \leq a \leq b$, we have
    \begin{equation*}
        \prod_{k = a}^b\Big(1 - \frac{c}{(k+t_0)^\eta} \Big) \leq \frac{(a+1+t_0)^{c\eta}}{(b+2+t_0)^{c\eta}}.
    \end{equation*}
\end{lemma}

\begin{proof}
    Using the inequality $1-x \leq e^{-x}$, it follows that
    \begin{equation}\label{eq:exp-switch}
        \prod_{k = a}^b\Big(1 - \frac{c}{(k+t_0)^\eta} \Big) \leq \prod_{k = a}^b\exp\bigg( - \frac{c}{(k+t_0)^\eta} \bigg) = \exp\bigg( -\sum_{k = a}^b \frac{c}{(k+t_0)^\eta} \bigg). 
    \end{equation} Next, note that for any $k \geq 1$
    \begin{equation*}
        \frac{1}{(k+t_0)^\eta} = \int_{k}^{k+1}\frac{1}{(k+t_0)^\eta}ds \geq \int_k^{k+1}\frac{1}{(s+t_0)^\eta}ds,
    \end{equation*} therefore, we have
    \begin{equation*}
        \sum_{k=a}^b\frac{c}{(k+t_0)^\eta} \geq \sum_{k = a}^b\int_k^{k+1}\frac{c}{(s+t_0)^\eta}ds.
    \end{equation*} We now differentiate between two cases: first, if $\eta = 1$, then $\int_k^{k+1}\frac{c}{(s+t_0)^\eta}ds = c\log\Big(\frac{k+t_0}{k+1+t_0} \Big)$, therefore
    \begin{equation*}
        \sum_{k=a}^b\frac{c}{(k+t_0)^\eta} \geq \sum_{k = a}^bc\log\bigg(\frac{k+t_0}{k+1+t_0} \bigg) = c\log\bigg(\frac{a+t_0}{b+t_0+1}\bigg). 
    \end{equation*} Plugging into \eqref{eq:exp-switch}, we get
    \begin{equation}\label{eq:bdd-case-1}
        \prod_{k = a}^b\Big(1 - \frac{c}{(k+t_0)^\eta} \Big) \leq \bigg(\frac{a+t_0}{b+t_0+1}\bigg)^c.
    \end{equation} Next, if $\eta \in (1/2,1)$, we use the fact that $\frac{1}{(s+t_0)^\eta} \geq \frac{\eta}{s+t_0+1}$, implying $\int_k^{k+1}\frac{c}{(s+t_0)^\eta}ds \geq c\eta\log\Big(\frac{k+t_0+1}{k+t_0+2}\Big)$, therefore
    \begin{equation*}
        \sum_{k=a}^b\frac{c}{(k+t_0)^\eta} \geq c\eta\sum_{k = a}^b\log\bigg(\frac{k+1+t_0}{k+2+t_0}\bigg) = c\eta\log\bigg(\frac{a+1+t_0}{b+2+t_0} \bigg).
    \end{equation*} Plugging into \eqref{eq:exp-switch}, we get
    \begin{equation}\label{eq:bdd-case-2}
        \prod_{k = a}^b\Big(1 - \frac{c}{(k+t_0)^\eta} \Big) \leq \bigg(\frac{a+1+t_0}{b+2+t_0}\bigg)^{c\eta}.
    \end{equation} The claim follows by combining \eqref{eq:bdd-case-1}-\eqref{eq:bdd-case-2} and noting that $\Big(\frac{a+t_0}{b+1+t_0}\Big)^c \leq \Big(\frac{a+1+t_0}{b+2+t_0}\Big)^c$.
\end{proof}

The next result extends Lemma \ref{lm:mgf-bound-str-cvx} for a general step-size.

\begin{lemma}\label{lm:mgf-bound-str-cvx-gen-step}
    Let $\{X^t\}_{t \in \N}$ be a sequence of random variables initialized by a deterministic $X^1 > 0$, such that, for some $\eta \in (1/2,1]$, some $C_1,C_2,C_3 > 0$ and every $t \geq 1$ 
    \begin{equation}\label{eq:almost-decrease-gen-step}
        \E[\exp(X^{t+1})] \leq \E\bigg[\exp\bigg(\Big(1-\frac{2}{(t+t_0)^{\eta}}\Big)X^t + \frac{C_1}{(t+t_0)^{2\eta}} + \frac{C_2}{(t+t_0)^{1+\eta}} + \frac{C_3}{t+t_0} \bigg)\bigg].
    \end{equation} If $t_0 \geq 2$, we have
    \begin{align*}
        \E[\exp(X^{t+1})] \leq \exp\bigg(\frac{X_1}{(t+1+t_0)^{2\eta}} + \frac{C_1}{(t+t_0)^{2\eta-1}} + \frac{C_2}{(t+t_0)^{\eta}} + C_3\bigg).
    \end{align*}
\end{lemma}

\begin{proof}
    Starting from \eqref{eq:almost-decrease-gen-step}, taking the logarithm, defining $Y_{t} \coloneqq \log\E[\exp(X^t)]$ and $b_t = 1-\frac{2}{t+t_0}$, we then have
    \begin{align}\label{eq:useful-recursion-gen-step}
        Y^{t+1} &\leq \frac{C_1}{(t+t_0)^{2\eta}} + \frac{C_2}{(t+t_0)^{1+\eta}} + \frac{C_3}{t+t_0} + \log\E[\exp(b_tX^t)] \nonumber \\ 
        &\leq b_tY^t + \frac{C_1}{(t+t_0)^{2\eta}} + \frac{C_2}{(t+t_0)^{1+\eta}} + \frac{C_3}{t+t_0},
    \end{align} where the second inequality follows from the fact that $b_t \in (0,1)$ and Proposition \ref{prop:Jensen}. Unrolling the recursion \eqref{eq:useful-recursion-gen-step} and noting that $Y_1 = X_1$, since $X_1 > 0$ is deterministic, we get
    \begin{align*}
        Y^{t+1} &\leq X^1\prod_{k \in [t]}b_k + \sum_{k \in [t]}\frac{C_1}{(k+t_0)^{2\eta}}\prod_{s = k + 1}^tb_s + \sum_{k \in [t]}\frac{C_2}{(k+t_0)^{1+\eta}}\prod_{s = k + 1}^tb_s + \sum_{k \in [t]}\frac{C_3}{(k+t_0)}\prod_{s = k + 1}^tb_s.
    \end{align*} Using Lemma \ref{lm:Ran}, it follows that $\prod_{s = k+1}^tb_s \leq \Big(\frac{k+t_0}{t+t_0}\Big)^{2\eta}$, therefore using Darboux sums, we get
    \begin{align}\label{eq:technical-intermed-gen-step}
        Y^{t+1} &\leq \frac{X^1}{(t+t_0)^{2\eta}} + \frac{C_1}{(t+t_0)^{2\eta-1}} + \frac{C_2}{(t+t_0)^\eta} + C_3.
    \end{align} Taking the exponent on both sides completes the proof.
\end{proof}

The next two results generalize Lemmas \ref{lm:bdd-iterates} and \ref{lm:consensus-str-cvx}. The proofs follow essentially the same steps as those of Lemmas \ref{lm:bdd-iterates} and \ref{lm:consensus-str-cvx}, with Proposition \ref{prop:Ran} replaced by Lemma \ref{lm:Ran}. For brevity, we omit the details.

\begin{lemma}\label{lm:bdd-iterates-gen-step}
    Let assumptions \textup{\textbf{(A1)}}-\textup{\textbf{(A4)}} and \textup{\textbf{(A6)}} hold and let $a,t_0,K > 0$ and $\nu \in (0,1]$ be some positive constants. If the step-size satisfies $\alpha_t \leq \min\lcb \frac{1}{\osigma\sqrt{2(t+t_0+2)^{\eta-1}K}}, \frac{1}{\mu} \rcb$ for all $t \geq 1$, with $\nu \leq \min\lcb 1, \frac{\mu}{24a\osigma^2K} \rcb$ and $K_{t+1} = (t+t_0+2)^{2\eta-1}K$, then 
    \begin{equation*}
        \E[\exp(\nu K_{t+1}\|\bxtp - \bxs\|^2)] \leq \exp\Bigg(\nu K_{t+1}\bigg(\frac{nC_1}{(t+t_0+2)^{2\eta-1}} + (t+t_0+2)^{1-\eta}C_2 + \frac{C_3}{(t+t_0+2)^{a\mu\eta}}\bigg)\Bigg),
    \end{equation*} for some problem dependent constants $C_i > 0$, $i \in [3]$, all independent of $n$.
\end{lemma}

\begin{lemma}\label{lm:consensus-str-cvx-gen-step}
    Let \textup{\textbf{(A1)}}-\textup{\textbf{(A4)}} and \textup{\textbf{(A6)}} hold, let $a,t_0,K > 0$ and the step-size be given by $\alpha_t = \frac{a}{(t+t_0)^\eta}$ for $\eta \in (1/2,1]$, and let $x^1_i = x^1_j$, for all $i,j\in[n]$. If $a = \frac{6}{\mu}$ and $t_0 \geq \max\Big\{6^{\frac{1}{\eta}},\frac{1}{1-\lambda},\frac{7776\osigma^2\lambda^2K}{\mu^2(1-\lambda)},\Big(\frac{12\lambda L\sqrt{10}}{\mu(1-\lambda)}\Big)^{\frac{1}{\eta}}\Big\}$, then for $K_{t+1} = (t+t_0+2)^{2\eta-1}K$ and any $\nu \leq \min \lcb 1, \frac{\mu^2}{144\sigma^2K} \rcb$, we have 
    \begin{align*}
        \E\bigg[\exp\Big(\nu K_{t+1}\sum_{i \in [n]}\|\xitp - \oxtp \|^2\Big)\bigg] \leq \exp\bigg(\nu K_{t+1}\Big(\sum_{k = 1}^{t}\lambda^{t-k}S_{k} + \sum_{k = 1}^{t}\lambda^{t-k}D_{k}\Big) \bigg),
    \end{align*} where $D_t = \alpha_t^2\Big(\frac{nC_1^\prime}{(t+t_0+2)^{2\eta-1}} + (t+t_0+2)^{1-\eta}C_2^\prime + \frac{C_3^\prime}{(t+t_0+2)^{a\mu\eta}}\Big)$ and $S_t = \alpha_t^2nC_4^\prime$, for some problem related constants $C_i^\prime > 0$, $i \in [4]$, all independent of $n$.
\end{lemma}

We are now ready to state the main result.

\begin{theorem}\label{thm:str-cvx-gen-step}
    Let \textup{\textbf{(A1)}-\textbf{(A4)}} and \textup{\textbf{(A6)}} hold, the step-size be given by $\alpha_t = \frac{a}{(t+t_0)^{\eta}}$ for $\eta \in (1/2,1]$ and let $x^1_i = x^1_j$, for all $i,j\in[n]$. If $a =\frac{6}{\mu}$, $t_0 \geq \max\Big\{6^\frac{1}{\eta},\frac{1+\lambda}{1-\lambda},\frac{23328\osigma^2\kappa^2\lambda^2}{\mu(1-\lambda)},\frac{9720\sigma^2\kappa}{\mu},\Big(\frac{12\kappa\lambda\sqrt{10}}{1-\lambda} \Big)^\frac{1}{\eta}\Big\}$ and $\nu = \min\Big\{1, \frac{\mu}{432\sigma^2\kappa^2}, \frac{\mu}{72\kappa } \Big\}$, then for any $\delta \in (0,1)$ and $T \geq 1$, with probability at least $1 - \delta$, it holds that 
    \begin{align*}
        \frac{1}{n}\sum_{i \in [n]}\big(f(\xit)-f^\star\big) = \bigO\bigg(\frac{\log(\nicefrac{2d}{\delta})}{n(t+t_0)^{2\eta-1}} + \frac{1}{(t+t_0)^{3\eta - 1}} \bigg).
    \end{align*} 
\end{theorem}

Theorem \ref{thm:str-cvx-gen-step} extends the results of Theorem \ref{thm:main-dsgd-str-cvx} for the polynomially decaying step-size $\alpha_t \propto t^{-\eta}$ for $\eta \in (1/2,1]$, recovering the results of Theorem \ref{thm:main-dsgd-str-cvx} when $\eta = 1$. Interestingly, it offers the following insight: \emph{for strongly convex costs, the step-size does not determine the rate in the number of users $n$, but only the rate in time $t$.} In other words, \dsgd achieves the optimal linear speed-up for strongly convex costs independent of the step-size schedule. This is intuitive, as the step-size in Theorem \ref{thm:str-cvx-gen-step} does not depend on $n$, whereas for non-convex costs the step-size needs to be chosen as $\alpha \propto (nT)^{-1/2}$ to guarantee linear speed-up. We now provide a proof sketch, highlighting the main differences to the proof of Theorem \ref{thm:main-dsgd-str-cvx}.

\textit{Proof sketch of Theorem \ref{thm:main-dsgd-str-cvx}.} The main difference stems from changing the scaling when defining the quantity $F^t$, which we now define as $F^{t} \coloneqq n(t+t_0)^{2\eta-1}\big(f(\oxt)-f^\star\big)$. Denoting by $A_t \coloneqq \alpha_t(t+t_0+1)^{2\eta-1}$, using Lemma \ref{lm:descent-inequality} and the properties of strongly convex functions, we get
\begin{align}\label{eq:proof-sketch-str-cvx-gen}
    F^{t+1} &\leq (1-\alpha_t\mu)\bigg(\frac{t+t_0+1}{t+t_0}\bigg)^{2\eta-1}F^t - A_t\langle \nabla f(\oxt),\ozt\rangle + \alpha_tA_tnL\|\ozt\|^2 + \frac{A_tL^2}{2}\sum_{i \in [n]}\|\xit - \oxt\|^2.
\end{align} Starting from \eqref{eq:proof-sketch-str-cvx-gen}, we follow similar steps as in the original proof, with Lemmas \ref{lm:consensus-str-cvx} and \ref{lm:mgf-bound-str-cvx} replaced by Lemmas \ref{lm:consensus-str-cvx-gen-step} and \ref{lm:mgf-bound-str-cvx-gen-step}. To bound the consensus gap, we use Lemma \ref{lm:consensus-str-cvx-gen-step} and additionally proceed as follows
\begin{equation*}
    \E\big[\exp\big(\nu q c_t\|\bxt - \obxt\|^2\big)\big] \leq \E\big[\exp\big(\nu K_t\|\bxt - \obxt\|^2\big)\big] \leq \exp\bigg(\nu K_{t}\sum_{k = 1}^{t-1}\lambda^{t-1-k}\Big(D_k + S_k\Big) \bigg),
\end{equation*} where $D_k = \alpha_k^2\Big(\frac{nC^\prime_1}{(k+t_0+2)^{2\eta-1}} + C^\prime_2(k+t_0+2)^{1-\eta} + \frac{C^\prime_3}{(k+t_0+2)^{a\mu\eta}} \Big)$ and $S_k = \alpha_k^2nC^\prime_4$. To further bound the above expression, we use Lemma \ref{lm:technical-result}, to get
\begin{align*}
    nC^\prime_1\sum_{k = 1}^{t-1}\frac{\lambda^{t-1-k}\alpha_k^2}{(k+t_0)^{2\eta-1}} &= \sum_{k = 1}^{t-1}\lambda^{t-1-k}\frac{nC^\prime_1}{(k+t_0)^{4\eta-1}} \leq \frac{nC^\prime_1}{(t+t_0)^{4\eta-1}} \\
    C^\prime_2\sum_{k = 1}^{t-1}\lambda^{t-1-k}\alpha_k^2(k+t_0)^{1-\eta} &= \sum_{k = 1}^{t-1}\lambda^{t-1-k}\frac{C^\prime_2}{(k+t_0)^{3\eta-1}} \leq \frac{C^\prime_2}{(t+t_0)^{3\eta-1}} \\
    C^\prime_3\sum_{k = 1}^{t-1}\lambda^{t-1-k}\frac{\alpha_k^2}{(k+t_0)^{a\mu\eta}} &= \sum_{k = 1}^{t-1}\lambda^{t-1-k}\frac{C^\prime_3}{(k+t_0)^{(a\mu+2)\eta}} \leq \frac{C^\prime_3}{(t+t_0)^{\eta(2+a\mu)}} \\
    nC^\prime_4\sum_{k = 1}^{t-1}\lambda^{t-1-k}\alpha_k^2 &= \sum_{k = 1}^{t-1}\lambda^{t-1-k}\frac{nC^\prime_4}{(k+t_0)^{2\eta}} \leq \frac{nC^\prime_4}{(t+t_0)^{2\eta}}. 
\end{align*} Noting that $2\eta - 1 \geq 0$ and that $a\mu \geq 1$, it then follows that 
\begin{align*}
    \E\big[\exp\big(\nu q c_t\|\bxt - \obxt\|^2\big)\big] \leq \exp\bigg(\nu K_t\bigg(\frac{G_1}{(t+t_0)^{3\eta-1}} + \frac{nG_2}{(t+t_0)^{2\eta}} \bigg) \bigg),
\end{align*} where $G_1 = C^\prime_2 + C^\prime_3$ and $G_2 = C^\prime_1 + C^\prime_4$. Since $\frac{1}{q} = \frac{\alpha_t\mu}{4+\alpha_t\mu} \leq \frac{\alpha_t\mu}{4}$, we finally get
\begin{align*}
    \sqrt[q]{\E\big[\exp\big(\nu q c_t\|\bxt - \obxt\|^2\big)\big]} &\leq \exp\bigg(\frac{\nu\mu \alpha_tK_t}{4}\bigg(\frac{G_1}{(t+t_0)^{3\eta-1}} + \frac{nG_2}{(t+t_0)^{2\eta}} \bigg) \bigg) \leq \exp\bigg(\frac{\nu G_1}{(t+t_0)^{2\eta}} + \frac{\nu nG_2}{(t+t_0)^{1+\eta}} \bigg), 
\end{align*} where the last inequality follows from the definition of $K_t$ and the step-size. The rest of the proof follows the same steps as that of Theorem \ref{thm:main-dsgd-str-cvx} and is omitted, for brevity.

\end{document}